\documentclass[10pt]{article}
\usepackage[utf8]{inputenc}
\usepackage[T1]{fontenc}
\usepackage{amsmath,amssymb,amsthm}
\usepackage[tt=false,type1=true]{libertine}
\usepackage[libertine]{newtxmath}
\usepackage{foserl_preprint}

\usepackage{amsmath,amsfonts,bm}

\def\eqref#1{equation~\ref{#1}}

\def\1{\bm{1}}

\DeclareMathAlphabet{\mathsfit}{\encodingdefault}{\sfdefault}{m}{sl}
\SetMathAlphabet{\mathsfit}{bold}{\encodingdefault}{\sfdefault}{bx}{n}

\newcommand{\E}{\mathbb{E}}

\newcommand{\R}{\mathbb{R}}

\usepackage{url}
\usepackage{booktabs}

\usepackage{graphicx}
\usepackage[dvipsnames]{xcolor}
\usepackage{subcaption}
\usepackage{enumitem}
\usepackage[most]{tcolorbox}
\usepackage{longtable}
\usepackage[textsize=tiny]{todonotes}
\usepackage{hyperref}
\usepackage{tabularx}
\newcolumntype{L}[1]{>{\raggedright\arraybackslash}p{#1}}
\usepackage{array}
\usepackage{pifont}
\usepackage{makecell}
\usepackage{multirow}
\usepackage{tikz}
\usepackage{comment}
\usepackage{colortbl}
\usepackage{placeins}   
\usepackage{wrapfig}
\newcounter{boxedalgorithm}
\renewcommand{\theboxedalgorithm}{\arabic{boxedalgorithm}}
\usetikzlibrary{arrows.meta}

\usetikzlibrary{
    arrows.meta,
    positioning,
    calc,
    decorations.pathreplacing,
    fit,
    backgrounds
}

\graphicspath{{figures/}}

\newtcolorbox{sarabox}[1]{
  colback=magenta!4,
  colframe=magenta!60!black,
  title={\textbf{Placeholder for Sara --- #1}},
  fonttitle=\small\bfseries
}

\newtcolorbox{todobox}[1]{
  colback=orange!6,
  colframe=orange!70!black,
  title={\textbf{#1}},
  fonttitle=\small\bfseries
}
\newcommand{\cS}{\mathcal{S}}
\newcommand{\cA}{\mathcal{A}}
\newcommand{\cX}{\mathcal{X}}
\newcommand{\cW}{\mathcal{W}}
\newcommand{\cG}{\mathcal{G}}

\newcommand{\Tatk}{\mathcal{T}_{\mathrm{atk}}}
\newcommand{\Aadv}{\mathcal{A}_{\mathrm{adv}}}

\newcommand{\Dtraj}{D_{\mathrm{traj}}}

\newcommand{\xhat}{\hat{x}}
\newcommand{\ubar}{\bar{u}}
\newcommand{\N}{\mathbb{N}}
\newcommand{\T}{\mathbb{T}}

\newcommand{\Prob}{\mathbb{P}}
\newcommand{\QQ}{\mathbb{Q}}
\newcommand{\UU}{\mathcal{U}}
\newcommand{\Bmax}{B_{\max}}

\newcommand{\hd}[1]{\rotatebox{90}{\textbf{#1}\,}}
\newcommand{\Y}{\ding{51}}
\newcommand{\No}{\ding{55}}
\newcommand{\Part}{\textcolor{gray}{$\circ$}}

\theoremstyle{definition}
\newtheorem{definition}{Definition}
\newtheorem{problem}{Problem}
\newtheorem{assumption}{Assumption}
\theoremstyle{plain}
\newtheorem{theorem}{Theorem}
\newtheorem{lemma}{Lemma}

\theoremstyle{remark}

\definecolor{NewBlue}{RGB}{20, 150, 200}
\definecolor{winered}{RGB}{150, 0, 50}
\definecolor{FoSeRLRosFill}{RGB}{247,237,241}
\definecolor{FoSeRLGray}{RGB}{120,120,120}
\definecolor{FoSeRLTableBlue}{RGB}{228,240,248}
\definecolor{PSZeroFill}{RGB}{252,232,240}
\definecolor{FoSeRLGoodFill}{RGB}{224,243,241}
\hypersetup{
    colorlinks=true,
    linkcolor=RoyalBlue,
citecolor=FoSeRLGray,
    urlcolor=lightgray,
}

\newtcolorbox{algobox}[1]{
  enhanced,
  sharp corners,
  boxrule=0.45pt,
  colframe=black!75,
  colback=FoSeRLGray!5,
  colbacktitle=FoSeRLRosFill,
  coltitle=black,
  fonttitle=\scriptsize\bfseries,
  fontupper=\scriptsize,
  title={#1},
  left=4pt,
  right=4pt,
  top=4pt,
  bottom=4pt,
  toptitle=3pt,
  bottomtitle=3pt,
  before skip=0pt,
  after skip=0pt
}
\title{FoSeRL: Formal Sequential Robustness Certification for Reinforcement Learning Policies}

\author{%
Sara Taheri$^{1}$ \quad
Deep Kumar Ganguly$^{2}$ \quad
Jan Křetínský$^{2,3}$ \quad
Majid Zamani$^{1,4}$\\[0.35em]
\small
$^{1}$LMU Munich, Germany \quad
$^{2}$TUM, Munich, Germany \quad
$^{3}$Masaryk University, Brno, Czech Republic \quad
$^{4}$CU Boulder, USA
}
\date{}

\begin{document}
\maketitle

%
\begin{abstract}
Even a few action perturbations can substantially degrade the performance of a deployed decision policy. Certifying the resulting return loss is challenging in stochastic environments, where returns vary even without an attack. We introduce FoSeRL, a framework for certifying deployed RL policies against precommitted, temporally sparse action attacks. The deployed policy is unchanged, with no smoothing or retraining. Certification requires a resettable simulator supporting shared randomness and independent one-step successor queries, but no analytical dynamics model. FoSeRL certifies that an attacked episode loses no more than a prescribed amount of return relative to the same episode unattacked, with at least a target probability and at a user-specified confidence level. Both runs share the initial state and randomness, so the measured loss reflects the attack, not the episode; carrying the running return gap as a state coordinate makes it the terminal value, reducing trajectory-level certification to terminal safety. Time-dependent barrier conditions on the augmented state bound the terminal failure probability: satisfied exactly, they certify every admissible precommitted attack; learned from sampled trajectories and verified on held-out data, they certify the same guarantee under a specified attack-episode setting. Across six stochastic continuous-control environments and three RL policy families (TD3, SAC, and PPO), FoSeRL certifies non-trivial cardinality--magnitude robustness frontiers, achieves substantially larger certified budgets than policy smoothing, and reveals marked robustness differences among policies with comparable nominal performance. 
\end{abstract}

\section{Introduction}
\label{sec:intro}
Reinforcement learning (RL) policies have demonstrated control of
physical systems, including tokamak plasmas
\citep{degrave2022magnetic}, stratospheric balloons
\citep{bellemare2020autonomous}, and racing drones
\citep{kaufmann2023champion}. Once such a policy is deployed, its decisions do not reach the plant directly:
each command travels through communication links and actuation software before
an actuator executes it. Every stage of that path is a place where an adversary
with modest resources can intervene, as end-to-end attacks on production
automotive control systems have shown
\citep{koscher2010experimental,checkoway2011comprehensive}. 
The executed action can therefore differ from the command computed by the
policy, even though the deployed policy itself remains unchanged. Thus, we study deployment-time perturbations of executed actions, with the decision
policy held fixed.

The effect of such perturbations depends on their magnitude and, just as much, on
their timing. On five Atari games, \citet{lin2017tactics} found that perturbing
roughly a quarter of the time steps reduced return about as much as perturbing
every step. The reason is that the closed loop is a dynamical system: a perturbed
action changes the state from which all later decisions are made, so its effect
persists after the perturbation itself has ended. An adversary with a limited
budget can therefore concentrate it where the loop is most sensitive instead of
spreading it uniformly, and the same holds when the perturbation is injected into
the executed action rather than the observation \citep{lee2020spatiotemporal}.
Two quantities govern what such an adversary can do: the number of perturbed time
steps, $K$, and a bound $b$ on the size of each perturbation.
Certifying a deployed policy therefore means controlling their cumulative effect
on return.
In a stochastic environment, the nominal return already varies from episode to
episode, so the baseline against which loss is measured becomes part of the
certification problem. Comparing an attacked run to the policy's average return
would charge that variation to the adversary. We therefore compare each attacked
execution with the run the same episode would have produced unattacked, using the
same initial state and the same environment randomness, and measure the loss
between the two. The attack is then the only difference between them, while the
randomness that shapes how an attack unfolds is preserved rather than averaged
away. Our objective is to identify the attack budgets under which this loss stays
within a prescribed tolerance with a target probability.


\begin{wrapfigure}{r}{0.34\textwidth}
\centering

\begin{tcolorbox}[
    enhanced,
    sharp corners,
    boxrule=0.45pt,
    colframe=black!75,
    colback=FoSeRLGray!5,
    boxsep=0pt,
    left=4pt,
    right=4pt,
    top=4pt,
    bottom=4pt,
    before skip=0pt,
    after skip=0pt
]

\resizebox{\linewidth}{!}{%
\begin{tikzpicture}[
    font=\scriptsize,
    arr/.style={-{Latex[length=1.5mm]}, line width=0.5pt}
]

\draw[arr] (0,0) -- (5.8,0) node[right] {$t$};
\draw[arr] (0,-0.15) -- (0,1.55) node[above] {state};

\draw[thick,RoyalBlue]
plot[smooth] coordinates {
(0,0.45) (0.4,0.70) (0.8,0.55) (1.2,0.78) (1.6,0.62)
(2.0,0.82) (2.4,0.72) (2.8,0.90) (3.2,0.78) (3.6,0.96)
(4.0,0.86) (4.4,1.00) (4.8,0.92) (5.2,1.05)
};

\draw[thick,WildStrawberry]
plot[smooth] coordinates {
(1.6,0.62) (2.0,0.60) (2.4,0.38) (2.8,0.48) (3.2,0.30)
(3.6,0.40) (4.0,0.22) (4.4,0.32) (4.8,0.18) (5.2,0.10)
};

\foreach \x in {1.6,3.0} {
    \draw[WildStrawberry,dashed] (\x,-0.10) -- (\x,1.22);
}

\node[WildStrawberry] at (1.6,1.33) {\textbf{attack}};
\node[WildStrawberry] at (3.0,1.33) {\textbf{attack}};

\node[RoyalBlue,anchor=west] at (5.2,0.86) {\textbf{Nominal}};
\node[WildStrawberry,anchor=west] at (5.2,0.35) {\textbf{Attacked}};

\node[FoSeRLGray,align=center] at (2.9,-0.30)
{shared initial condition and episode randomness};

\begin{scope}[yshift=-2.75cm]

\fill[FoSeRLRosFill] (0,0.95) rectangle (5.5,1.60);

\draw[arr] (0,0) -- (5.8,0) node[right] {$t$};
\draw[arr] (0,-0.15) -- (0,1.70) node[above,align=center] {return \\gap};

\draw[dashed, FoSeRLGray] (0,0.95) -- (5.5,0.95);
\node[left] at (0,0.95) {$\alpha$};

\node[WildStrawberry,anchor=west] at (0.18,1.28)
{\textbf{Beyond tolerance}};

\node[teal,anchor=west] at (0.18,0.38)
{\textbf{Within tolerance}};

\draw[thick,black]
plot[smooth] coordinates {
(0,0) (0.8,0) (1.6,0) (2.0,0.10) (2.4,0.28)
(2.8,0.36) (3.2,0.55) (3.6,0.70) (4.0,0.80)
(4.4,0.88) (4.8,0.92) (5.2,0.94)
};

\fill (5.2,0.94) circle (1.6pt);

\node[anchor=west] at (4.28,0.56)
{\textbf{final degradation}};

\foreach \x in {1.6,3.0} {
    \draw[WildStrawberry,dashed] (\x,-0.10) -- (\x,1.58);
}

\node[FoSeRLGray,align=center] at (2.75,-0.33)
{accumulated discounted reward difference};

\end{scope}
\end{tikzpicture}%
}

\end{tcolorbox}

\vspace{0.1cm}

\caption{\footnotesize
Matched nominal and attacked executions share episode randomness;
their accumulated return gap becomes the terminal degradation
certified by FoSeRL.
}
\label{fig:foserl-overview}

\vspace{-2em}
\end{wrapfigure}
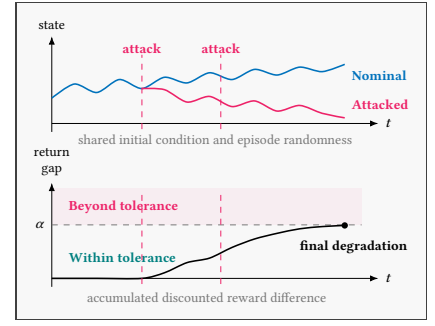
\vspace{-0.3cm}
\paragraph{Contributions.} We introduce \textsc{FoSeRL}, a post-hoc framework for certifying fixed deployed policies against temporally sparse action perturbations. Our contributions are threefold. \textbf{\textit{(i)}} We formulate sequential robustness under joint cardinality--magnitude budgets and define a paired return-degradation criterion using synchronously coupled nominal and attacked executions, isolating attack-induced loss from episode variability. \textbf{\textit{(ii)}} We lift the cumulative return gap to a terminal quantity of an augmented stochastic system and derive time-dependent barrier conditions for its tail reliability (Fig.~\ref{fig:foserl-overview}). \textbf{\textit{(iii)}} We learn neural barrier certificates from sampled rollouts and verify them using finite-sample successor bounds and episode-level conformal calibration, yielding high-confidence operational certificates under a declared attack--episode law. FoSeRL leaves the deployed policy unchanged and requires no analytical dynamics model. Across six stochastic continuous-control environments and TD3, SAC, and PPO policies, it yields non-trivial robustness frontiers and substantially larger certified
budgets than the baseline.

\paragraph{Related work.}
Temporally sparse attacks are well established empirically
\citep{lin2017tactics,sun2020stealthy}, including on the action channel
\citep{lee2020spatiotemporal}, and defences against them harden a policy by
retraining \citep{pinto2017robust,tessler2019action,liang2022wocar} without
emitting a certificate. Certified robustness for RL does emit one, but of a
different shape. Smoothing-based methods certify actions or cumulative reward
under norm-bounded perturbations
\citep{kumar2022policy,wu2022crop,mu2024reward}; certified state-action bounds
support robust action selection \citep{everett2021certifiable}; and CAROL learns
a policy jointly with a model-based robustness certificate
\citep{yang2024carol}. Each bounds perturbations under a norm ball, typically
guarantees an expected return, and certifies a smoothed, retrained, or
robust-argmax surrogate rather than the policy as deployed. Stochastic and neural
barrier certificates supply the machinery we build on
\citep{prajna2007framework,steinhardt2012finite,abate2021fossil,anand2023formally,luo2021crabs},
but bound state-safety events rather than adversarial performance loss; the
framework closest to ours certifies a supervised learner against data poisoning
by modeling training as a stochastic dynamical system
\citep{taheri2026probabilistic}, and has no time axis for an attacker to exploit.
Attacks in the literature are budgeted along two axes: how often the adversary
strikes and how hard, but existing certificates fix one or the other: they either
bound perturbations at every step or count corruptions without bounding their
size. The cardinality--magnitude trade-off therefore appears throughout the
\emph{attack} literature and in no prior \emph{certificate}. FoSeRL is the first
to certify it, and the first to do so for the policy exactly as deployed, against
perturbations of its executed actions, as a tail bound on \emph{paired}
degradation. Table~\ref{tab:related} in Appendix~\ref{app:rel}
places each method on these axes.

\section{Problem Formulation}
\label{sec:setup}

All proofs and a summary of notation are provided in the appendix.
We consider a finite-horizon Markov decision process
$\mathcal M\!=\!(\cS,\cA,P,r,\gamma,\mu_0,T)$, where
$\cS\!\subseteq\!\R^{n_s}$ and $\cA\!\subseteq\!\R^{n_a}$ are the
state and action spaces, $\cA$ is closed and convex,
$P(\cdot\mid x,u)$ is the transition kernel,
$r:\cS\times\cA\!\to\!\R$ is measurable,
$\gamma\!\in\!(0,1]$, and $x(0)\!\sim\!\mu_0$.
We use $T\!\in\!\N$, $\T\!:=\!\{0,\ldots,T\}$, and
$\T'\!:=\!\{0,\ldots,T-1\}$.
Fix a deployed decision policy $\pi$, held unchanged throughout
certification, and let $u(t)\!\in\!\cA$ denote its realized action.
Its execution is specified in Assumption~\ref{ass:noise}; no
restriction on internal policy randomization is imposed.
The state contains all variables required for the closed-loop Markov
property, and conditional on the current state and executed action,
the successor has law $P(\cdot\mid x(t),u(t))$.
For a nominal state--action trajectory $\xi$, define
$J(\pi;\xi)\!:=\!\sum_{t\in\T'}\gamma^t r(x(t),u(t))$ and
$J(\pi)\!:=\!\E[J(\pi;\xi)]$.
All certification statements condition on the deployed policy;
randomness used to train it is not resampled during execution.

\begin{assumption}[Bounded rewards]
\label{ass:bounded}
The reward satisfies $|r(x,u)|\!\le\!r_{\max}$ for all
$(x,u)\!\in\!\cS\times\cA$. Hence every realized discounted return
lies in $[-S_{\max},S_{\max}]$, where
$S_{\max}\!:=\!r_{\max}\sum_{t\in\T'}\gamma^t$.
\end{assumption}

We next formalize the deployment-time adversary acting on the
policy's executed actions.

\begin{definition}[Oblivious sequential attack]
\label{def:attack}
Fix $K\!\in\!\T$, $b\!\ge\!0$, and
$p\!\in\!\{1,2,\infty\}$.
An attack scenario is $\upsilon\!=\!(\Tatk,\Delta)$, where
$\Tatk\!\subseteq\!\T'$ contains the attacked time indices and
$\Delta\!=\!\{\Delta_t\}_{t\in\Tatk}$ the corresponding
perturbations $\Delta_t\!\in\!\R^{n_a}$.
Both are fixed before the initial state and episode randomness are
realized. We set $\Delta_t\!=\!0$ for $t\!\notin\!\Tatk$.
The admissible class $\Aadv(K,b)$ consists of all such scenarios with
$|\Tatk|\!\le\!K$ and $\|\Delta_t\|_p\!\le\!b$ for every
$t\!\in\!\Tatk$.
\end{definition}

To compare an attack with the nominal execution of the \emph{same} episode, we make all fresh closed-loop randomness explicit.

\begin{assumption}[Execution randomness representation]
\label{ass:noise}
There exist measurable execution maps
$\pi:\cS\times\cW\!\to\!\cA$ and
$F:\cS\times\cA\times\cW\!\to\!\cS$, and an i.i.d.\ sequence
$w(t)\!\sim\!P_w$, independent of $x(0)$.
Each $w(t)$ collects all fresh one-step policy and transition
randomness under the declared joint law. The fixed policy generates
$u(t)\!=\!\pi(x(t),w(t))$, and, given the current state and executed
action $\bar u(t)$,
$F(x(t),\bar u(t),w(t))$ has law
$P(\cdot\mid x(t),\bar u(t))$.
A policy without internal randomization ignores the corresponding
inputs; persistent cross-step dependence is represented in the state.
\end{assumption}

The next definition uses this representation to construct matched nominal and attacked executions.

\begin{definition}[Paired degradation via synchronous coupling]
\label{def:coupling}
Fix $\upsilon\!\in\!\Aadv(K,b)$ and let
$\omega\!:=\!(x(0),\{w(t)\}_{t\in\T'})\!\sim\!\Prob_\omega$
denote the episode randomness.
Both branches share $x(0)$ and the complete random-input sequence. With $\Pi_{\cA}$ denoting Euclidean projection onto $\cA$,
set
$u^0(t)\!:=\!\pi(x^0(t),w(t))$ and
$\ubar^a(t;\upsilon)\!:=\!
\Pi_{\cA}\!\bigl(\pi(x^a(t),w(t))+\Delta_t\bigr)$.
At each step, transition and reward use the same realized action. The magnitude budget applies to $\Delta_t$ before
projection; admissibility details are given in
Appendix~\ref{app:attack-admissibility}.
For $t\!\in\!\T'$,
$x^0(t{+}1)\!=\!F(x^0(t),u^0(t),w(t))$ and
$x^a(t{+}1)\!=\!F(x^a(t),\ubar^a(t;\upsilon),w(t))$,
with $x^0(0)\!=\!x^a(0)\!=\!x(0)$.
Let
$R_{\mathrm{nom}}(\omega)\!:=\!
\sum_{t\in\T'}\gamma^t r(x^0(t),u^0(t))$ and
$R_{\mathrm{att}}(\upsilon,\omega)\!:=\!
\sum_{t\in\T'}\gamma^t
r(x^a(t),\ubar^a(t;\upsilon))$.
The paired trajectory-level degradation is
\vspace{-0.1cm}
\begin{equation}
\Dtraj(\upsilon,\omega)
:=
R_{\mathrm{nom}}(\omega)
-
R_{\mathrm{att}}(\upsilon,\omega).
\label{eq:pairedD}
\end{equation}
\end{definition}
\vspace{-0.2cm}
The coupling is part of the specification: although
$\E_\omega[\Dtraj(\upsilon,\omega)]
\!=\!J(\pi)-\E_\omega[R_{\mathrm{att}}(\upsilon,\omega)]$
depends only on the trajectory marginals, the law of $\Dtraj$ generally
does not.
Under the synchronous coupling, $K\!=\!0$ or $b\!=\!0$ gives
$\Dtraj\!\equiv\!0$ pathwise.
The representation $(\pi,F,P_w)$ and its shared-randomness semantics
are fixed before synthesis and verification; Appendix~\ref{app:coupling-semantics}
gives details showing that the marginals alone are
insufficient. We can now state the desired robustness criterion.

\begin{definition}[Certified performance]
\label{def:certperf}
Fix an attack budget $(K,b)$, $\alpha\!\ge\!0$, and
$\rho\!\in\!(0,1]$.
The fixed policy $\pi$ is
\emph{$(\alpha,\rho)$-certified performant} under $\Aadv(K,b)$ if
its paired degradation does not exceed $\alpha$ with probability at
least $\rho$ for every admissible attack:
\begin{equation}
\Prob_\omega\!\left(
\Dtraj(\upsilon,\omega)\le\alpha
\right)
\ge\rho,
\qquad
\forall\,\upsilon\in\Aadv(K,b).
\label{eq:certified-forall}
\end{equation}
The probability is over $\omega$, conditional on the fixed policy and
execution representation; attacks are quantified universally, with no
attack distribution.
\end{definition}

We finally quantify the certified perturbation tolerance at a fixed
cardinality budget.

\begin{problem}[Sequential robustness certification]
\label{prob:radius}
Given the fixed policy $\pi$, prescribed levels
$\alpha\!\ge\!0$ and $\rho\!\in\!(0,1]$, and a cardinality budget
$K$, determine the supremal attack magnitude for which $\pi$ remains
$(\alpha,\rho)$-certified performant under $\Aadv(K,b)$.
We denote this value by $b_{\mathrm{cert}}^*(K)\!\in\![0,+\infty]$.
As $K$ varies, $K\!\mapsto\!b_{\mathrm{cert}}^*(K)$ defines the
\emph{cardinality--magnitude robustness frontier}.
\end{problem}

\vspace{-0.1cm}
\section{Methodology}
\label{sec:method}
\vspace{-0.15cm}

The criterion in Definition~\ref{def:certperf} is trajectory-level and
not directly a state-space safety property. We therefore augment the
synchronously coupled nominal and attacked dynamics with their cumulative
discounted return gap, making the certified degradation a terminal-state
quantity. This recasts sequential robustness certification as a finite-horizon
terminal safety problem amenable to stochastic barrier certificates.

\begin{definition}[Augmented dt-SDS]
\label{def:dtsds}
Fix an attack budget $(K,b)$ and the corresponding class
$\Aadv(K,b)$ of Definition~\ref{def:attack}.
The augmented discrete-time stochastic dynamical system is
$\mathfrak S\!:=\!(\cX,\cX_0,\Aadv(K,b),\cW,P_w,f,T)$, where
$\cX\!:=\!\cS\times\cS\times[-2S_{\max},2S_{\max}]$,
$\cX_0\!:=\!\{(x,x,0):x\!\in\!\operatorname{supp}(\mu_0)\}$, and
$\xhat(t)\!:=\!(x^0(t),x^a(t),z(t))$, initialized at
$\xhat(0)\!=\!(x(0),x(0),0)$ with $x(0)\!\sim\!\mu_0$.
For each fixed $\upsilon\!\in\!\Aadv(K,b)$, let
$u^0(t)\!:=\!\pi(x^0(t),w(t))$,
let $\ubar^a(t;\upsilon)$ be as in Definition~\ref{def:coupling},
and let $\cX_t^\upsilon\!\subseteq\!\cX$ denote the support of the
augmented-state distribution at time $t$ under $\upsilon$. Whenever $f(\xhat,\upsilon,w,t)$ is evaluated, the actions below are
recomputed from the supplied $w$, consistently with
Definition~\ref{def:coupling}.
The augmented dynamics are
\vspace{-0.15cm}
\begin{equation}
\xhat(t{+}1)
\!=\!
f(\xhat(t),\upsilon,w(t),t)
\!:=\!
\begin{bmatrix}
F\!\left(x^0(t),u^0(t),w(t)\right)\\[2pt]
F\!\left(x^a(t),\ubar^a(t;\upsilon),w(t)\right)\\[2pt]
z(t)+\gamma^t\!\left[
r(x^0(t),u^0(t))
-r(x^a(t),\ubar^a(t;\upsilon))
\right]
\end{bmatrix}\!,
\qquad t\!\in\!\T'.
\label{eq:dtsds}
\end{equation}
\end{definition}
\vspace{-0.2cm}
Each evaluation of $f(\xhat,\upsilon,w,t)$ recomputes both actions and
the return-gap increment using its supplied $w$; successor draws therefore
reuse neither cached actions nor cached gap increments from the source episode.
Hence $z(t{+}1)$ may remain random conditional on
$(\xhat(t),\upsilon,t)$.
Since each admissible attack is precommitted, $\upsilon$ parameterizes the
augmented dynamics and need not be included in the state.
The identity $z(T)\!=\!\Dtraj(\upsilon,\omega)$ yields the following
terminal safety specification.

\begin{definition}[Terminal safety specification]
\label{def:safety}
For the augmented dt-SDS $\mathfrak S$ of
Definition~\ref{def:dtsds} and $\alpha\!\ge\!0$, define
$\cX_s^\alpha\!:=\!\{\xhat\!\in\!\cX:z\!\le\!\alpha\}$ and
$\cX_u^\alpha\!:=\!\cX\setminus\cX_s^\alpha
=\{\xhat\!\in\!\cX:z\!>\!\alpha\}$.
For $\rho\!\in\!(0,1]$, let
$\Psi_{\alpha,\rho}\!:=\!(\cX_0,\cX_s^\alpha,\cX_u^\alpha,T,\rho)$.
We say $\mathfrak S$ satisfies $\Psi_{\alpha,\rho}$ if
$\Prob_\omega(\xhat(T)\!\in\!\cX_s^\alpha)\!\ge\!\rho$ for every
$\upsilon\!\in\!\Aadv(K,b)$.
\end{definition}
This specification is equivalent to Definition~\ref{def:certperf} and imposes no intermediate safety constraint. We next derive sufficient time-dependent barrier conditions,
uniformly over $\Aadv(K,b)$.
\begin{definition}[Barrier certificate]
\label{def:barrier}
Consider the augmented dt-SDS $\mathfrak S$ in
Definition~\ref{def:dtsds} with attack class $\Aadv(K,b)$ and terminal
specification $\Psi_{\alpha,\rho}$ of Definition~\ref{def:safety}.
A measurable function
$B:\cX\times\T\!\to\!\R_{\ge0}$ is a
\emph{barrier certificate} for $\mathfrak S$ with respect to
$\Psi_{\alpha,\rho}$ if there exist
$b_1,b_2,b_3\!\in\!\R_{>0}$ and $\kappa\!\in\!(0,1]$
such that the following conditions hold for every
$\upsilon\!\in\!\Aadv(K,b)$, where
$\cG_T(\kappa)\!:=\!\sum_{j=0}^{T-1}\kappa^j$:
\vspace{-0.1cm}
\begin{equation}
\begin{aligned}
&B(\xhat,0)
\le b_1,
&&\forall\,\xhat\!\in\!\cX_0,\\
&B(\xhat,T)
\ge b_2,
&&\forall\,\xhat\!\in\!\cX_u^\alpha,\\
&\E_{w\sim P_w}\!\left[
B\!\left(f(\xhat,\upsilon,w,t),t{+}1\right)
\right]
\le
\kappa B(\xhat,t)+b_3,
&&\forall\,\xhat\!\in\!\cX_t^\upsilon,\;
\forall\,t\!\in\!\T',\\
&b_1\kappa^T+b_3\cG_T(\kappa)
\le
b_2(1-\rho).
\end{aligned}
\label{eq:BCconditions}
\end{equation}
\end{definition}
\vspace{-0.2cm}
The four conditions bound the initial value, separate terminal-unsafe
states, control the expected one-step drift along states reachable under
each admissible attack, and calibrate the resulting bound to $\rho$.
Together they bound terminal failure uniformly over $\Aadv$.
Because the same attack-independent barrier must accommodate all
precommitted attack plans, these sufficient conditions can be conservative
at states shared across different plans.
Their conservatism is discussed in
Appendix~\ref{app:common-barrier-limitation}.

\begin{theorem}[Ideal sequential robustness guarantee]
\label{thm:ideal}
Fix $K\!\in\!\T$, $b,\alpha\!\ge\!0$, and $\rho\!\in\!(0,1)$.
Under Assumptions~\ref{ass:bounded} and~\ref{ass:noise}, consider the
augmented dt-SDS $\mathfrak S$ of Definition~\ref{def:dtsds} with
terminal safety specification $\Psi_{\alpha,\rho}$ of
Definition~\ref{def:safety}.
Suppose $B$ is a barrier certificate in the sense of
Definition~\ref{def:barrier}, with parameters
$b_1,b_2,b_3$, and $\kappa$.
Then $\mathfrak S$ satisfies $\Psi_{\alpha,\rho}$; equivalently, the
fixed policy $\pi$ is $(\alpha,\rho)$-certified performant under
$\Aadv(K,b)$ in the sense of Definition~\ref{def:certperf}.
Consequently, $b$ belongs to the feasible set defining
$b_{\mathrm{cert}}^*(K)$ in Problem~\ref{prob:radius}.
More precisely, for every $\upsilon\!\in\!\Aadv(K,b)$,
\begin{equation}
\Prob_\omega\!\left(
\Dtraj(\upsilon,\omega)\le\alpha
\right)
\ge
1-\frac{b_1\kappa^T+b_3\cG_T(\kappa)}{b_2}
\ge\rho.
\label{eq:idealbound}
\end{equation}
\end{theorem}
\vspace{-0.2cm}
Theorem~\ref{thm:ideal} requires the barrier inequalities to hold
uniformly over the specified reachable states and admissible attacks.
The finite-sample procedure below instead certifies operational
reliability under a declared attack law and does not establish
feasibility for the universal objective in Problem~\ref{prob:radius}.

\vspace{-0.1cm}
\section{Data-Driven Sequential Robustness Certification}
\label{sec:datadriven}
\vspace{-0.2cm}
We assume access to synchronously coupled rollouts and repeatable
one-step queries from reachable augmented state--attack--time tuples.
Each query restores both branches before action generation and draws a
fresh complete $w\!\sim\!P_w$, shared within the pair and independent
of the source episode. Thus, policy and transition randomness are
regenerated, and the reward-gap coordinate is updated from the same
realized actions. No analytical dynamics model is required.
We next specify the attack--episode laws used for candidate construction
and formal verification.

\begin{definition}[Attack--episode laws]
\label{def:jointlaws}
For an attack budget $(K,b)$, let $\UU_{K,b}$ be a prescribed
probability law supported on $\Aadv(K,b)$. We model the sampled attack
and episode randomness as independent, consistently with the oblivious
threat model, and define the operational law
$\QQ_{K,b}\!:=\!\UU_{K,b}\otimes\Prob_\omega$.
For synthesis at a design magnitude
$b_{\mathrm{cnd}}\!\in\!\R_{\ge0}$, we allow a separate law
$\UU^{\mathrm{syn}}_{K,b_{\mathrm{cnd}}}$ supported on
$\Aadv(K,b_{\mathrm{cnd}})$ and define
$\QQ^{\mathrm{syn}}_{K,b_{\mathrm{cnd}}}
\!:=\!\UU^{\mathrm{syn}}_{K,b_{\mathrm{cnd}}}\otimes\Prob_\omega$.
\end{definition}
The synthesis law is used only to construct the candidate barrier;
the formal data-driven guarantee is stated under $\QQ_{K,b}$.
For the fixed policy, tolerance, horizon, and coupling, define
$R_{\mathrm{op}}(K,b;\UU_{K,b})
\!:=\!
\Prob_{(\upsilon,\omega)\sim\QQ_{K,b}}
\!\left(\Dtraj(\upsilon,\omega)\!\le\!\alpha\right)$.
The operational target is $R_{\mathrm{op}}\!\ge\!\rho$ for a tested
specification, rather than the universal per-attack requirement of
Problem~\ref{prob:radius}.
We reserve $b_{\mathrm{cert}}^{\mathrm{op}}$ for an accepted tested
magnitude under this operational criterion, distinct from the universal
radius $b_{\mathrm{cert}}^*(K)$.
Henceforth, we consider
$K\!>\!0$, $b\!>\!0$, $0\!\le\!\alpha\!<\!2S_{\max}$, and
$\rho\!\in\!(0,1)$; other cases are summarized in
Appendix~\ref{app:analytical-cases}.

\vspace{-0.1cm}
\subsection{Certificate Synthesis}
\label{sec:synthesis}
\vspace{-0.15cm}
We synthesize a candidate for the conditions of
Definition~\ref{def:barrier} by jointly learning a neural barrier
and its associated constants from coupled trajectories.
The synthesis specification consists of a cardinality budget
$K\!\in\!\T$, a design magnitude
$b_{\mathrm{cnd}}\!\in\!\R_{\ge0}$, a degradation tolerance
$\alpha\!\in\!\R_{\ge0}$, and a target reliability
$\rho\!\in\!(0,1)$.
We choose a synthesis-only reserve
$\eta_{\mathrm{syn}}\!\in\!(0,1-\rho)$ and set
$\rho_{\mathrm{syn}}\!:=\!\rho+\eta_{\mathrm{syn}}$.
This reserve provides optimization headroom only: it has no formal
role, and the certification target remains $\rho$.
Let $h_\phi$ be a scalar neural network with parameters $\phi$.
For a prescribed $\Bmax\!\in\!\R_{>0}$, we parameterize
$B_\phi(\xhat,t)\!:=\!\Bmax\operatorname{sigmoid}(h_\phi(\xhat,t))$,
which ensures $B_\phi(\xhat,t)\!\in\!(0,\Bmax)$ for every
$(\xhat,t)\!\in\!\cX\times\T$.

\vspace{-0.25cm}
\paragraph{Synthesis data.}
We draw $N_s$ independent scenarios
$(\upsilon_i,\omega_i)
\!\stackrel{\mathrm{i.i.d.}}{\sim}\!
\QQ^{\mathrm{syn}}_{K,b_{\mathrm{cnd}}}$, $i\!\in\![N_s]$,
and generate the corresponding trajectories
$\xhat_i(0{:}T)$ according to Definition~\ref{def:dtsds}. Let
$d_i\!:=\!\Dtraj(\upsilon_i,\omega_i)$,
$\mathcal I_s\!:=\!\{i\!\in\![N_s]\mid d_i\!\le\!\alpha\}$, and
$\mathcal I_u\!:=\!\{i\!\in\![N_s]\mid d_i\!>\!\alpha\}$.
We organize their initial states, state--attack--time tuples, and
terminal safe--unsafe partition as indexed collections that retain
repeated observations:
\vspace{-0.1cm}
\begin{equation}
I\!:=\!\bigl(\xhat_i(0)\bigr)_{i\in[N_s]},\quad
X\!:=\!\bigl(\xhat_i(t),\upsilon_i,t\bigr)_{i\in[N_s],t\in\T'},\quad
S\!:=\!\bigl(\xhat_i(T)\bigr)_{i\in\mathcal I_s},\quad
U\!:=\!\bigl(\xhat_i(T)\bigr)_{i\in\mathcal I_u}.
\label{eq:setSU}
\end{equation}
Since $z_i(T)=\Dtraj(\upsilon_i,\omega_i)$, every element of
$I$ lies in $\cX_0$, while the elements of $S$ and $U$ lie in
$\cX_s^\alpha$ and $\cX_u^\alpha$, respectively.
Thus, $S$ and $U$ collect sampled terminal safe and unsafe states;
$U$ enters the synthesis loss, while $S$ is retained for supplementary
analysis. Additional synthesis-pool conventions are given in
Appendix~\ref{app:certificate-parameterization}.
To approximate the expectation in the third condition of
Definition~\ref{def:barrier}, for each indexed tuple
$(\xhat_i(t),\upsilon_i,t)$, $i\!\in\![N_s]$, $t\!\in\!\T'$, we draw
$M_s$ fresh conditionally i.i.d. random inputs
$w_{itj}\!\sim\!P_w$, $j\!\in\![M_s]$, independently
of the source episode.
For each tuple, define
$\bar E_{it}\!:=\!M_s^{-1}
\sum_{j=1}^{M_s}
B_\phi(f(\xhat_i(t),\upsilon_i,w_{itj},t),t{+}1)$
and replace the exact expectation during synthesis by
$\bar E_{it}+\epsilon$, where
$\epsilon\!\in\!\R_{\ge0}$ is a prescribed synthesis margin.
This margin tightens the sampled drift condition during synthesis only
and carries no formal confidence guarantee.

\vspace{-0.25cm}
\paragraph{Synthesis objective.}
We combine the sampled barrier conditions in a weighted training loss.
For prescribed weights
$c_\rho,c_I,c_U,c_f\!\in\!\R_{>0}$, we jointly optimize
$\phi$, $b_1,b_2,b_3\!\in\!\R_{>0}$, and
$\kappa\!\in\!(0,1]$, subject to $b_2\!<\!\Bmax$, by minimizing
$\mathcal L/b_2$, where
\vspace{-0.3cm}
\begin{equation}
\begin{aligned}
\mathcal L ={}&
c_\rho\!\operatorname{ReLU}\!\left(
b_1\kappa^T+b_3\cG_T(\kappa)-b_2(1-\rho_{\mathrm{syn}})
\right)
~+~
c_I\!\!\sum_{i\in[N_s]}
\operatorname{ReLU}\!\left(
B_\phi(\xhat_i(0),0)-b_1
\right)
\\
&+
c_U\!\!\!\sum_{i\in\mathcal I_u}
\!\operatorname{ReLU}\!\left(
b_2\!-\!B_\phi(\xhat_i(T),T)
\right)
~+~
c_f\!\!\!\!\sum_{i\in[N_s]}\!\sum_{t\in\T'}
\operatorname{ReLU}\!\left(
\bar E_{it}\!+\!\epsilon
\!-\!\kappa B_\phi(\xhat_i(t),t)\!-\!b_3
\right).\!
\end{aligned}
\label{eq:loss}
\end{equation}
The four terms penalize violations of the reliability, initial,
unsafe-terminal, and sampled drift conditions, respectively.
After training, the candidate $B_{\phi^*}$ and constants
$b_1^*,b_2^*,b_3^*,\kappa^*$ are fixed before independent
verification. Synthesis constructs a candidate only; its formal
validity is established by the verification stage.
\vspace{-0.1cm}
\subsection{Certificate Verification}
\label{sec:verification}
\vspace{-0.15cm}
Verification converts the frozen synthesis output into a finite-sample
robustness guarantee under the operational law $\QQ_{K,b}$ of
Definition~\ref{def:jointlaws}. Fix a verification specification
$K\!\in\!\T$, $b,\alpha\!\in\!\R_{\ge0}$, and
$\rho\!\in\!(0,1)$, and write $B_\phi^*\!:=\!B_{\phi^*}$ with
$b_1^*,b_2^*,b_3^*,\kappa^*$ from Sec.~\ref{sec:synthesis}.
The candidate and all verification-design choices are fixed before
calibration; fresh calibration data are drawn independently of data
used for candidate construction or selection. All probability
statements below condition on these pre-calibration choices.
We first express the barrier conditions as signed residuals.

\begin{definition}[Barrier residuals]
\label{def:residuals}
For the frozen candidate and verification specification, define
$\mathcal Z_1\!:=\!\cX_0\times\Aadv(K,b)\times\{0\}$,
$\mathcal Z_2\!:=\!\cX_u^\alpha\times\Aadv(K,b)\times\{T\}$, and
$\mathcal Z_3\!:=\!\{(\xhat,\upsilon,t):
\upsilon\!\in\!\Aadv(K,b),\,
t\!\in\!\T',\,
\xhat\!\in\!\cX_t^\upsilon\}$,
where $\cX_t^\upsilon$ is defined in
Definition~\ref{def:dtsds}. For $j\!\in\!\{1,2,3\}$, set
$q_j\!=\!0$ outside $\mathcal Z_j$.
Positive residuals indicate violations of the corresponding conditions
in Definition~\ref{def:barrier}:
\vspace{-0.1cm}
\begin{align}
q_1(\xhat,\upsilon,t)
&\!:=B_\phi^*(\xhat,0)-b_1^*,
&&\forall\,(\xhat,\upsilon,t)\!\in\!\mathcal Z_1,
\label{eq:q1}\\
q_2(\xhat,\upsilon,t)
&\!:=b_2^*-B_\phi^*(\xhat,T),
&&\forall\,(\xhat,\upsilon,t)\!\in\!\mathcal Z_2,
\label{eq:q2}\\
q_3(\xhat,\upsilon,t)
&\!:=\E_{w\sim P_w}\!\big[
B_\phi^*\!\left(f(\xhat,\upsilon,w,t),t{+}1\right)
\big]
-\kappa^*B_\phi^*(\xhat,t)-b_3^*,
&&\forall\,(\xhat,\upsilon,t)\!\in\!\mathcal Z_3,
\label{eq:q3}\\
q_4
&\!:=b_1^*(\kappa^*)^T
+b_3^*\cG_T(\kappa^*)
-b_2^*(1-\rho).
\label{eq:q4}
\end{align}
\end{definition}
\vspace{-0.3cm}
The initial and terminal residuals are directly evaluable, so
$\bar q_1\!:=\!q_1$ and $\bar q_2\!:=\!q_2$.
The drift residual $q_3$ contains an expectation under the stochastic
dynamics. Fix $M\!\in\!\N$ and
$\epsilon_M\!\in\!\R_{>0}$ before calibration.
At each queried
$(\xhat,\upsilon,t)\!\in\!\mathcal Z_3$, draw
$w_1,\ldots,w_M\stackrel{\mathrm{i.i.d.}}{\sim}P_w$
using fresh random inputs independent of the complete episode that
supplies the tuple. Define
\vspace{-0.2cm}
\begin{align}
\overline E_M^*(\xhat,\upsilon,t)
&\!:=\frac{1}{M}\sum_{m=1}^{M}
B_\phi^*\!\left(
f(\xhat,\upsilon,w_m,t),t{+}1
\right)+\epsilon_M,
\label{eq:Ebar}\\[-0.5ex]
\bar q_3(\xhat,\upsilon,t)
&\!:=\overline E_M^*(\xhat,\upsilon,t)
-\kappa^*B_\phi^*(\xhat,t)-b_3^*,
\qquad
(\xhat,\upsilon,t)\!\in\!\mathcal Z_3,
\label{eq:q3bar}
\end{align}
with $\bar q_3\!:=\!0$ outside $\mathcal Z_3$.
Unlike the synthesis margin, $\epsilon_M$ has a formal finite-sample
role through the following bound.

\begin{lemma}[Supermartingale residual bound]
\label{lem:finiteM}
Fix $\beta_0\!\in\!(0,1)$,
$\epsilon_M\!\in\!\R_{>0}$, and
$(\xhat,\upsilon,t)\!\in\!\mathcal Z_3$.
Suppose $B_\phi^*(\xhat',t')\!\in\!(0,\Bmax)$ for every
$(\xhat',t')\!\in\!\cX\times\T$, and the successor random inputs
used in~\eqref{eq:Ebar} are conditionally i.i.d.\ from $P_w$ at the
fixed tuple. Then, with probability at least $1-\beta_0$ over the
successor batch,
$q_3(\xhat,\upsilon,t)\!\le\!\bar q_3(\xhat,\upsilon,t)$, provided
\vspace{-0.1cm}
\begin{equation}
M\ge
\Big\lceil
\frac{\Bmax^2}{2\epsilon_M^2}
\ln\frac{1}{\beta_0}
\Big\rceil .
\label{eq:Mbound}
\end{equation}
\end{lemma}
\vspace{-0.4cm}
\paragraph{Conformal calibration.}
We calibrate violations at the episode level, so no independence
between time steps within one episode is required.
\begin{definition}[Episode-level verification score]
\label{def:score}
Fix $N\!\in\!\N$ and let
$\mathcal D_{\mathrm{cal}}\!=\!(\bar\zeta_i)_{i=1}^{N}$ contain
independent copies of
$\bar\zeta\!=\!(\upsilon,\omega,W)$, where
$(\upsilon,\omega)\!\sim\!\QQ_{K,b}$ generates the augmented trajectory
of Definition~\ref{def:dtsds}.
For each episode, $W$ contains the fresh successor batches used in
\eqref{eq:Ebar} at all $t\!\in\!\T'$, drawn independently of the
complete scenario $(\upsilon,\omega)$.
Write
$\bar q_{1,i}\!:=\!\bar q_1(\xhat_i(0),\upsilon_i,0)$ and
$\bar q_{2,i}\!:=\!\bar q_2(\xhat_i(T),\upsilon_i,T)$.
The episode score combines the initial and drift residuals while
retaining terminal separation as a hard condition:
\vspace{-0.15cm}
\begin{equation}
S_i\!:=\!
\small\begin{cases}
\displaystyle
\max\!\big\{
\bar q_{1,i},\,
\max_{t\in\T'}
\bar q_3(\xhat_i(t),\upsilon_i,t)
\big\},
& \bar q_{2,i}\le0,\\[3pt]
+\infty,
& \bar q_{2,i}>0.
\end{cases}
\label{eq:score}
\end{equation}
\end{definition}
\vspace{-0.3cm}
Thus an unsafe-terminal separation violation receives an infinite score
and cannot be repaired by a finite calibration margin; all scores,
including infinite ones, are retained.
\begin{lemma}[Conformal calibration]
\label{lem:conformal}
Fix $\beta\!\in\!(0,1]$ and let
$S_{(1)}\!\le\!\cdots\!\le\!S_{(N)}$ be the order statistics of the
scores in Definition~\ref{def:score}.
For each $k\!\in\![N]$, let $\bar v_k$ be the smallest
$v\!\in\![0,1]$ satisfying the binomial-tail condition below.
Conditional on the pre-calibration choices, with probability at least
$1-\beta$ over $\mathcal D_{\mathrm{cal}}$, the score of a fresh
independent episode exceeds $S_{(k)}$ with probability at most
$\bar v_k$, simultaneously for all $k\!\in\![N]$.
The statement permits ties and infinite scores:
\vspace{-0.1cm}
\begin{equation}
\sum_{j=k}^{N}\binom Nj(1-v)^jv^{N-j}
\le\frac{\beta}{N}.
\label{eq:conformal-vk}
\end{equation}
\end{lemma}
\vspace{-0.5cm}
\paragraph{Calibration and acceptance.}
Let
$\mathcal K_{\mathrm{fin}}\!:=\!
\{k\!\in\![N]:S_{(k)}<+\infty\}$
and, for each finite rank, set
$r_k\!:=\!\max\{0,S_{(k)}\}$.
The margin $r_k$ relaxes the initial and drift bounds to
$b_1^*+r_k$ and $b_3^*+r_k$, while retaining $b_2^*$.
Lemma~\ref{lem:conformal} bounds the out-of-sample probability of
violating these empirical episode conditions by $\bar v_k$.
Together with Lemma~\ref{lem:finiteM}, the first-exact-violation
argument transfers this to the exact adjusted conditions with penalty
$\bar v_k/(1-\beta_0)$, without a union bound over time; the
good-prefix barrier argument then bounds terminal failure.
When $\mathcal K_{\mathrm{fin}}\!\neq\!\varnothing$, select the rank
that maximizes the resulting end-to-end reliability lower bound:
\begin{equation}
k^*\in\arg\max_{k\in\mathcal K_{\mathrm{fin}}}
\Big\{
1-
\frac{
(b_1^*+r_k)(\kappa^*)^T
+(b_3^*+r_k)\cG_T(\kappa^*)
}{b_2^*}
-\frac{\bar v_k}{1-\beta_0}
\Big\}.
\label{eq:select}
\end{equation}
Set
$r^*\!:=\!r_{k^*}$,
$\widetilde b_1^*\!:=\!b_1^*+r^*$, and
$\widetilde b_3^*\!:=\!b_3^*+r^*$.
Simultaneous validity over all ranks permits this data-dependent
selection without an additional confidence penalty.
The final verification residual is
\begin{equation}
\bar q_4
:=
\widetilde b_1^*(\kappa^*)^T
+\widetilde b_3^*\cG_T(\kappa^*)
-b_2^*\Big(
1-\rho-\frac{\bar v_{k^*}}{1-\beta_0}
\Big).
\label{eq:q4bar}
\end{equation}
We accept exactly when $\bar q_4\!\le\!0$.
If $\mathcal K_{\mathrm{fin}}\!=\!\varnothing$, we set
$\bar q_4\!:=\!+\infty$ and verification is inconclusive.

\begin{theorem}[FoSeRL certificate]
\label{thm:conformal}
Fix $K\!\in\!\T$, $b,\alpha\!\in\!\R_{\ge0}$,
$\rho\!\in\!(0,1)$, $\beta\!\in\!(0,1]$, and
$\beta_0\!\in\!(0,1)$.
Under Assumptions~\ref{ass:bounded} and~\ref{ass:noise}, consider the
augmented system of Definition~\ref{def:dtsds} with operational law
$\QQ_{K,b}$ from Definition~\ref{def:jointlaws}.
Let
$B_\phi^*:\cX\times\T\!\to\!(0,\Bmax)$ be a measurable frozen
candidate with
$b_1^*,b_2^*,b_3^*\!\in\!\R_{>0}$,
$\kappa^*\!\in\!(0,1]$, and $b_2^*\!<\!\Bmax$.
Choose $N,M\!\in\!\N$ and
$\epsilon_M\!\in\!\R_{>0}$ satisfying~\eqref{eq:Mbound},
and fix the candidate, constants, operational law, and all
verification-design choices before drawing
$\mathcal D_{\mathrm{cal}}$ as in Definition~\ref{def:score}, using
fresh i.i.d data.
Construct $\bar q_4$ by~\eqref{eq:select}--\eqref{eq:q4bar}, with
$\bar q_4\!:=\!+\infty$ if no finite rank exists.
Then, conditional on all pre-calibration choices, acceptance certifies
the operational target with confidence at least $1-\beta$ over
calibration:
\vspace{-0.1cm}
\begin{equation}
\Prob_{\mathcal D_{\mathrm{cal}}}\!\left[
\bar q_4\le0
\ \Longrightarrow\
\Prob_{(\upsilon,\omega)\sim\QQ_{K,b}}\!\left(
\Dtraj(\upsilon,\omega)\le\alpha
\right)\ge\rho
\right]
\ge1-\beta .
\label{eq:final}
\end{equation}
\end{theorem}
\vspace{-0.1cm}
Algorithm~\ref{alg:foserl} summarizes the workflow: synthesize and freeze
a candidate, calibrate on independent episodes, and accept when
$\bar q_4\le0$; otherwise, verification is inconclusive and does not
establish non-robustness. Confidence is per frozen specification;
adaptive and collection-level reporting are addressed in
Appendix~\ref{app:collection-confidence}.
The rationale for using a barrier witness rather than direct
terminal-event certification is discussed in
Sec.~\ref{sec:main-discussion} and
Appendix~\ref{app:discussion-direct}.

\vspace{-0.2cm}
\section{Experiments}
\label{sec:experiments}
\vspace{-0.2cm}

\begin{wrapfigure}{r}{0.51\textwidth}
\vspace{-0.9em}
\centering
\vspace{-0.2cm}
\refstepcounter{boxedalgorithm}
\begin{algobox}{Algorithm~\theboxedalgorithm: FoSeRL certification\vspace{-0.1cm}}
\label{alg:foserl}
\vspace{-0.1cm}
\noindent
\textbf{Input:}
fixed policy $\pi$; design magnitude $b_{\mathrm{cnd}}$;
tested specification $(K,b,\alpha,\rho)$;
synthesis law $\QQ^{\mathrm{syn}}_{K,b_{\mathrm{cnd}}}$;
operational law $\QQ_{K,b}$;
sample sizes $(N_s,M_s,N,M)$;
confidence parameters $(\beta,\beta_0)$.

\vspace{0.1em}

\begin{enumerate}[
    leftmargin=1.2em,
    label=\textbf{\arabic*.},
    itemsep=1.5pt,
    topsep=1.5pt,
    parsep=0pt,
    partopsep=0pt
]

\item
\textcolor{RoyalBlue}{\textbf{Certificate synthesis.}}
Draw $N_s$ coupled episodes from
$\QQ^{\mathrm{syn}}_{K,b_{\mathrm{cnd}}}$ and the required
$M_s$ successor samples.
Train $B_\phi$ and $(b_1,b_2,b_3,\kappa)$ via
\eqref{eq:loss}, then freeze the candidate.

\item
\textcolor{teal}{\textbf{Certificate verification.}}
With the supplied $M$, before drawing calibration data, fix
$\epsilon_M$ satisfying \eqref{eq:Mbound}, together with all remaining
verification-design choices. Draw $N$ independent calibration episodes under $\QQ_{K,b}$ and use fresh successor batches at the queried tuples, with successor randomness independent of the source episodes.
Compute the scores $S_i$ and $(r_k,\bar v_k)$, select $k^*$ by
\eqref{eq:select}, and construct $\bar q_4$ using
\eqref{eq:q4bar}.

\end{enumerate}

\textbf{\textcolor{Gray}{Return.}}

\centering
\textcolor{WildStrawberry}{
If no finite rank exists or $\bar q_4\!>\!0$, return
\emph{inconclusive}; o.w. certify
$\Prob_{(\upsilon,\omega)\sim\QQ_{K,b}}
(\Dtraj\!\le\!\alpha)\!\ge\!\rho$
with confidence at least $1\!-\!\beta$.}
\vspace{-0.2cm}
\end{algobox}

\vspace{-1em}
\end{wrapfigure}

We instantiate the two-stage procedure in
Algorithm~\ref{alg:foserl} on six continuous-control environments:
Inverted Pendulum (IP), Adaptive Cruise Control (ACC), Unicycle,
Hopper, HalfCheetah, and Ant. The first three are custom control tasks,
with IP and ACC following classical settings
\citep{astrom2000swinging,ames2014control}; Hopper, HalfCheetah, and
Ant use Gymnasium/MuJoCo
\citep{towers2024gymnasium,todorov2012mujoco}.
Across these environments, we certify frozen policies trained with
TD3 \citep{fujimoto2018addressing}, SAC
\citep{haarnoja2018soft}, and PPO
\citep{schulman2017proximal}. Each policy is trained before
certification and remains fixed throughout synthesis and verification.
Complete environment and policy settings, attack generators and norm budgets,
horizons, attack grids, and framework parameters are summarized in
Table~\ref{tab:fixed-case-study-settings} in
Appendix~\ref{app:fixed-experimental-settings}.

\vspace{-0.15cm}
\paragraph{Results.}
\label{sec:cert-results}
We evaluate FoSeRL on all six case studies using
Algorithm~\ref{alg:foserl}. For each tested $(K,b,\alpha)$ specification,
the candidate barrier is synthesized and frozen before independent verification
on fresh calibration episodes. All formal results target $\rho\!=\!0.90$ with
$\beta\!=\!0.001$, and each reported point carries a separate per-specification
confidence guarantee. Figure~\ref{fig:ant-qualitative} first provides a high-level view of what an
accepted FoSeRL budget means in closed-loop behavior.  In the auxiliary
$T\!=\!500$ Ant visualization, SAC and PPO are shown under the paired nominal
execution, their accepted FoSeRL budgets, and larger uncertified attacks, all
at $K/T\!=\!10\%$. In the shown accepted-budget rollouts,
$z(T)/\alpha=D_{\mathrm{traj}}/\alpha<1$, whereas the larger uncertified
attacks finish above the terminal tolerance and visibly collapse. 
\begin{wrapfigure}{r}{0.62\textwidth}
\vspace{-1em}
\centering
\includegraphics[width=\linewidth]{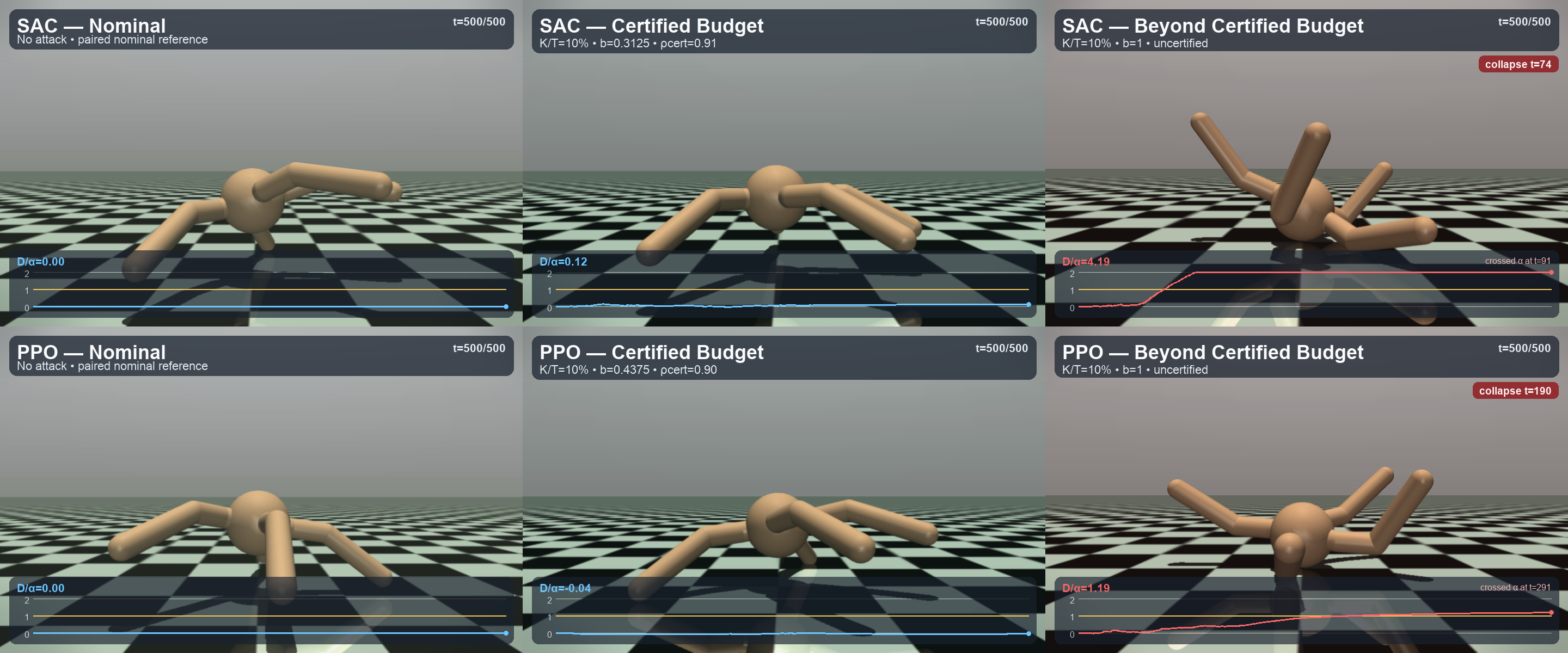}
\vspace{-1.6em}
\caption{\scriptsize
Illustrative Ant rollouts for SAC and PPO in a separate $T\!=\!500$ experiment. 
Each row shows a nominal rollout, an accepted FoSeRL budget, and a larger uncertified attack at $K/T\!=\!10\%$. 
The lower trace shows the running gap $z(t)/\alpha$; at $T$,
$z(T)=D_{\mathrm{traj}}$, and $1$ marks the terminal tolerance. Videos for this and additional case studies are available at
\href{https://figshare.com/s/9bb42c4050f749744af8}{this link}.}
\label{fig:ant-qualitative}
\vspace{-1.5em}
\end{wrapfigure}
This
experiment is used only for qualitative visualization; the quantitative Ant
frontiers and comparisons reported below use the main $T=1000$ campaign.
We report
$\alpha_{\%}\!:=\!100\,\alpha/S_{\mathrm{env}}$, where
$S_{\mathrm{env}}>0$ is a fixed environment-level return scale;
synthesis and verification use the corresponding raw threshold $\alpha$. Complete benchmark, policy,
attack, and certification settings are reported in
Table~\ref{tab:fixed-case-study-settings}, with additional environment details
in Appendix~\ref{app:case-studies}.
Across environments and policy families, FoSeRL yields non-trivial
cardinality--magnitude robustness frontiers. As shown in
Fig.~\ref{fig:main-cert-profiles}, larger degradation tolerances expand the
certified magnitude budget, while denser attacks contract it. For example, at
$\alpha_{\%}=5\%$, the Unicycle--TD3 budget decreases from about $1.0$ at
$K/T=5\%$ to about $0.3$ at $K/T=30\%$; for Ant--SAC at
$\alpha_{\%}=20\%$, it decreases from roughly $0.7$ to $0.23$. Thus,
robustness to isolated perturbations does not directly translate to robustness
under more frequent attacks.
Figure~\ref{fig:cert-vs-emp} shows that the certified frontiers generally
preserve the structure of their empirical design references, but with a
setting-dependent gap: some regimes remain close, whereas others lose a
substantial fraction of the empirical budget after independent verification.
Hence the verification conservatism is not a fixed penalty, but depends on the
policy, attack density, and degradation tolerance. Curves connect individually
verified specifications and do not certify interpolation. Representative
certificates are reported in Table~\ref{tab:main-baseline-results}, with
complete profiles in Appendix~\ref{app:certificate-reporting}.
\begin{figure}[!t]
\centering
\includegraphics[width=\linewidth]{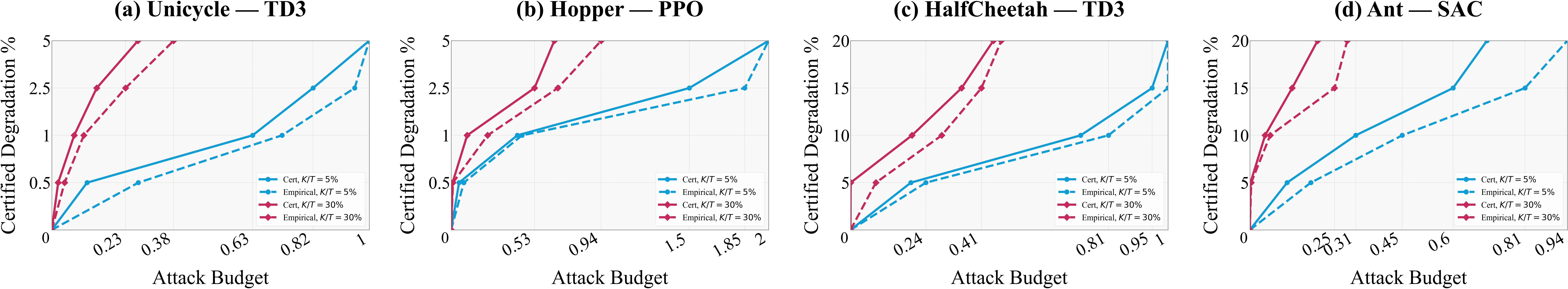}
\vspace{-0.65cm}
\caption{\footnotesize
Certified degradation tolerance versus attack budget across
different settings.
Larger tolerances expand the certified budget, whereas denser temporal
attacks generally contract the robustness frontier.
Markers denote individually verified specifications and connecting
lines are visual guides only. ($\rho=0.90$, $\beta=0.001$)
}
\label{fig:main-cert-profiles}

\vspace{0.2cm}

\includegraphics[width=\linewidth]{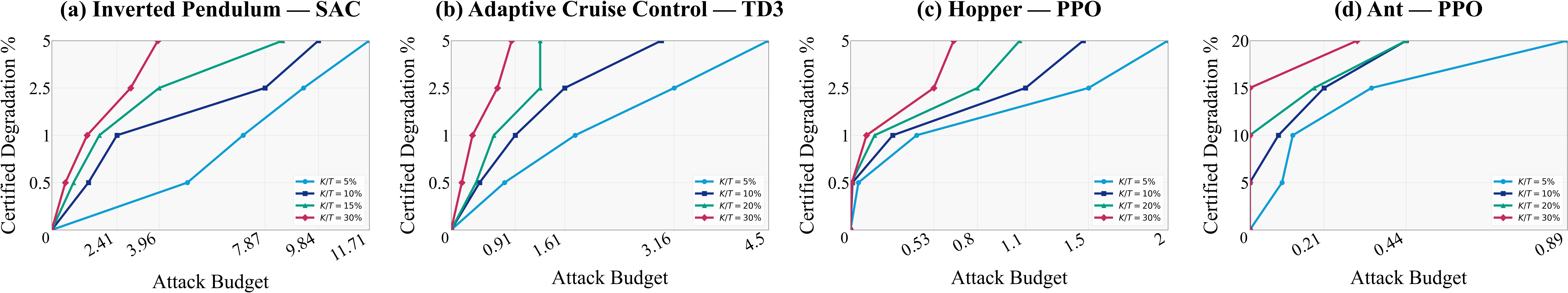}
\vspace{-0.6cm}
\caption{\footnotesize
Certified versus empirical attack-budget profiles across degradation
tolerances.
Solid curves show FoSeRL-certified budgets and dashed curves show the
corresponding empirical design references at matched $K/T$ values.
The empirical--certified gap varies across environments and policies. ($\rho=0.90$, $\beta=0.001$)
}
\label{fig:cert-vs-emp}
\vspace{-0.6cm}
\end{figure}

\newcommand{\envrule}{%
  \arrayrulecolor{black!20}%
  \specialrule{0.3pt}{0.4ex}{0.4ex}%
  \arrayrulecolor{black}%
}
\begin{wraptable}{r}{0.6\textwidth}
\centering
\vspace{-0.45cm}
\caption{\small
Representative formal certificates at selected
$(K/T,\alpha_{\%})$ settings.
Direct certification, policy smoothing (PS), and FoSeRL are compared by
$b_{\mathrm{cert}}$, $\rho_{\mathrm{cert}}$, and runtime (RT, seconds).
Direct uses the corresponding benchmark attack and norm specification,
whereas PS is evaluated under its native $\ell_2$ perturbation setting.
All results use $\rho\!=\!0.90$ and $\beta\!=\!0.001$.
FoSeRL remains non-trivial where PS fails (Failures are denoted by ``--'').}
\label{tab:main-baseline-results}

\vspace{-0.15cm}

\tiny
\setlength{\tabcolsep}{2pt}
\renewcommand{\arraystretch}{1}

\begin{tcolorbox}[
    enhanced,
    sharp corners,
    boxrule=0.45pt,
    colframe=black!75,
    colback=FoSeRLGray!5,
    boxsep=0pt,
    left=4pt,
    right=4pt,
    top=4pt,
    bottom=4pt,
    before skip=0pt,
    after skip=0pt
]

\resizebox{\linewidth}{!}{%
\begin{tabular}{
@{}
l@{\hspace{2pt}}
c@{\hspace{2pt}}
c@{\hspace{2pt}}
c@{\hspace{2pt}}
|
ccc
|
ccc
|
ccc
@{}
}
\toprule

{\multirow{2}{*}{\textbf{Env.}}}
&
{\multirow{2}{*}{\textbf{Policy}}}
&
{\multirow{2}{*}{$\mathbf{K/T}$}}
&
\multicolumn{1}{l}{\multirow{2}{*}{$\boldsymbol{\alpha_{\%}}$}}
&
\multicolumn{3}{c}{\textbf{Direct}}
&
\multicolumn{3}{c}{\textbf{PS}}
&
\multicolumn{3}{c}{\textbf{FoSeRL}}
\\

\cmidrule(lr){5-7}
\cmidrule(lr){8-10}
\cmidrule(lr){11-13}

&
&
&
&
$b_{\mathrm{cert}}$
&
$\rho_{\mathrm{cert}}$
&
RT
&
$b_{\mathrm{cert}}$
&
$\rho_{\mathrm{cert}}$
&
RT
&
$b_{\mathrm{cert}}$
&
$\rho_{\mathrm{cert}}$
&
RT
\\
\midrule

\multirow{3}{*}{IP}
& SAC & 8\%  & 5\%
& 4 & 0.92 & 14.1
& 0.0132 & 0.9 & 11.2
& 3.00 & 0.92 & 10.5
\\

& SAC & 13.5\% & 1\%
& 0.82 & 0.91 & 8.6
& -- & 0.9 & --
& 0.65 & 0.92 & 12.2
\\

& SAC & 20\% & 5\%
& 2.84 & 0.91 & 9.2
& 0.0083 & 0.9 & 11.2
& 2.50 & 0.92 & 15.5
\\

\envrule

\multirow{3}{*}{Hopper}
& TD3 & 10\% & 2.5\%
& 1.57 & 0.91 & 28.9
& 0.0014 & 0.9 & 357.3
& 1.32 & 0.93 & 96.4
\\
& SAC & 10\% & 2.5\%
& 1.31 & 0.94 & 48.5
& 0.0023 & 0.9 & 438.4
& 1.12 & 0.94 & 119.6
\\
& PPO & 10\% & 2.5\%
& 1.29 & 0.93 & 31.1
& -- & 0.9 & --
& 1.17 & 0.93 & 54.3
\\

\envrule

\multirow{3}{*}{HCH}
& TD3 & 30\% & 20\%
& 0.63 & 0.92 & 32.1
& 0.0003 & 0.9 & 345.3
& 0.45 & 0.92 & 139.5
\\
& SAC & 30\% & 20\%
& 0.26 & 0.94 & 47.7
& 0.0013 & 0.9 & 321.3
& 0.22 & 0.94 & 134.1
\\
& PPO & 30\% & 20\%
& 0.92 & 0.96 & 33.4
& -- & 0.9 & --
& 0.91 & 0.92 & 145
\\

\envrule

\multirow{3}{*}{Ant}
& TD3 & 20\% & 20\%
& 0.63 & 0.92 & 62.1
& -- & -- & --
& 0.49 & 0.91 & 94.7
\\
& SAC & 20\% & 20\%
& 0.39 & 0.9 & 81.1
& -- & -- & --
& 0.31 & 0.91 & 138.6
\\
& PPO & 20\% & 20\%
& 0.55 & 0.91 & 52.9
& -- & -- & --
& 0.44 & 0.90 & 95.6
\\

\bottomrule
\end{tabular}%
}

\end{tcolorbox}

\vspace{-1.5em}
\end{wraptable}
In addition, Table~\ref{tab:main-baseline-results} compares FoSeRL with direct
terminal-event certification and policy smoothing (PS)
\citep{kumar2022policy,mu2024reward}. Direct certification uses exactly the same benchmark attack and norm
specification as FoSeRL and can yield larger scalar budgets by bypassing
the sufficient barrier conditions, without requiring barrier synthesis or
conditional successor queries. Nevertheless, FoSeRL retains roughly
$71$--$99\%$ of the direct certified budget in the representative settings,
while additionally returning a state--time barrier witness and residual
diagnostics that expose how the certificate is limited along the closed-loop
evolution.
No existing external formal baseline matches all elements of our setting
simultaneously (detailed comparison provided in Appendix~\ref{app:rel}). We therefore adapt PS, the closest formal alternative, to the
same action channel, paired degradation event, and
reliability and confidence
targets, while retaining its native $\ell_2$ perturbation geometry. The
comparison should therefore be interpreted as an aligned external baseline
rather than a fully matched threat-model comparison. Under this alignment,
FoSeRL certifies orders-of-magnitude larger nonzero budgets in the comparable
settings, while PS is unable to certify a nonzero budget in several cases.
Detailed protocols, applicability conditions, and complete results are given
in Appendix~\ref{app:rs-comparison}.

\newpage
\section{Discussion}
\label{sec:main-discussion}
\vspace{-0.2cm}

FoSeRL provides post-hoc guarantees for the cumulative performance loss
of an unchanged policy under temporally sparse action attacks. Its
paired-state construction makes return degradation a terminal quantity,
connecting sequential robustness to stochastic barrier analysis and
independent finite-sample verification. Experiments across six
continuous-control environments and three RL policy families demonstrate
non-trivial cardinality--magnitude robustness profiles and larger
certified budgets than the evaluated smoothing baseline at the reported
settings.

\vspace{-0.22cm}
\paragraph{Robustness at the episode level.}
The joint budget $(K,b)$ distinguishes how many actions can be perturbed
from the strength of each perturbation, while cumulative degradation
captures consequences that persist after the attack ends. Pairing each
attacked execution with its nominal counterpart under shared initial
conditions and random inputs provides a consistent reference for this
loss. It avoids variability introduced by an independently randomized
comparison without removing stochastic variation in the attack's
consequences, and gives exactly zero degradation when no attack occurs.
The resulting tail criterion addresses excessive episode-level loss,
rather than only average performance or instantaneous action changes.
Coupling is therefore part of the robustness specification itself,
not merely a device for improving estimation.

\vspace{-0.22cm}
\paragraph{Policy- and environment-agnostic certification.}
FoSeRL separates certification from how the policy is obtained.
Under the stated execution and sampling assumptions, its argument does
not depend on a particular RL algorithm, policy architecture, policy
differentiability, or analytical family of dynamics. Learned and
analytically specified policies, including deployment-time randomization,
fit the same formulation. Relative to the robust-training,
policy-smoothing, and model-based approaches reviewed in
Appendix~\ref{app:rel}, its distinguishing feature is that certification
requires neither a robustness-specific policy modification nor a
separately learned dynamics surrogate. Instead, coupled rollouts and
fresh conditional successors support certification of the
\emph{policy as deployed}. This makes the framework applicable when
the original training pipeline is unavailable and complementary to
robust policy learning: a policy can be assessed regardless of whether
robustness was considered during its training.

\vspace{-0.22cm}
\paragraph{The value of a dynamical certificate.}
Direct terminal testing and FoSeRL address the same operational event
but provide different certifying objects. A one-sided binomial bound
on independent paired episodes can certify the terminal probability
without barrier synthesis or successor queries. FoSeRL additionally
provides a state--time witness linking initial conditions, expected
one-step evolution, and terminal separation. Its residuals and
reliability decomposition identify which certificate conditions and
sampled episode phases limit witness quality, supporting targeted
refinement. These are diagnostics, not certified conditional failure
probabilities at arbitrary states. Direct testing can yield tighter
budgets at lower cost by bypassing the sufficient barrier conditions.
The trade-off is therefore scalar tightness and computation versus
dynamical structure, not formal versus empirical validity.
Appendix~\ref{app:discussion-direct} develops this distinction.

\vspace{-0.22cm}
\paragraph{Formal assurance from finite samples.}
Separating synthesis from verification allows flexible neural candidate
construction while grounding validity in an independent statistical
test. Successor approximation uncertainty and out-of-sample violations
are incorporated into the end-to-end guarantee, rather than inferred
from training loss or empirical success rates. Episode-level
calibration accommodates temporal dependence within each trajectory,
and the empirical-to-exact transfer avoids an additional union-bound
penalty over time. Uniform satisfaction of the ideal barrier conditions
covers every admissible precommitted attack; the data-driven procedure
certifies the declared attack--episode law
(Theorem~\ref{thm:conformal}). This is a trade-off in attack coverage
that enables certification from finite generative data without global
verification of the closed-loop inequalities, while retaining a formal
guarantee at the prescribed confidence level.

\vspace{-0.22cm}
\paragraph{Extensibility and future directions.}
The common pipeline across different dynamics, horizons, action
dimensions, and policy families supports FoSeRL's portability.
Its modular design offers complementary improvements: better witness
synthesis can reduce representational conservatism, while additional
successor samples trade computation for smaller approximation margins.
 New specifications can be added incrementally using fresh
calibration and the confidence-control protocols of
Appendix~\ref{app:collection-confidence}, without predeclaring the
entire future experimental family. Alongside richer synthesis and more
efficient verification, a natural extension is certificate-aware policy
improvement, using witness diagnostics to guide refinement while
independently re-certifying each updated policy.

\newpage
\subsubsection*{Ethics Statement}
Our work provides a formal framework for assessing the robustness of deployed
decision policies against adversarial action perturbations. Although such
analyses could, in principle, be misused to identify vulnerable operating
regimes, we believe that exposing and quantifying these vulnerabilities is
important for the safe deployment of learned control policies. The societal
benefits of advancing robustness certification therefore outweigh the potential
risks, and we do not anticipate any immediate misuse arising from the methods or
experiments presented in this work. No human-subject or personal data are used.

\subsubsection*{Reproducibility Statement}
We have taken considerable care to ensure the reproducibility of our results.
The complete assumptions, proofs, and certification procedure are provided in
the main paper and appendix. Detailed descriptions of the environments,
policies, attack models and norm budgets, certificate architectures, and
synthesis and verification sample budgets are reported in
Sec.~\ref{sec:experiments} and Appendix~\ref{app:expdetails}. Qualitative
rollout videos for the reported case studies are provided through the link in Fig.~\ref{fig:ant-qualitative}. The complete codebase, along with the configuration files for every experiment, is available at
\href{https://figshare.com/s/9bb42c4050f749744af8}{this link}.

\subsubsection*{AI Use Statement}
This paper was entirely written by the authors. Large language models (LLMs), were used solely for final-stage language editing and polishing, without contributing to the scientific content or experimental results.

\bibliography{references}
\bibliographystyle{plainnat}

\newpage
\appendix

\section*{Appendix Overview}

\begingroup
\small
\setlength{\parindent}{0pt}
\setlength{\parskip}{0pt}

\newcommand{\appsectionentry}[2]{
  \par\addvspace{0.5em}
  \noindent\textbf{\ref{#1}\quad #2}\hfill\pageref{#1}\par
}
\newcommand{\appsubsectionentry}[2]{
  \noindent\hspace*{1em}\makebox[3em][l]{\ref{#1}}
  #2\dotfill\pageref{#1}\par
}

\appsectionentry{app:notation}{Notation}

\appsectionentry{app:formulation}{Additional Formulation and Verification Details}
\appsubsectionentry{app:attack-admissibility}{Attack admissibility and action projection}
\appsubsectionentry{app:coupling-semantics}{Coupling semantics}
\appsubsectionentry{app:frontier-monotonicity}{Monotonicity of the ideal robustness frontier}
\appsubsectionentry{app:analytical-cases}{Analytical boundary cases and zero tolerance}
\appsubsectionentry{app:common-barrier-limitation}{Conservatism of the barrier witness}
\appsubsectionentry{app:verification-planning}{Verification planning and sample cost}
\appsubsectionentry{app:collection-confidence}{Tested specifications and collection-level confidence}

\appsectionentry{app:proofs}{Proofs}
\appsubsectionentry{app:proof-augmented}{Augmented-state identity}
\appsubsectionentry{app:proof-ideal}{Proof of Theorem~\ref{thm:ideal}}
\appsubsectionentry{app:proof-finiteM}{Proof of Lemma~\ref{lem:finiteM}}
\appsubsectionentry{app:proof-calibration}{Proof of Lemma~\ref{lem:conformal}}
\appsubsectionentry{app:proof-foserl}{Proof of Theorem~\ref{thm:conformal}}

\appsectionentry{app:expdetails}{Additional Experimental Details and Results}
\appsubsectionentry{app:benchmark-formulation}{Benchmark formulation}
\appsubsectionentry{app:policy-families}{Policy families}
\appsubsectionentry{app:case-studies}{Case studies}
\appsubsectionentry{app:certificate-parameterization}{Certificate parameterization}
\appsubsectionentry{app:fixed-experimental-settings}{Fixed experimental settings}
\appsubsectionentry{app:certificate-workflow}
{Certificate construction and reporting workflow}
\appsubsectionentry{app:certificate-reporting}
{Complete FoSeRL certification results}
\appsubsectionentry{app:rs-comparison}{Comparison with Policy Smoothing}

\appsectionentry{sec:discussion}{Additional Discussion}
\appsubsectionentry{app:discussion-training}
{From training-time robustness to long-horizon certification}
\appsubsectionentry{app:discussion-calibration}
{Episode-level calibration and statistical validity}
\appsubsectionentry{app:discussion-probability-levels}
{Reliability, confidence, and violation levels}
\appsubsectionentry{app:discussion-direct}
{Barrier witnesses versus direct terminal certification}
\appsubsectionentry{app:discussion-scope}
{Generality, simulator access, and attack coverage}
\appsubsectionentry{app:discussion-frontiers}
{Reading robustness frontiers and selecting specifications}
\appsubsectionentry{app:discussion-performance}
{Performance robustness versus pathwise safety}
\appsubsectionentry{app:discussion-refinement}
{Certificate refinement and policy improvement}
\appsubsectionentry{app:discussion-limitations}
{Limitations and outlook}

\appsectionentry{app:rel}{Positioning, Baseline Selection, and Extended Comparisons}

\endgroup

\vspace{2cm}
\section{Notation}
\label{app:notation}

\noindent
We denote the sets of real, nonnegative real, and strictly positive real numbers by $\mathbb{R}$, $\mathbb{R}_{\ge 0}$, and $\mathbb{R}_{>0}$, respectively. The set of positive integers is denoted by $\mathbb{N}$, and we define $\mathbb{N}_0\!:=\!\{0\}\cup\mathbb{N}$. For integers $a,b\!\in\!\mathbb{N}_0$ with $a\le b$, we write $\mathbb{N}_{a:b}\!:=\!\{a,a+1,\ldots,b\}$, and for $N\!\in\!\mathbb{N}$ we use $[N]\!:=\!\{1,\ldots,N\}$. We define the time-index sets $\T:=\{0,\ldots,T\}$ and $\mathbb{T}'\!:=\!\{0,\ldots,T-1\}$. For a scalar $x\!\in\!\mathbb{R}$, $|x|$ and $\lceil x\rceil$ denote its absolute value and ceiling, respectively. The $d$-dimensional Euclidean space is denoted by $\mathbb{R}^d$, and for $x\!\in\!\mathbb{R}^d$ and $p\!\in\![1,\infty]$, $\|x\|_p$ denotes the $\ell_p$ norm, with $\|x\|_\infty\!:=\!\max_{i\!\in\![d]}|x_i|$. For a finite set $A$, $|A|$ denotes its cardinality. For a probability measure $\mu$, $\operatorname{supp}(\mu)$ denotes its support, while $\mathbb{E}[\cdot]$ and $\mathbb{P}(\cdot)$ denote expectation and probability with respect to the underlying probability space; conditional quantities are written in the standard form $\mathbb{E}[\cdot\mid\cdot]$ and $\mathbb{P}(\cdot\mid\cdot)$. The product of probability measures is denoted by $\otimes$, and $X_i\stackrel{\mathrm{i.i.d.}}{\sim}P$ indicates that the random variables $X_i$ are independent and identically distributed according to $P$. For an event $\mathcal E$, $\mathbf{1}_{\mathcal E}$ denotes its indicator. We use $\operatorname{ReLU}(x)\!:=\!\max\{0,x\}$ and $\operatorname{sigmoid}(x)\!:=\!(1+e^{-x})^{-1}$.

\clearpage

\section{Additional Formulation Details}
\label{app:formulation}
This section supplements Secs.~\ref{sec:setup}--\ref{sec:datadriven}
with attack admissibility, coupling semantics, ideal-frontier
monotonicity, and boundary cases of the certification framework.

\subsection{Attack admissibility and action projection}
\label{app:attack-admissibility}

The class $\Aadv(K,b)$ in Definition~\ref{def:attack} constrains the
precommitted schedule and perturbation vectors, not the state-dependent
command after projection. Along the attacked branch, the executed
action is $\ubar^a(t;\upsilon)=\Pi_{\cA}(\pi(x^a(t),w(t))+\Delta_t)$,
where $\Delta_t=0$ outside $\Tatk$. Thus an admissible attack is
specified before the episode begins, while the executed action remains
in $\cA$ for every realized state.
We do not test attack membership by requiring
$\pi(x^a(t),w(t))+\Delta_t\in\cA$ along the realized trajectory: such a
membership test would depend on that trajectory rather than only on
$(\Tatk,\Delta)$. Instead, projection is part of the attacked
closed-loop dynamics, and the budget $b$ always refers to
$\|\Delta_t\|_p$ before projection.
The norm used in a projection bound must be specified. For a closed
convex action set, Euclidean projection is nonexpansive in the
Euclidean norm. For each realized $u=\pi(x,w)\in\cA$,
\begin{equation}
\|\Pi_{\cA}(u+\delta)-u\|_2
=\|\Pi_{\cA}(u+\delta)-\Pi_{\cA}(u)\|_2
\le\|\delta\|_2.
\label{eq:app-projection-l2}
\end{equation}
For a box action set, projection is coordinatewise clipping and
$|[\Pi_{\cA}(u+\delta)]_j-u_j|\le|\delta_j|$ for every coordinate
$j$. Consequently, the same displacement bound holds in each of the
$\ell_1$, $\ell_2$, and $\ell_\infty$ norms for boxes. We do not
assert this all-norm property for arbitrary closed convex sets.
The attack class and its certification criterion remain defined by
the pre-projection budget and the specified projected dynamics.

\subsection{Coupling semantics}
\label{app:coupling-semantics}

The tail criterion in Definition~\ref{def:certperf} depends on the
joint coupling of the nominal and attacked executions, not only on
their trajectory marginals. To make this distinction explicit, fix
an attack $\upsilon$ and let $\mathsf C$ be any coupling with the
prescribed nominal and attacked marginals. Linearity of expectation
gives
\begin{equation}
\E_{\mathsf C}\!\left[
R_{\mathrm{nom}}-R_{\mathrm{att}}
\right]
=
J(\pi)-\E_\omega\!\left[
R_{\mathrm{att}}(\upsilon,\omega)
\right].
\label{eq:app-coupling-mean}
\end{equation}
Thus the mean degradation is coupling-invariant, whereas
$\Prob_{\mathsf C}(R_{\mathrm{nom}}-R_{\mathrm{att}}\!\le\!\alpha)$
generally is not. A coupling must therefore be fixed to define the
paired tail event.

Definition~\ref{def:coupling} uses the synchronous coupling induced
by the declared execution representation $(\pi,F,P_w)$.
Both branches share $x(0)$ and the complete random-input sequence,
including deployment-time policy randomization, while evaluating the
policy at their own states. Hence the certification statement is
relative to this fixed shared-randomness semantics, chosen before
synthesis and verification; the marginal transition kernel and policy
action law alone do not determine the paired tail distribution.

The synchronous coupling also gives an exact nominal anchor.
If $K\!=\!0$ or $b\!=\!0$ and
$x^0(t)\!=\!x^a(t)$, the shared $w(t)$ produces the same policy action
in both branches. Projection leaves that admissible action unchanged,
and the same $F$ and $w(t)$ then produce equal next states and rewards.
Induction from the common initial state yields
$\Dtraj(\upsilon,\omega)\!\equiv\!0$ for every $\omega$.
This is an analytical property of the declared coupling, not an
empirical calibration result.

The need to specify the coupling cannot be removed by specifying only
the marginals. For example, take $T\!=\!2$, $\gamma\!=\!1$,
$x(0)\!=\!0$, $\pi\!=\!0$, $\cS\!=\!\{0,1\}$,
$\cA\!=\![0,1]$, $r(x,u)\!=\!x$, and fair Bernoulli noise $w$.
Let $F_1(x,u,w)\!=\!w$, while
$F_2(x,u,w)\!=\!w$ for $u\!<\!1/2$ and
$F_2(x,u,w)\!=\!1-w$ otherwise. Both realize the same Bernoulli
transition kernel. A time-zero perturbation from zero to one gives
$\Dtraj\!=\!0$ under $F_1$ but
$\Dtraj\!=\!2w(0)-1$ under $F_2$.
At $\alpha\!=\!0$, the corresponding paired reliabilities are
one and one-half despite identical trajectory marginals.

\subsection{Monotonicity of the ideal robustness frontier}
\label{app:frontier-monotonicity}

Fix the policy, horizon, attack norm, coupling, and the pair
$(\alpha,\rho)$. The attack classes are nested:
\begin{align}
\Aadv(K,b)&\subseteq\Aadv(K,b'), && 0\le b\le b',
\label{eq:app-magnitude-inclusion}\\
\Aadv(K,b)&\subseteq\Aadv(K',b), && 0\le K\le K'\le T.
\label{eq:app-cardinality-inclusion}
\end{align}
Since Definition~\ref{def:certperf} quantifies over every admissible
attack, satisfaction for a larger class implies satisfaction for each
contained class.
Let $\mathcal B_K:=\{b\in\R_{\ge0}:\text{\eqref{eq:certified-forall}
holds}\}$ denote the set of magnitudes satisfying the ideal criterion.
By~\ref{eq:app-magnitude-inclusion}, this set is downward closed:
if $b'\in\mathcal B_K$, then $b\in\mathcal B_K$ for every
$0\le b\le b'$. It is nonempty because the zero-attack property
implies $0\in\mathcal B_K$. Moreover,
\ref{eq:app-cardinality-inclusion} gives
$\mathcal B_{K'}\subseteq\mathcal B_K$ when $K\le K'$, and therefore
\begin{equation}
b_{\mathrm{cert}}^*(K')\le b_{\mathrm{cert}}^*(K)
\qquad\text{for }K\le K'.
\label{eq:app-frontier-monotonicity}
\end{equation}
This monotonicity concerns the ideal per-attack objective and does not
require a sampled-prefix construction.
The distributional guarantee of Theorem~\ref{thm:conformal} has a
different scope. Under
$\QQ_{K,b}=\UU_{K,b}\otimes\Prob_\omega$, changing $K$ or $b$ may
change the attack sampling law; set inclusion alone therefore does
not imply monotonicity of the distributional success probability.
In particular, a certificate at one $(K,b)$ does not automatically
certify other budgets under different operational laws. This is a
distinction between the ideal and distributional criteria, not merely
a finite-sample effect.

\subsection{Conservatism of the barrier witness}
\label{app:common-barrier-limitation}

The conditions of Definition~\ref{def:barrier} are sufficient but not
necessary. Conservatism arises both from requiring one common barrier
across admissible attacks and from the pointwise initial-set bound.

\paragraph{Common witness across attacks.}
Let $T\!=\!2$, $K\!=\!b\!=\!\gamma\!=\!1$,
$\cS\!=\!\cA\!=\![-1,1]$, $x(0)\!=\!0$, and $\pi\!=\!0$.
Take $F(x,u,w)\!=\!w$ with fair $w\!\in\!\{-1,1\}$ and
$r(x,u)\!=\!-(u^2+xu)/2$, and set $\alpha\!=\!\rho\!=\!1/2$.
A time-zero perturbation $u$ gives
$\Dtraj\!=\!u^2/2\!\le\!1/2$.
A time-one perturbation gives $(u^2+su)/2$ at
$x(1)\!=\!s\!\in\!\{-1,1\}$; for each precommitted $u$, at most one
sign exceeds $1/2$. Thus every admissible precommitted attack satisfies
the target.

At the shared time-one states $(s,s,0)$, however, a common barrier must
also accommodate the respective attacks $u\!=\!s$, which give terminal
gap one. Hence
$\kappa B((s,s,0),1)+b_3\!\ge\!b_2$ for both signs, and the initial
drift inequality implies
$b_1\kappa^2+b_3(1+\kappa)\!\ge\!b_2$.
The resulting ideal barrier bound is nonpositive even though the
universal robustness property holds. This is conservatism of the common
witness, not a counterexample to Theorem~\ref{thm:ideal}.

\paragraph{Pointwise initial condition.}
The proof of Theorem~\ref{thm:ideal} only requires
$\E_{x\sim\mu_0}[B((x,x,0),0)]\!\le\!b_1$.
Definition~\ref{def:barrier} instead imposes the stronger pointwise
condition $B(\xhat,0)\!\le\!b_1$ on $\cX_0$.
We retain this form because the verifier directly calibrates the
corresponding residual $q_1$; no separate empirical test of the initial
expectation is introduced.

\paragraph{Unit reliability.}
For $T\!\ge\!1$ and
$b_1,b_3\!\in\!\R_{>0}$,
$b_1\kappa^T+b_3\cG_T(\kappa)\!>\!0$.
Hence the sufficient bound in Theorem~\ref{thm:ideal} cannot certify
$\rho\!=\!1$, even when the underlying robustness property holds.
This limitation concerns the certificate conditions, not the property
itself.
\subsection{Analytical boundary cases and zero tolerance}
\label{app:analytical-cases}

Some specifications are resolved analytically before invoking the
data-driven procedure. If $K\!=\!0$ or $b\!=\!0$, synchronous coupling
gives $\Dtraj(\upsilon,\omega)\!\equiv\!0$ pathwise, so
Definition~\ref{def:certperf} holds for every $\alpha\!\ge\!0$ and
$\rho\!\le\!1$. If $\alpha\!\ge\!2S_{\max}$, bounded rewards imply
$\Dtraj(\upsilon,\omega)\!\le\!2S_{\max}\!\le\!\alpha$ for every
admissible attack and episode. These are analytical certificates, not
finite-sample acceptances.

\paragraph{Zero-tolerance boundary.}
Assume $S_{\max}\!>\!0$ and that $B$ is continuous in $z$.
At $\alpha\!=\!0$, global terminal separation for all $z\!>\!0$
implies $B((x,x,0),T)\!\ge\!b_2$ by continuity.
Under zero attack, the actual terminal state has this form; hence the
ideal recursion cannot yield a positive reliability bound through such
a globally separating continuous candidate.
The analytical zero-attack certificate above bypasses this restriction.
The data-driven theorem instead controls violations under the
operational episode law and does not require global terminal separation.

\subsection{Verification planning and sample cost}
\label{app:verification-planning}

For one fixed tested specification, formal verification uses $N$
paired calibration episodes and $NTM$ augmented successor queries when
all $T$ time indices require Monte Carlo successor evaluation; each
augmented query advances both branches. A complete synthesis pass uses
$N_sTM_s$ successor queries, with additional cost if
successor batches are regenerated during optimization.

For a fixed $M$, the choice
$\epsilon_M\!=\!\Bmax\sqrt{\log(1/\beta_0)/(2M)}$
saturates the sufficient condition in~\eqref{eq:Mbound}.
The conformal calibration also imposes a finite-$N$ statistical floor.
Since the smallest rank bound is
$\bar v_N\!=\!1-(\beta/N)^{1/N}$, define
$c_{\mathrm{stat,min}}\!:=\!
[1-(\beta/N)^{1/N}]/(1-\beta_0)$.
Because the normalized barrier contribution is strictly positive,
acceptance requires
$c_{\mathrm{stat,min}}\!<\!1-\rho$.
This condition can be checked before calibration and expensive
successor sampling; failure means that $N$ must be increased or no
certificate can be obtained with that verification design.
For a planned positive dynamical contribution $d$, the stronger
planning condition is
$c_{\mathrm{stat,min}}\!\le\!1-\rho-d$.

Avoiding a union bound over the horizon does not eliminate horizon
dependence. For rank $k$, the calibration margin contributes
$r_k[(\kappa^*)^T+\cG_T(\kappa^*)]/b_2^*$ to the normalized
barrier term, reducing to $(T+1)r_k/b_2^*$ when $\kappa^*=1$.

\subsection{Scope across tested specifications and collection-level confidence}
\label{app:collection-confidence}

Theorem~\ref{thm:conformal} gives a separate confidence statement for
each frozen tested specification, including its policy, coupling,
attack budget, tolerance, operational law, candidate, and sampling
design. New specifications may therefore be added incrementally using
fresh calibration data; the complete future experimental family need
not be fixed in advance.

Simultaneous validity over conformal ranks within one specification does
not imply simultaneous validity across different specifications.
Consequently, a displayed collection of individually certified points
carries no automatic family-wise guarantee, interpolation guarantee, or
downward-closure claim when the operational law changes.
Likewise, selecting the largest accepted budget after repeated formal
tests does not preserve the same nominal confidence for that selected
output. For a single selected-budget claim at a stated confidence,
candidate selection may use design or previous search data, but the
selected specification must then be frozen and assessed using fresh
final calibration data. Repeating that final test until acceptance
requires additional confidence control.

For a simultaneous claim over a finite predeclared collection of $J$
specifications, fix all candidates and designs before calibration and
choose $\beta_j>0$ with
$\sum_{j=1}^{J}\beta_j\le\beta_{\mathrm{total}}$.
Applying Theorem~\ref{thm:conformal} to specification $j$ with
$\beta=\beta_j$ and taking a union bound yields probability at least
$1-\beta_{\mathrm{total}}$ that every accepted specification satisfies
its own operational target. Independence across specifications is not
required, provided each test satisfies its own calibration conditions.

The entire future family also need not be predeclared.
For sequential testing, one may fix a summable allocation such as
$\beta_j=\beta_{\mathrm{total}}/[j(j+1)]$.
Before attempt $j$, freeze its specification using only previously
available information and draw fresh calibration and successor samples
conditionally independently of that information.
The conditional false-acceptance probability at attempt $j$ is then at
most $\beta_j$; summing these bounds gives the same
$1-\beta_{\mathrm{total}}$ collection-level guarantee over all
attempts, including under stopping. Failed attempts still consume their
assigned confidence budget. These stronger protocols do not
retrospectively apply to runs performed at other confidence levels and
do not certify untested or interpolated budgets.

\section{Proofs}
\label{app:proofs}

This appendix establishes the theoretical results used in the main
text. We first establish the
augmented-state identity and prove the ideal sequential robustness
guarantee. We then prove the two statistical ingredients of the primary
data-driven verification method: the finite-sample supermartingale
approximation and the conformal calibration guarantee. These results are combined
directly with the stochastic barrier argument in the proof of the final
data-driven FoSeRL certificate.

\subsection{Augmented-state identity}
\label{app:proof-augmented}

By Definition~\ref{def:dtsds}, the third component of the augmented
state satisfies $z(0)=0$ and, for each $t\in\T'$,
\begin{equation}
z(t+1)-z(t)
=
\gamma^t
\left[
r(x^0(t),u^0(t))
-r(x^a(t),\ubar^a(t;\upsilon))
\right].
\label{eq:proof-zincrement}
\end{equation}
Summing~\ref{eq:proof-zincrement} over
$t=0,\ldots,T-1$ gives
\begin{equation}
\begin{aligned}
z(T)
&=
\sum_{t=0}^{T-1}
\gamma^t
\left[
r(x^0(t),u^0(t))
-r(x^a(t),\ubar^a(t;\upsilon))
\right] \\
&=
R_{\mathrm{nom}}(\omega)
-
R_{\mathrm{att}}(\upsilon,\omega)
=
\Dtraj(\upsilon,\omega).
\end{aligned}
\label{eq:proof-zT}
\end{equation}
Hence the terminal augmented coordinate is exactly the paired
trajectory-level degradation.
For completeness, fix any $\upsilon\in\Aadv(K,b)$.
Since the nominal branch depends only on $\omega$ and the fixed
policy $\pi$,
\begin{equation}
\E_\omega\!\left[R_{\mathrm{nom}}(\omega)\right]
=
J(\pi).
\end{equation}
Taking expectation in~\ref{eq:proof-zT} over $\omega$ therefore gives
\begin{equation}
\E_\omega\!\left[\Dtraj(\upsilon,\omega)\right]
=
J(\pi)
-
\E_\omega\!\left[R_{\mathrm{att}}(\upsilon,\omega)\right],
\end{equation}
which is the mean identity stated after
Definition~\ref{def:coupling}.

\subsection{Proof of Theorem~\ref{thm:ideal}}
\label{app:proof-ideal}

\begin{proof}
Let $B$ satisfy Definition~\ref{def:barrier} with constants
$b_1,b_2,b_3\in\R_{>0}$ and $\kappa\in(0,1]$. Fix an arbitrary
$\upsilon\in\Aadv(K,b)$. All probabilities and expectations below are
over the episode randomness $\omega$, conditional on the
fixed policy and execution representation.
Define
$B_t:=B(\xhat(t),t)$ and the filtration
\begin{equation}
\mathcal F_t
:=
\sigma
\big(
x(0),w(0),\ldots,w(t-1)
\big),
\qquad t\in\T .
\label{eq:proof-ideal-filtration}
\end{equation}
For fixed $\upsilon$, the augmented state $\xhat(t)$ and hence $B_t$
are $\mathcal F_t$-measurable.
Because $\xhat(0)\in\cX_0$ almost surely, the initial barrier
condition gives
\begin{equation}
\E_\omega[B_0]
\le b_1 .
\label{eq:proof-ideal-initial}
\end{equation}
Since $\xhat(t)\in\cX_t^\upsilon$ almost surely and $w(t)$ is independent of $\mathcal F_t$ with law $P_w$ by Assumption~\ref{ass:noise}, the third condition in Definition~\ref{def:barrier} and the tower property give
\begin{equation}
\begin{aligned}
\E_\omega[B_{t+1}]
&=
\E_\omega\!\left[
\E_\omega\!\left[
B_{t+1}\mid\mathcal F_t
\right]
\right]\\
&\le
\kappa\E_\omega[B_t]+b_3,
\qquad t\in\T'.
\end{aligned}
\label{eq:proof-ideal-recursion}
\end{equation}
The conditional barrier inequality implies that the process
$\kappa^{T-t}B_t+b_3\sum_{j=0}^{T-t-1}\kappa^j$,
$t\in\T$, is a nonnegative supermartingale with respect to
$\{\mathcal F_t\}_{t\in\T}$, where an empty sum is zero.
Iterating~\ref{eq:proof-ideal-recursion} and using
\ref{eq:proof-ideal-initial} yields
\begin{equation}
\begin{aligned}
\E_\omega[B_T]
\le
b_1\kappa^T
+
b_3\cG_T(\kappa).
\end{aligned}
\label{eq:proof-ideal-terminal}
\end{equation}
By~\ref{eq:proof-zT}, the terminal unsafe event is
\begin{equation}
\left\{
\xhat(T)\in\cX_u^\alpha
\right\}
=
\left\{
\Dtraj(\upsilon,\omega)>\alpha
\right\}.
\label{eq:proof-ideal-unsafe}
\end{equation}
On this event the terminal separation condition gives
$B_T\ge b_2$. Since $B_T\ge0$, Markov's inequality and
\ref{eq:proof-ideal-terminal} imply
\begin{equation}
\Prob_\omega
\left(
\Dtraj(\upsilon,\omega)>\alpha
\right)
\le
\frac{
b_1\kappa^T+b_3\cG_T(\kappa)
}{b_2}.
\label{eq:proof-ideal-failure}
\end{equation}
Taking complements gives
\begin{equation}
\Prob_\omega
\left(
\Dtraj(\upsilon,\omega)\le\alpha
\right)
\ge
1-
\frac{
b_1\kappa^T+b_3\cG_T(\kappa)
}{b_2}.
\label{eq:proof-ideal-safe}
\end{equation}
Finally, the scalar condition in Definition~\ref{def:barrier} implies
\begin{equation}
1-
\frac{
b_1\kappa^T+b_3\cG_T(\kappa)
}{b_2}
\ge\rho .
\label{eq:proof-ideal-rho}
\end{equation}
Because $\upsilon$ was arbitrary, the bound holds for every
$\upsilon\in\Aadv(K,b)$, proving~\ref{eq:idealbound} and completing
the proof.
\end{proof}

\subsection{Proof of Lemma~\ref{lem:finiteM}}
\label{app:proof-finiteM}

\begin{proof}
Fix $(\xhat,\upsilon,t)\in\mathcal Z_3$. The candidate and verification choices are fixed; the only
randomness considered here is that of the fresh successor samples. Let
$Y_m:=B_\phi^*(f(\xhat,\upsilon,w_m,t),t+1)$,
$\widehat\mu_M:=M^{-1}\sum_{m=1}^{M}Y_m$, and
$\mu:=\E_w[Y_m\mid\xhat,\upsilon]$.
By the assumed output bound, $Y_m\in(0,\Bmax)$ almost surely, and
$Y_1,\ldots,Y_M$ are conditionally i.i.d.\ given
$(\xhat,\upsilon,t)$.
Define the approximation event
$E_{\mathrm{Approx}}
:=\{q_3(\xhat,\upsilon,t)\le
\bar q_3(\xhat,\upsilon,t)\}$.
At the fixed tuple, $\overline E_M^*=\widehat\mu_M+\epsilon_M$
by~\eqref{eq:Ebar}, while the other terms in $q_3$ and
$\bar q_3$ coincide.
Therefore, $\neg E_{\mathrm{Approx}}$ is exactly the event
$\mu-\widehat\mu_M>\epsilon_M$.
Applying the one-sided Hoeffding inequality to the conditionally
i.i.d.\ variables $Y_m\in(0,\Bmax)$, and then using
~\ref{eq:Mbound}, gives
\begin{equation}
\begin{aligned}
\Prob_W\!\left(
\neg E_{\mathrm{Approx}}
\,\middle|\,
\xhat,\upsilon,t
\right)
&=
\Prob_W\!\left(
\mu-\widehat\mu_M>\epsilon_M
\,\middle|\,
\xhat,\upsilon,t
\right)\\
&\le
\exp\!\left(
-\frac{2M(\epsilon_M)^2}{(\Bmax)^2}
\right)
\le\beta_0 .
\end{aligned}
\label{eq:proof-finiteM-failure}
\end{equation}
Taking complements therefore yields
\begin{equation}
\Prob_W\!\left(
E_{\mathrm{Approx}}
\,\middle|\,
\xhat,\upsilon,t
\right)
=
\Prob_W\!\left(
q_3(\xhat,\upsilon,t)
\le
\bar q_3(\xhat,\upsilon,t)
\,\middle|\,
\xhat,\upsilon,t
\right)
\ge1-\beta_0 .
\label{eq:proof-finiteM-success}
\end{equation}
Hence the corrected empirical supermartingale residual upper-bounds the
exact supermartingale residual at the fixed state--attack--time tuple with
the confidence stated in Lemma~\ref{lem:finiteM}.
\end{proof}

\subsection{Proof of Lemma~\ref{lem:conformal}}
\label{app:proof-calibration}

\begin{proof}
Condition on the synthesis data and on all quantities fixed before
calibration. Under this conditioning, the score construction
in~\ref{eq:score} is fixed, and the calibration scores
$S_1,\ldots,S_N$ are i.i.d.\ with the same distribution as the score
$S(\bar\zeta)$ of a fresh independent augmented episode
$\bar\zeta$.
Let
\begin{equation}
F(s)
:=
\Prob_{\bar\zeta}\!\left(
S(\bar\zeta)\le s
\right),
\qquad
s\in\overline{\R}:=\R\cup\{+\infty\},
\label{eq:proof-conformal-cdf}
\end{equation}
denote the population distribution function of the score. The score
distribution may contain atoms, including an atom at $+\infty$
induced by hard unsafe-terminal violations; no continuity or
tie-breaking assumption is required.
For each rank $k\!\in\![N]$, define the conditional exceedance probability
associated with the $k$-th order statistic by
\begin{equation}
V_k(\mathcal D_{\mathrm{cal}})
:=
\Prob_{\bar\zeta}\!\left(
S(\bar\zeta)>S_{(k)}
\,\middle|\,
\mathcal D_{\mathrm{cal}}
\right)
=
1-F\!\left(S_{(k)}\right).
\label{eq:proof-conformal-Vk}
\end{equation}
We first control $V_k$ for a fixed rank $k$ and a fixed
$v\in[0,1)$. Define the population $(1-v)$-quantile
\begin{equation}
c_v
:=
\inf
\left\{
s\in\overline{\R}:
F(s)\ge1-v
\right\},
\label{eq:proof-conformal-quantile}
\end{equation}
and let
\begin{equation}
p_v^{-}
:=
\Prob_{\bar\zeta}\!\left(
S(\bar\zeta)<c_v
\right).
\label{eq:proof-conformal-pminus}
\end{equation}
By the definition of $c_v$, for every $s<c_v$,
\begin{equation}
F(s)<1-v.
\label{eq:proof-conformal-below-quantile}
\end{equation}
Therefore, by continuity of probability from below,
\begin{equation}
p_v^{-}
=
\Prob_{\bar\zeta}\!\left(
S(\bar\zeta)<c_v
\right)
=
\sup_{s<c_v}F(s)
\le
1-v.
\label{eq:proof-conformal-pminus-bound}
\end{equation}
Suppose now that $V_k(\mathcal D_{\mathrm{cal}})>v$. By
\ref{eq:proof-conformal-Vk}, this implies
\begin{equation}
F\!\left(S_{(k)}\right)
<
1-v.
\label{eq:proof-conformal-bad-rank}
\end{equation}
By the definition of $c_v$ in
\ref{eq:proof-conformal-quantile},
\ref{eq:proof-conformal-bad-rank} implies
\begin{equation}
S_{(k)}<c_v.
\label{eq:proof-conformal-order-below}
\end{equation}
Hence at least $k$ of the $N$ calibration scores are strictly smaller
than $c_v$. Defining
\begin{equation}
N_v
:=
\sum_{i=1}^{N}
\mathbf 1_{\{S_i<c_v\}},
\label{eq:proof-conformal-count}
\end{equation}
we therefore have the event inclusion
\begin{equation}
\left\{
V_k(\mathcal D_{\mathrm{cal}})>v
\right\}
\!\subseteq\!
\left\{
N_v\ge k
\right\}.
\label{eq:proof-conformal-event-inclusion}
\end{equation}
Since $c_v$ depends only on the fixed population score distribution,
the indicators in~\ref{eq:proof-conformal-count} are i.i.d.\
Bernoulli random variables with success probability $p_v^{-}$.
Therefore,
\begin{equation}
N_v
\sim
\operatorname{Binomial}
\left(
N,p_v^{-}
\right).
\label{eq:proof-conformal-binomial-law}
\end{equation}
Using~\ref{eq:proof-conformal-event-inclusion}, the binomial law in
\ref{eq:proof-conformal-binomial-law}, the bound
$p_v^{-}\le1-v$ from~\ref{eq:proof-conformal-pminus-bound}, and
monotonicity of the binomial upper tail in its success probability,
we obtain
\begin{equation}
\begin{aligned}
\Prob_{\mathcal D_{\mathrm{cal}}}\!\left(
V_k(\mathcal D_{\mathrm{cal}})>v
\right)
&\le
\Prob_{\mathcal D_{\mathrm{cal}}}\!\left(
N_v\ge k
\right)\\
&=
\sum_{j=k}^{N}
\binom Nj
(p_v^{-})^j
(1-p_v^{-})^{N-j}\\
&\le
\sum_{j=k}^{N}
\binom Nj
(1-v)^j
v^{N-j}.
\end{aligned}
\label{eq:proof-conformal-fixed-rank-tail}
\end{equation}
We now set $v=\bar v_k$, where $\bar v_k$ is defined by
\ref{eq:conformal-vk}. By construction,
\begin{equation}
\sum_{j=k}^{N}
\binom Nj
(1-\bar v_k)^j
\bar v_k^{\,N-j}
\le
\frac{\beta}{N}.
\label{eq:proof-conformal-vk-definition}
\end{equation}
Combining~\ref{eq:proof-conformal-fixed-rank-tail} and
\ref{eq:proof-conformal-vk-definition} gives, for every fixed
$k\!\in\![N]$,
\begin{equation}
\Prob_{\mathcal D_{\mathrm{cal}}}\!\left(
V_k(\mathcal D_{\mathrm{cal}})>\bar v_k
\right)
\le
\frac{\beta}{N}.
\label{eq:proof-conformal-fixed-rank}
\end{equation}
Define the rank-wise validity events
\begin{equation}
E_k
:=
\left\{
V_k(\mathcal D_{\mathrm{cal}})
\le
\bar v_k
\right\},
\qquad
k\!\in\![N].
\label{eq:proof-conformal-Ek}
\end{equation}
Equation~\ref{eq:proof-conformal-fixed-rank} implies
\begin{equation}
\Prob_{\mathcal D_{\mathrm{cal}}}\!\left(
\neg E_k
\right)
\le
\frac{\beta}{N},
\qquad
k\!\in\![N].
\label{eq:proof-conformal-Ek-failure}
\end{equation}
Applying a union bound over all $N$ ranks yields
\begin{equation}
\begin{aligned}
\Prob_{\mathcal D_{\mathrm{cal}}}\!\left(
\bigcap_{k=1}^{N}E_k
\right)
&=
1-
\Prob_{\mathcal D_{\mathrm{cal}}}\!\left(
\bigcup_{k=1}^{N}\neg E_k
\right)\\
&\ge
1-
\sum_{k=1}^{N}
\Prob_{\mathcal D_{\mathrm{cal}}}\!\left(
\neg E_k
\right)\\
&\ge
1-\beta .
\end{aligned}
\label{eq:proof-conformal-simultaneous}
\end{equation}
Using the definition of $V_k$ in
\ref{eq:proof-conformal-Vk}, the event in
\ref{eq:proof-conformal-simultaneous} is exactly
\begin{equation}
\left\{
\forall k\!\in\![N]:
\Prob_{\bar\zeta}\!\left(
S(\bar\zeta)>S_{(k)}
\,\middle|\,
\mathcal D_{\mathrm{cal}}
\right)
\le
\bar v_k
\right\}.
\label{eq:proof-conformal-final-event}
\end{equation}
Combining~\eqref{eq:proof-conformal-simultaneous} and
\eqref{eq:proof-conformal-final-event} yields
\begin{equation}
\Prob_{\mathcal D_{\mathrm{cal}}}\!\left[
\forall k\in[N]:\quad
\Prob_{\bar\zeta}\!\left(
S(\bar\zeta)>S_{(k)}\mid\mathcal D_{\mathrm{cal}}
\right)\le\bar v_k
\right]\ge1-\beta.
\label{eq:conformal}
\end{equation}
This proves Lemma~\ref{lem:conformal}.
In particular, the guarantee holds simultaneously for all ranks on
the same calibration-validity event. Therefore, a rank may
subsequently be selected as a function of the calibration data,
including through~\ref{eq:select}, without introducing any
additional confidence penalty.
\end{proof}


\subsection{Proof of Theorem~\ref{thm:conformal}}
\label{app:proof-foserl}

\begin{proof}
Condition on all pre-calibration choices, including the data used
for construction or selection. Under this conditioning, the candidate,
constants, operational law, and verification design are fixed, and
the calibration episodes are independent copies drawn as in
Definition~\ref{def:score}.
Define the simultaneous conformal-validity event
\begin{equation}
E_{\mathrm{Conf}}
:=
\left\{
\forall k\!\in\![N]:
\Prob_{\bar\zeta}\!\left(
S(\bar\zeta)>S_{(k)}
\,\middle|\,
\mathcal D_{\mathrm{cal}}
\right)
\le
\bar v_k
\right\}.
\label{eq:proof-theorem-conf-event}
\end{equation}
By Lemma~\ref{lem:conformal},
\begin{equation}
\Prob_{\mathcal D_{\mathrm{cal}}}\!\left(
E_{\mathrm{Conf}}
\right)
\ge
1-\beta .
\label{eq:proof-theorem-conf}
\end{equation}
Because the bounds in~\ref{eq:proof-theorem-conf-event} hold
simultaneously for every rank $k\!\in\![N]$, the data-dependent selection
of $k^*$ in~\ref{eq:select} does not incur an additional confidence
penalty.
Fix an arbitrary realization of $\mathcal D_{\mathrm{cal}}$ belonging
to $E_{\mathrm{Conf}}$ for which $\bar q_4\le0$.
By the empty-rank convention in Sec.~\ref{sec:verification},
this implies $\mathcal K_{\mathrm{fin}}\neq\varnothing$. Hence $k^*$, $r^*$,
$\widetilde b_1^*$, $\widetilde b_3^*$, and $\bar v_{k^*}$ are
well defined and fixed for the remainder of the argument. All
probabilities below involving a fresh episode are evaluated for this
fixed calibration realization; we suppress this conditioning for
notational simplicity.
For a fresh augmented episode
$\bar\zeta=(\upsilon,\omega,W)$, define the selected empirical barrier
event
\begin{equation}
E_{M,*}
:=
\left\{
\bar q_1(\xhat(0),\upsilon,0)\le r^*,\;
\bar q_2(\xhat(T),\upsilon,T)\le0,\;
\max_{t\in\T'}
\bar q_3(\xhat(t),\upsilon,t)\le r^*
\right\}.
\label{eq:proof-theorem-empirical-event}
\end{equation}
Since $k^*\in\mathcal K_{\mathrm{fin}}$,
\begin{equation}
S_{(k^*)}<+\infty,
\qquad
r^*
=
\max\left\{
0,S_{(k^*)}
\right\}.
\label{eq:proof-theorem-rstar}
\end{equation}
If $\neg E_{M,*}$ occurs, then either
$\bar q_2(\xhat(T),\upsilon,T)>0$, in which case
$S(\bar\zeta)=+\infty$, or one of the calibrated residuals exceeds
$r^*$. In the latter case,
\begin{equation}
S(\bar\zeta)
>
r^*
\ge
S_{(k^*)}.
\label{eq:proof-theorem-score-exceed}
\end{equation}
Therefore,
\begin{equation}
\neg E_{M,*}
\!\subseteq\!
\left\{
S(\bar\zeta)>S_{(k^*)}
\right\}.
\label{eq:proof-theorem-score-inclusion}
\end{equation}
Since the fixed calibration realization belongs to
$E_{\mathrm{Conf}}$, equations
\ref{eq:proof-theorem-conf-event} and
\ref{eq:proof-theorem-score-inclusion} imply
\begin{equation}
V_{M,*}
:=
\Prob_{\bar\zeta}\!\left(
\neg E_{M,*}
\right)
\le
\bar v_{k^*}.
\label{eq:proof-theorem-empirical-risk}
\end{equation}
We next transfer this empirical violation probability to the exact
state-dependent barrier conditions. For the corresponding episode
$\zeta=(\upsilon,\omega)\sim\QQ_{K,b}$, define the exact adjusted
barrier event
\begin{equation}
E_{\mathrm{Barrier},*}
:=
\left\{
q_1(\xhat(0),\upsilon,0)\le r^*,\;
q_2(\xhat(T),\upsilon,T)\le0,\;
q_3(\xhat(t),\upsilon,t)\le r^*
\ \forall t\in\T'
\right\}.
\label{eq:proof-theorem-barrier-event}
\end{equation}
Thus, on $E_{\mathrm{Barrier},*}$, the initial and supermartingale
conditions hold with the adjusted constants
\begin{equation}
\widetilde b_1^*
=
b_1^*+r^*,
\qquad
\widetilde b_3^*
=
b_3^*+r^*,
\label{eq:proof-theorem-adjusted-constants}
\end{equation}
while the unsafe-terminal condition retains the original threshold
$b_2^*$.
Let
\begin{equation}
V_*
:=
\Prob_{\QQ_{K,b}}\!\left(
\neg E_{\mathrm{Barrier},*}
\right).
\label{eq:proof-theorem-exact-risk-def}
\end{equation}
If $V_*=0$, the empirical-to-exact risk bound established below is
immediate. Assume therefore that $V_*>0$.
On an episode in $\neg E_{\mathrm{Barrier},*}$, let
$E_{\mathrm{Det},*}$ denote the event that the finite-$M$ test detects
at least one exact adjusted barrier violation. If
$q_1(\xhat(0),\upsilon,0)>r^*$ or
$q_2(\xhat(T),\upsilon,T)>0$, the violation is detected
deterministically because
\begin{equation}
\bar q_1=q_1,
\qquad
\bar q_2=q_2.
\label{eq:proof-theorem-deterministic-residuals}
\end{equation}
Otherwise, an exact supermartingale violation must occur. On this
remaining event, define the first such time by
\begin{equation}
\tau
:=
\min
\left\{
t\in\T':
q_3(\xhat(t),\upsilon,t)>r^*
\right\}.
\label{eq:proof-theorem-first-violation}
\end{equation}
Conditional on the complete episode $\zeta$, the index $\tau$ and the
tuple $(\xhat(\tau),\upsilon,\tau)$ are fixed, whereas the successor
batch associated with that tuple is fresh and conditionally i.i.d.
By Lemma~\ref{lem:finiteM},
\begin{equation}
\Prob_W\!\left(
q_3(\xhat(\tau),\upsilon,\tau)
\le
\bar q_3(\xhat(\tau),\upsilon,\tau)
\,\middle|\,
\zeta
\right)
\ge
1-\beta_0 .
\label{eq:proof-theorem-first-violation-detection}
\end{equation}
Since the exact residual at time $\tau$ already satisfies
$q_3(\xhat(\tau),\upsilon,\tau)>r^*$, the event in
\ref{eq:proof-theorem-first-violation-detection} implies
\begin{equation}
\bar q_3(\xhat(\tau),\upsilon,\tau)
>
r^*,
\label{eq:proof-theorem-empirical-detection}
\end{equation}
and therefore produces an empirical violation.
Consequently, whether the exact violation occurs in $q_1$, $q_2$, or
$q_3$, its conditional probability of being detected is at least
$1-\beta_0$. Hence
\begin{equation}
\Prob_{\bar\zeta}\!\left(
E_{\mathrm{Det},*}
\,\middle|\,
\neg E_{\mathrm{Barrier},*}
\right)
\ge
1-\beta_0 .
\label{eq:proof-theorem-detection}
\end{equation}
Moreover, every detected exact violation is also an empirical
violation, so
\begin{equation}
E_{\mathrm{Det},*}
\cap
\neg E_{\mathrm{Barrier},*}
\!\subseteq\!
\neg E_{M,*}.
\label{eq:proof-theorem-detection-inclusion}
\end{equation}
Using~\ref{eq:proof-theorem-detection} and
\ref{eq:proof-theorem-detection-inclusion},
\begin{align}
V_{M,*}
&\ge
\Prob_{\bar\zeta}\!\left(
E_{\mathrm{Det},*}
\cap
\neg E_{\mathrm{Barrier},*}
\right),
\label{eq:proof-theorem-transfer-a}\\
&=
\Prob_{\bar\zeta}\!\left(
E_{\mathrm{Det},*}
\,\middle|\,
\neg E_{\mathrm{Barrier},*}
\right)
V_*,
\label{eq:proof-theorem-transfer-b}\\
&\ge
(1-\beta_0)V_* .
\label{eq:proof-theorem-transfer-c}
\end{align}
Combining~\ref{eq:proof-theorem-empirical-risk} and
\ref{eq:proof-theorem-transfer-c} yields
\begin{equation}
V_*
\le
\frac{V_{M,*}}{1-\beta_0}
\le
\frac{\bar v_{k^*}}{1-\beta_0}.
\label{eq:proof-theorem-exact-risk}
\end{equation}
The argument uses only the first exact supermartingale violation after
conditioning on the complete episode. Lemma~\ref{lem:finiteM} is
therefore invoked once per exact violating episode, and no union bound
over the horizon is required.
It remains to relate satisfaction of the exact adjusted barrier
conditions to the trajectory-level degradation event. Define the
trajectory-safety event
\begin{equation}
E_{\mathrm{Safe}}
:=
\left\{
\Dtraj(\upsilon,\omega)\le\alpha
\right\},
\label{eq:proof-theorem-safe-event}
\end{equation}
and let
\begin{equation}
B_t
:=
B_\phi^*(\xhat(t),t).
\label{eq:proof-theorem-Bt}
\end{equation}
With
\begin{equation}
\mathcal F_t
:=
\sigma
\big(
\upsilon,x(0),w(0),\ldots,w(t-1)
\big),
\qquad
t\in\T,
\label{eq:proof-theorem-filtration}
\end{equation}
define the good-prefix events
\begin{equation}
G_0
:=
\left\{
q_1(\xhat(0),\upsilon,0)\le r^*
\right\},
\qquad
G_{t+1}
:=
G_t
\cap
\left\{
q_3(\xhat(t),\upsilon,t)\le r^*
\right\},
\quad
t\in\T'.
\label{eq:proof-theorem-prefix}
\end{equation}
By Assumption~\ref{ass:noise} and the product law in
Definition~\ref{def:jointlaws}, $w(t)$ is independent of
$\mathcal F_t$ with law $P_w$. Thus the expectation in
\eqref{eq:q3} satisfies, along the trajectory,
\begin{equation}
\E\!\left[B_{t+1}\mid\mathcal F_t\right]
=
\E_{w\sim P_w}\!\left[
B_\phi^*\!\left(f(\xhat(t),\upsilon,w,t),t+1\right)
\right].
\label{eq:proof-theorem-conditional-third}
\end{equation}
On $G_{t+1}$, the exact adjusted supermartingale condition therefore
gives
\begin{equation}
\E\!\left[
B_{t+1}
\,\middle|\,
\mathcal F_t
\right]
\le
\kappa^*B_t+\widetilde b_3^*.
\label{eq:proof-theorem-adjusted-third}
\end{equation}
Moreover, $G_{t+1}$ is $\mathcal F_t$-measurable,
$G_{t+1}\!\subseteq\! G_t$, and $B_t\ge0$. Thus, by the tower property,
\begin{equation}
\begin{aligned}
\E\!\left[
B_{t+1}\mathbf 1_{G_{t+1}}
\right]
&=
\E\!\left[
\mathbf 1_{G_{t+1}}
\E\!\left[
B_{t+1}
\,\middle|\,
\mathcal F_t
\right]
\right]\\
&\le
\kappa^*
\E\!\left[
B_t\mathbf 1_{G_{t+1}}
\right]
+
\widetilde b_3^*
\Prob(G_{t+1})\\
&\le
\kappa^*
\E\!\left[
B_t\mathbf 1_{G_t}
\right]
+
\widetilde b_3^* .
\end{aligned}
\label{eq:proof-theorem-prefix-recursion}
\end{equation}
On $G_0$, the adjusted initial condition gives
\begin{equation}
B_0
\le
b_1^*+r^*
=
\widetilde b_1^*.
\label{eq:proof-theorem-adjusted-initial}
\end{equation}
Iterating~\ref{eq:proof-theorem-prefix-recursion} from
\ref{eq:proof-theorem-adjusted-initial} yields
\begin{equation}
\E\!\left[
B_T\mathbf 1_{G_T}
\right]
\le
\widetilde b_1^*(\kappa^*)^T
+
\widetilde b_3^*\cG_T(\kappa^*).
\label{eq:proof-theorem-terminal-bound}
\end{equation}
On
$\neg E_{\mathrm{Safe}}\cap E_{\mathrm{Barrier},*}$,
all good-prefix conditions hold. In addition,
$\neg E_{\mathrm{Safe}}$ implies
$\xhat(T)\in\cX_u^\alpha$. Since
$q_2(\xhat(T),\upsilon,T)\le0$ on
$E_{\mathrm{Barrier},*}$, the unsafe-terminal residual
\ref{eq:q2} gives
\begin{equation}
B_T
\ge
b_2^*.
\label{eq:proof-theorem-terminal-separation}
\end{equation}
Consequently,
\begin{equation}
\begin{aligned}
b_2^*
\Prob_{\QQ_{K,b}}\!\left(
\neg E_{\mathrm{Safe}}
\cap
E_{\mathrm{Barrier},*}
\right)
&\le
\E\!\left[
B_T
\mathbf 1_{\{
\neg E_{\mathrm{Safe}}
\cap
E_{\mathrm{Barrier},*}
\}}
\right]\\
&\le
\E\!\left[
B_T\mathbf 1_{G_T}
\right]\\
&\le
\widetilde b_1^*(\kappa^*)^T
+
\widetilde b_3^*\cG_T(\kappa^*) .
\end{aligned}
\label{eq:proof-theorem-barrier-failure}
\end{equation}
Define
\begin{equation}
\delta_B^*
:=
\frac{
\widetilde b_1^*(\kappa^*)^T+\widetilde b_3^*\cG_T(\kappa^*)
}{
b_2^*
}.
\label{eq:proof-theorem-deltaB}
\end{equation}
Then~\ref{eq:proof-theorem-barrier-failure} implies
\begin{equation}
\Prob_{\QQ_{K,b}}\!\left(
\neg E_{\mathrm{Safe}}
\cap
E_{\mathrm{Barrier},*}
\right)
\le
\delta_B^*.
\label{eq:proof-theorem-barrier-bound}
\end{equation}

We now decompose trajectory-level failure according to whether the
exact adjusted barrier conditions hold. Using
\ref{eq:proof-theorem-exact-risk} and
\ref{eq:proof-theorem-barrier-bound},
\begin{equation}
\begin{aligned}
\Prob_{\QQ_{K,b}}\!\left(
\neg E_{\mathrm{Safe}}
\right)
&=
\Prob_{\QQ_{K,b}}\!\left(
\neg E_{\mathrm{Safe}}
\cap
E_{\mathrm{Barrier},*}
\right)\\
&\quad+
\Prob_{\QQ_{K,b}}\!\left(
\neg E_{\mathrm{Safe}}
\cap
\neg E_{\mathrm{Barrier},*}
\right)\\
&\le
\delta_B^*+V_*\\
&\le
\frac{
\widetilde b_1^*(\kappa^*)^T+\widetilde b_3^*\cG_T(\kappa^*)
}{
b_2^*
}
+
\frac{\bar v_{k^*}}{1-\beta_0}.
\end{aligned}
\label{eq:proof-theorem-failure-bound}
\end{equation}
By the definition of $\bar q_4$ in~\ref{eq:q4bar} and the fact that
$b_2^*>0$, the condition $\bar q_4\le0$ is equivalent to
\begin{equation}
\frac{
\widetilde b_1^*(\kappa^*)^T+\widetilde b_3^*\cG_T(\kappa^*)
}{
b_2^*
}
+
\frac{\bar v_{k^*}}{1-\beta_0}
\le
1-\rho .
\label{eq:proof-theorem-q4-equivalent}
\end{equation}
Combining~\ref{eq:proof-theorem-failure-bound} and
\ref{eq:proof-theorem-q4-equivalent} therefore gives
\begin{equation}
\bar q_4\le0
\quad\Longrightarrow\quad
\Prob_{(\upsilon,\omega)\sim\QQ_{K,b}}\!\left(
\neg E_{\mathrm{Safe}}
\right)
\le
1-\rho .
\label{eq:proof-theorem-failure-rho}
\end{equation}
Taking complements,
\begin{equation}
\bar q_4\le0
\quad\Longrightarrow\quad
\Prob_{(\upsilon,\omega)\sim\QQ_{K,b}}\!\left(
E_{\mathrm{Safe}}
\right)
\ge
\rho .
\label{eq:proof-theorem-safe-rho}
\end{equation}
Since
$E_{\mathrm{Safe}}
=
\{\Dtraj(\upsilon,\omega)\le\alpha\}$, this is equivalent to
\begin{equation}
\bar q_4\le0
\quad\Longrightarrow\quad
\Prob_{(\upsilon,\omega)\sim\QQ_{K,b}}\!\left(
\Dtraj(\upsilon,\omega)\le\alpha
\right)
\ge
\rho .
\label{eq:proof-theorem-rho}
\end{equation}
We have established~\ref{eq:proof-theorem-rho} for every calibration
realization belonging to $E_{\mathrm{Conf}}$. Since
\ref{eq:proof-theorem-conf} gives
$\Prob_{\mathcal D_{\mathrm{cal}}}(E_{\mathrm{Conf}})\ge1-\beta$,
it follows that
\begin{equation}
\Prob_{\mathcal D_{\mathrm{cal}}}\!\left[
\bar q_4\le0
\ \Longrightarrow\
\Prob_{(\upsilon,\omega)\sim\QQ_{K,b}}\!\left(
\Dtraj(\upsilon,\omega)\le\alpha
\right)
\ge\rho
\right]
\ge
1-\beta .
\label{eq:proof-theorem-final}
\end{equation}
Equation~\ref{eq:proof-theorem-final} is exactly
\ref{eq:final}, completing the proof.
\end{proof}


\section{Additional Experimental Details and Results}
\label{app:expdetails}

This appendix gives the complete experimental specification underlying
Sec.~\ref{sec:experiments}.  We first state the benchmark instantiation used in
our experiments, then summarize the policy families and the six case studies,
describe the certificate parameterization, and finally collect all fixed
experimental settings in a single table.  Numerical robustness results are
reported separately so that benchmark definitions are not mixed with the
outcomes of a particular certification run.
The complete codebase, experiment configurations, and per-specification
results, including the learned certificate parameters, are available at
\href{https://figshare.com/s/9bb42c4050f749744af8}{this link}.
\subsection{Benchmark formulation}
\label{app:benchmark-formulation}

Each case study instantiates the finite-horizon stochastic closed-loop process
of Sec.~\ref{sec:setup}. For every tested environment--policy pair, nominal
and attacked trajectories start from the same initial state and use the shared
episode-randomness semantics of Definition~\ref{def:coupling}. The operational
attack law is instantiated separately for each benchmark using the attack
method, attacked action channel, and norm budget reported in
Table~\ref{tab:fixed-case-study-settings}. Each realized attack is fixed before
the episode, perturbs exactly $K$ decision times, and satisfies the corresponding
per-step norm bound. Thus, $K$ controls temporal sparsity, whereas $b$ controls
the perturbation magnitude.

Because the horizons differ across tasks, we report temporal attack density as
$K/T$ in addition to the absolute cardinality $K$.  The final grids are chosen
so that comparable attack densities are represented across environments.  The
formal event always uses the raw paired-degradation threshold
$\Dtraj\le\alpha$.  For cross-task presentation we additionally report
\begin{equation}
    \alpha_{\%}:=100\,\frac{\alpha}{S_{\mathrm{env}}},
    \qquad
    \alpha=\frac{\alpha_{\%}}{100}S_{\mathrm{env}},
    \label{eq:app-alpha-normalization}
\end{equation}
where $S_{\mathrm{env}}>0$ is a fixed environment-level return scale.  The
same $S_{\mathrm{env}}$ is used for every policy compared within one
environment.  Hence $\alpha_{\%}$ is neither a success probability nor a
policy-specific percentage drop from that policy's nominal return; it is only a
common reporting scale for the raw terminal-degradation threshold.  The
values used in the current experiments are reported in
Table~\ref{tab:fixed-case-study-settings}.

The operating law also fixes the source of episode randomness.  In the custom
control tasks, the random input enters the declared stochastic state update.
For the current Hopper, HalfCheetah, and Ant campaigns, the environmental
randomness follows the corresponding native reset semantics, with no artificial
per-step actuation or process noise. Formal verification approximates the
conditional successor expectations using fresh Monte Carlo successor batches,
with the corresponding sample counts reported in
Table~\ref{tab:fixed-case-study-settings}.

\subsection{Policy families}
\label{app:policy-families}

The formulation of Sec.~\ref{sec:setup} is agnostic to how the fixed
deployed policy is obtained: $\pi$ may be a learned policy or an
analytically designed decision rule, provided its deployment-time
execution satisfies Assumption~\ref{ass:noise}. The main empirical
study focuses on learned RL policies.

Specifically, we evaluate frozen policies trained with TD3
\citep{fujimoto2018addressing}, SAC \citep{haarnoja2018soft}, and PPO
\citep{schulman2017proximal}. The exact policy coverage, actor architectures, and deployment
conventions are collected in Table~\ref{tab:fixed-case-study-settings}, while
the resulting complete robustness profiles are shown in
Fig.~\ref{fig:all-cert-results}. Each checkpoint is trained before
certification and then held fixed throughout attack generation, paired rollout
generation, barrier synthesis, calibration, and formal verification. Thus the training algorithm does not alter the
certification problem: the certified object is the deployed policy
itself, including its chosen deployment-time randomization mode.

For the three custom tasks, we additionally retain analytic reference
policies: swing-up with local LQR stabilization for the inverted
pendulum, headway regulation for ACC, and waypoint/pure-pursuit
feedback for the unicycle. These are supplementary instances of the
same fixed-policy formulation, not competing certification methods.
All main-text empirical results report RL policies only.

\subsection{Case studies}
\label{app:case-studies}

\subsubsection{Inverted pendulum}
The inverted-pendulum task captures nonlinear swing-up and stabilization under
bounded torque and stochastic angular-velocity disturbances.  With physical
state $x=(\theta,\omega)$, the implementation uses the benchmark
semi-implicit discretization with $\Delta t=0.05$ and an additive Gaussian
term on the angular-velocity update.  The RL policy observes
$(\cos\theta,\sin\theta,\omega)$ and acts through a scalar torque in
$[-2,2]$.  The initial state is sampled near the downward configuration,
$\theta_0\in\pi\pm0.1$ and $\omega_0\in[-0.1,0.1]$.  The per-step reward
penalizes wrapped-angle error, angular velocity, and control effort,
\begin{align}
 r_t=-\bigl(\bar\theta_t^2+0.1\omega_t^2+10^{-3}u_t^2\bigr),
 \qquad \bar\theta:=\operatorname{wrap}(\theta).
\end{align}
This task is deliberately nonlinear over the full angular domain: a bounded
action perturbation applied during swing-up can affect the subsequent
closed-loop transient long after the attacked time step.

\subsubsection{Adaptive cruise control}
ACC models longitudinal car following with state
$x=(d,v_E,v_L)$, consisting of inter-vehicle distance, ego speed, and lead
speed.  With $\Delta t=0.1$,
\begin{align}
 v^L_{t+1}&=v^L_t+\Delta t\,a^L_t+\sigma\xi_t,\\
 v^E_{t+1}&=v^E_t+\Delta t\,u_t,\\
 d_{t+1}&=d_t+\Delta t\,(v^L_t-v^E_t),
 \label{eq:app-acc-dynamics}
\end{align}
where the lead-vehicle stochastic input is shared between the nominal and
attacked branches.  The ego action is bounded in $[-3,2]$ and the state is
clipped to the declared speed/headway box.  Initial speeds satisfy
$v_E(0)=v_L(0)\sim\mathcal U[20,25]$ with the initial headway placed at the
desired constant-time-headway reference.  Writing
$e=d-(5+1.4v_E)$ and $\dot e=v_L-v_E$, the reward is
\begin{align}
 r_t=-(e/10)^2-0.1(\dot e/5)^2-0.02(u/3)^2.
\end{align}
The case study combines a physically interpretable stochastic disturbance with
a direct action-space attack on the throttle/brake command.

\subsubsection{Unicycle waypoint tracking}
The unicycle task represents planar goal-directed motion with physical state
$x=(p_x,p_y,\theta)$ and two control channels, forward speed and turn rate.
The discrete-time model is
\begin{align}
 p_{x,t+1}&=p_{x,t}+\Delta t\,v_t\cos\theta_t,\\
 p_{y,t+1}&=p_{y,t}+\Delta t\,v_t\sin\theta_t,\\
 \theta_{t+1}&=\operatorname{wrap}(\theta_t+\Delta t\,\omega_t+\sigma\xi_t),
 \label{eq:app-unicycle-dynamics}
\end{align}
with $\Delta t=0.1$ and stochastic heading evolution.  The action satisfies
$v\in[0,1]$ and $\omega\in[-1.5,1.5]$.  Initial positions are sampled from
an annulus of radius $[2,4]$ around the goal with a random initial heading.
The reward combines squared distance-to-goal with normalized forward-speed and
turn-rate penalties.  Because the attacked channels have heterogeneous
physical units, the fixed channel scaling used to map the scalar budget $b$ to
$(v,\omega)$ perturbations is part of the experimental specification.

\subsubsection{Hopper-v5}
Hopper-v5 is a MuJoCo planar locomotion task with intermittent ground contact
\citep{towers2024gymnasium,todorov2012mujoco}.  Its generalized dynamics can
be written schematically as
\begin{align}
 M(q)\ddot q+C(q,\dot q)\dot q+g(q)=Bu+J(q)^\top\lambda,
\end{align}
where the simulator resolves the contact forces $\lambda$.  The deployed
policy observes the standard $11$-dimensional observation and controls three
bounded joint-torque channels.  The retained runs use the native Hopper-v5
reset randomness with reset-noise scale $0.005$, and no artificial per-step
transition or actuation noise is added.  The default unhealthy-state
termination is retained.  This task tests certification under hybrid contact
transitions, where a sparse action perturbation can change both the accumulated
return and the subsequent contact sequence.

\subsubsection{HalfCheetah-v5}
HalfCheetah-v5 is a higher-dimensional planar MuJoCo locomotion task with six
coupled torque inputs \citep{towers2024gymnasium,todorov2012mujoco}.  As in
Hopper, its multibody dynamics include state-dependent inertial coupling and
ground-contact forces.  The policy uses the standard $17$-dimensional
observation and actions are clipped to $[-1,1]^6$.  The retained campaign uses
the native HalfCheetah-v5 reset randomness (reset-noise scale $0.1$) and no
artificial per-step actuation or process noise.  Small action perturbations can
shift the phase and coordination of the learned gait over the finite horizon,
which makes the paired return degradation sensitive to both attack magnitude
and attack timing.

\subsubsection{Ant-v5}
Ant-v5 is a three-dimensional quadrupedal MuJoCo locomotion task with eight
actuated joints \citep{towers2024gymnasium,todorov2012mujoco}.  The current
policy interface uses the standard $105$-dimensional Ant-v5 observation and an
$8$-dimensional action in $[-1,1]^8$.  The default simulator control interval
is $0.05$~s, the retained horizon is $T=1000$, and the environment uses the
native Ant reset/health semantics of the final run.  The reward combines
forward motion and healthy-state reward with control and contact penalties.
The task provides the highest-dimensional locomotion setting in the current
suite and therefore probes the scalability of the learned certificate beyond
the planar MuJoCo systems.

\subsection{Certificate parameterization}
\label{app:certificate-parameterization}

All experiments use the bounded neural barrier parameterization
introduced in Sec.~\ref{sec:synthesis}. The barrier is a scalar
structured MLP on features of the augmented nominal--attacked state,
degradation coordinate, and time, with a sigmoid output bounded by
$\Bmax$. Its learned constants
$b_1^*,b_2^*,b_3^*$ and $\kappa^*$ retain the meanings in
Definition~\ref{def:barrier}: they summarize the initial-set bound,
terminal separation, stochastic one-step drift, and multiplicative
propagation factor. The accumulated finite-horizon contribution uses
$\cG_T(\kappa)\!=\!\sum_{j=0}^{T-1}\kappa^j$ as in
Definition~\ref{def:barrier}, rather than replacing the geometric sum
by the looser factor $T$.

The hidden layers use SiLU activations in all retained certificates.
This is an optimization choice, not a theoretical restriction.
During synthesis, the conditional drift term is estimated using empirical
averages over $M_s$ fresh successor samples. Formal verification likewise
uses empirical averages over $M$ fresh successor samples together with the
finite-sample correction of Lemma~\ref{lem:finiteM}. The values of
$N_s$, $M_s$, $N$, and $M$ used in each case study are collected in
Table~\ref{tab:fixed-case-study-settings}.

For the synthesis sets in~\eqref{eq:setSU}, if
$U\!=\!\varnothing$, its loss contribution is defined as zero.
This provides no evidence of robustness, since the resulting candidate
remains subject to independent verification.
The safe pool $S$ is retained for supplementary analysis and does not
enter the loss in~\eqref{eq:loss}.
If auxiliary unsafe examples are used during synthesis, their full
augmented trajectories are retained so that intermediate
state--attack--time tuples may also contribute to the sampled drift
condition; such data affect candidate construction only and do not
themselves establish a formal guarantee.

\subsection{Fixed experimental settings}
\label{app:fixed-experimental-settings}
Table~\ref{tab:fixed-case-study-settings} provides the complete experimental
specification for the six retained case studies, including environment and
policy settings, attack methods and norm budgets, attack and
degradation-tolerance grids, normalization constants, certificate
parameterization, and synthesis/verification sample budgets. Unless stated
otherwise, all formal FoSeRL experiments use target reliability
$\rho=0.90$ and confidence parameters $\beta=\beta_0=0.001$.

\begin{table*}[p]
\centering
\caption{
Complete experimental specification for the six FoSeRL case studies.
For each environment, we report the environment and policy configuration,
attack specification, degradation-tolerance grid, certificate architecture,
and synthesis/verification sample budgets.
$N_s$ and $M_s$ denote the synthesis episodes and successor samples,
while $N$ and $M$ denote the corresponding formal-verification sample counts.
All formal runs target $\rho=0.90$ with $\beta=\beta_0=0.001$.
}
\label{tab:fixed-case-study-settings}

\setlength{\tabcolsep}{3pt}
\renewcommand{\arraystretch}{1.85}

\newcommand{\lighttabrule}{%
    \arrayrulecolor{black!20}%
    \specialrule{0.4pt}{0pt}{0pt}%
    \arrayrulecolor{black}%
}

\resizebox{\textwidth}{!}{%
\begin{tabular}{
    >{\columncolor{FoSeRLTableBlue}\hspace{2mm}}c<{\hspace{2mm}}
    l
    !{\vrule width 0.4pt}
    c
    !{\vrule width 0.4pt}
    c
    !{\vrule width 0.4pt}
    c
    !{\vrule width 0.4pt}
    c
    !{\vrule width 0.4pt}
    c
    !{\vrule width 0.4pt}
    c !{\vrule width 0.1pt}
}
\toprule
&
\textbf{Parameter}
&
\textbf{Inv. Pend.}
&
\textbf{ACC}
&
\textbf{Unicycle}
&
\textbf{Hopper}
&
\textbf{HalfCheetah}
&
\textbf{Ant}
\\

\lighttabrule


& Model
& Custom
& Custom
& Custom
& Gym v5
& Gym v5
& Gym v5
\\

& $n_{\mathrm{obs}}$
& 3 & 3 & 3 & 11 & 17 & 105
\\

& $n_x$
& 2 & 3 & 3 & 12 & 18 & 29
\\

& $n_a$
& 1 & 1 & 2 & 3 & 6 & 8
\\

& $\Delta t$
& 0.05 & 0.10 & 0.10 & 0.008 & 0.05 & 0.05
\\

& $T$
& 200 & 300 & 500 & 1000 & 500 & 1000
\\

& $\gamma$
& 0.99 & 0.99 & 0.99 & 0.99 & 0.99 & 0.99
\\

& Task const.
& \makecell{$\theta$ wrapped, $|\omega|\le 8$\\$g=10,\ m=\ell=1$}
& \makecell{$d\in[0,60]$, $v_E,v_L\in[0,40]$\\$d_0=5,\ h=1.4$}
& \makecell{$(p_x,p_y)\in[-5,5]^2$\\goal $(0,0)$}
& Native v5
& Native v5
& Native v5
\\

& Randomness
& $\omega$ noise
& Lead-process noise
& Heading noise
& Reset
& Reset
& Reset
\\

& Noise scale
& $\sigma=0.01$
& $\sigma=0.30$
& $\sigma=0.05$
& $0.005$
& $0.1$
& $0.1$
\\

& Initial law
& \makecell{$\theta_0\in\pi\pm0.1$\\$\omega_0\in[-0.1,0.1]$}
& $v_E(0)=v_L(0)\sim\mathcal U[20,25]$
& \makecell{annulus $[2,4]$\\random $\theta_0$}
& Native
& Native
& Native
\\

\multirow{-12}{*}{\rotatebox[origin=c]{90}{\textbf{Environment}}}
& $\mathcal A$
& $[-2,2]$
& $[-3,2]$
& \makecell{$v\in[0,1]$\\$\omega\in[-1.5,1.5]$}
& $[-1,1]^3$
& $[-1,1]^6$
& $[-1,1]^8$
\\

\lighttabrule


& Policy
& TD3, SAC, PPO
& TD3, SAC, PPO
& TD3, SAC, PPO
& TD3, SAC, PPO
& TD3, SAC, PPO
& TD3, SAC, PPO
\\

& Analytic ref.
& Swing-up + LQR
& Headway LQR
& Pure pursuit
& --
& --
& --
\\

& Actor
& $3\!\to\!64\!\to\!64\!\to\!1$
& $3\!\to\!64\!\to\!64\!\to\!1$
& $3\!\to\!64\!\to\!64\!\to\!2$
& $11\!\to\!64\!\to\!64\!\to\!3$
& $17\!\to\!64\!\to\!64\!\to\!6$
& $105\!\to\!256\!\to\!256\!\to\!8$
\\

& Activation
& ReLU & ReLU & ReLU & ReLU & ReLU & ReLU
\\

\multirow{-5}{*}{\rotatebox[origin=c]{90}{\textbf{Policy}}}
& Output
& Bounded
& Bounded
& Bounded
& Bounded
& tanh / bounded
& tanh / bounded
\\

\lighttabrule


& Channel
& Torque
& Ego accel.
& $(v,\omega)$
& $3$ actuators
& $6$ actuators
& $8$ actuators
\\

& Attack method
& PGD / Uniform
& PGD
& PGD
& LAS / Uniform
& PGD / Uniform
& DE / Uniform
\\

& Norm budget
& $\ell_\infty$
& $\ell_2$
& $\ell_\infty$
& $\ell_2$
& $\ell_\infty$
& $\ell_\infty$
\\

& $\alpha_{\%}$
& $\{0.5,1,2.5,5\}$
& $\{0.5,1,2.5,5\}$
& $\{0.5,1,2.5,5\}$
& $\{0.5,1,2.5,5\}$
& $\{5,10,15,20\}$
& $\{5,10,15,20\}$
\\

\multirow{-7}{*}{\rotatebox[origin=c]{90}{\textbf{Threat Model}}}
& $S_{\mathrm{env}}$
& 453.340
& 6.229
& 22.921
& 174.525
& 682.400
& 336.554
\\

\lighttabrule


& Certificate
& MLP
& MLP
& MLP
& MLP
& MLP
& MLP
\\

& Input dim.
& 12 & 18 & 21 & 39 & 57 & 92
\\

& Widths
& $48,48$
& $48,48$
& $48,48$
& $64,64$
& $96,96,64$
& $128,128,96$
\\

& Output dim.
& 1 & 1 & 1 & 1 & 1 & 1
\\

& Activation
& SiLU & SiLU & SiLU & SiLU & SiLU & SiLU
\\

\multirow{-6}{*}{\rotatebox[origin=c]{90}{\textbf{Synthesis}}}
& $(N_s,M_s)$
& $(1024,8)$
& $(1024,8)$
& $(1024,8)$
& $(1000,10)$
& $(1000,10)$
& $(2000,40)$
\\

\lighttabrule


& $\rho$
& 0.90 & 0.90 & 0.90 & 0.90 & 0.90 & 0.90
\\

& $\beta$
& 0.001 & 0.001 & 0.001 & 0.001 & 0.001 & 0.001
\\

& $\beta_0$
& 0.001
& 0.001
& 0.001
& 0.001
& 0.001
& 0.001
\\

& $N$
& 2000
& $1395$--$2000$
& $1566$--$2000$
& 2000
& 2000
& 3000
\\

\multirow{-5}{*}{\rotatebox[origin=c]{90}{\textbf{Verification}}}
& $M$
& 49,152
& 49,152
& 49,152
& 30,000
& 50,000
& 20,000
\\

\bottomrule
\end{tabular}%
}
\end{table*}

\paragraph{Computational setup.}
All FoSeRL experiments reported in this paper were run on a workstation
with an Intel Xeon w5-3433 CPU, 128~GB DDR5-4400 RAM, and an NVIDIA
RTX~4500 Ada Generation GPU with 24~GB GDDR6 ECC memory. Reported runtimes are wall-clock certification
times and exclude policy training.

\subsection{Certificate construction and reporting workflow}
\label{app:certificate-workflow}

This subsection summarizes how each tested attack specification is converted
into a reported FoSeRL certificate. The formal definitions and guarantees are
given in Sec.~\ref{sec:verification}; here we describe the experimental
workflow used to produce the robustness profiles in
Fig.~\ref{fig:all-cert-results}.

\paragraph{From a tested specification to a certified point.}
Fix an environment--policy pair and a specification $(K,b,\alpha,\rho)$.
Candidate construction first learns the bounded time-dependent barrier
$B_{\phi^*}$ and the associated constants
$b_1^*,b_2^*,b_3^*$ and $\kappa^*$ using the synthesis procedure of
Sec.~\ref{sec:synthesis}. The candidate and all verification-design choices
are then frozen. Independent calibration episodes are drawn under the
operational law $\QQ_{K,b}$, and each episode is assigned the score
$S_i$ of Definition~\ref{def:score}. For every finite conformal rank $k$,
the score determines the residual allowance
$r_k=\max\{0,S_{(k)}\}$ and the corresponding out-of-sample violation bound
$\bar v_k$ from~\eqref{eq:conformal-vk}. The selected rank $k^*$ is the one
maximizing the resulting end-to-end reliability lower bound, as in
\eqref{eq:select}.

\paragraph{Reliability decomposition.}
For reporting and diagnosis, we separate the reliability bound into three
terms:
\begin{equation}
\begin{aligned}
&\mathrm{Base}
:=
\frac{
b_1^*(\kappa^*)^T
+
b_3^*\,\cG_T(\kappa^*)
}{
b_2^*
},\\
&\mathrm{Margin}_k
:=
r_k
\frac{
(\kappa^*)^T
+
\cG_T(\kappa^*)
}{
b_2^*
},\\
&\mathrm{Bridge}_k
:=
\frac{\bar v_k}{1-\beta_0}.
\end{aligned}
\label{eq:app-cert-components}
\end{equation}
The corresponding certified reliability at rank $k$ is
\begin{equation}
\rho_{\mathrm{cert}}(k)
=
1-
\Big(
\mathrm{Base}
+
\mathrm{Margin}_k
+
\mathrm{Bridge}_k
\Big).
\label{eq:app-rho-cert}
\end{equation}
Hence $k^*$ equivalently maximizes
$\rho_{\mathrm{cert}}(k)$ over the finite ranks, and the tested magnitude
$b$ is accepted exactly when
$\rho_{\mathrm{cert}}(k^*)\ge\rho$, which is equivalent to the acceptance
condition $\bar q_4\le0$ in~\eqref{eq:q4bar}.
Base is the contribution of the frozen barrier itself, Margin is the
propagated calibration slack, and Bridge accounts for violations of the
exact adjusted barrier conditions that remain after finite-sample
calibration. These are components of a certified lower bound, not empirical
terminal-failure frequencies.

\paragraph{Finite-successor correction.}
When the conditional expectation in the drift condition is estimated from
$M$ fresh successors, the correction $\epsilon_M$ is chosen to satisfy
Lemma~\ref{lem:finiteM}; the tight Hoeffding choice at fixed $M$ is
$\epsilon_M=\Bmax\sqrt{\log(1/\beta_0)/(2M)}$.
This correction enters the empirical drift residual and therefore affects the
calibration scores and the selected Margin. It is not subtracted again from
$\rho_{\mathrm{cert}}$.

\paragraph{From certified points to robustness profiles.}
For each fixed $(K,\alpha)$, attack magnitudes are explored over the prescribed
experimental grid. The quantity $b_{\mathrm{emp}}$ is used only as a
pre-verification design reference for candidate construction and search; it is
not a formal certificate. A tested magnitude $b$ contributes to the certified
frontier only after the frozen candidate passes the independent verification
procedure above. Thus, the reported $b_{\mathrm{cert}}^{\mathrm{op}}$ is an
accepted tested magnitude under the operational criterion of
Sec.~\ref{sec:datadriven}, not an interpolated robustness radius.

Repeating this procedure across the tested values of $\alpha$ and $K$
produces the paired empirical and certified profiles in
Fig.~\ref{fig:all-cert-results}. Each displayed certified point has its own
per-specification confidence statement from
Theorem~\ref{thm:conformal}; connecting lines are visual guides only and do
not certify intermediate budgets or imply simultaneous validity across the
displayed collection. Adaptive or collection-level reporting follows the
confidence-control procedures in Appendix~\ref{app:collection-confidence}.

\paragraph{Practical certificate design.}
The decomposition in~\eqref{eq:app-rho-cert} provides a direct way to diagnose
a weak certificate. A large Base points to the synthesized witness itself:
the propagated initial and drift contributions are large relative to terminal
separation. A large Margin indicates that the calibration scores require
substantial residual slack. A large Bridge indicates limited episode-level
statistical resolution. These effects should be treated separately rather
than increasing every synthesis and verification budget indiscriminately.

Long horizons require particular care because both the learned drift term and
the calibrated residual slack propagate through $\cG_T(\kappa)$. When
$\kappa=1$, $\cG_T(\kappa)=T$, so even a small residual allowance can become
dominant for large $T$. For long-horizon problems, selecting or fixing
$\kappa<1$ during candidate design can substantially reduce this propagation:
$\cG_T(\kappa)$ is then bounded by $(1-\kappa)^{-1}$ and $\kappa^T$ decays
with the horizon. This is nevertheless a design trade-off rather than a rule
that smaller $\kappa$ is always better, since decreasing $\kappa$ also
tightens the one-step drift condition and may increase $b_3^*$ or the
resulting residuals. Any such choice must be made before independent
calibration and then kept fixed during verification.

The verification sample sizes play different roles. Increasing the number of
calibration episodes $N$ improves the resolution of the conformal violation
bound and can reduce Bridge, whereas increasing the successor count $M$
reduces $\epsilon_M$ and can improve the drift residuals that determine
Margin. The barrier output bound $\Bmax$ also enters this trade-off through
the finite-successor requirement: an unnecessarily large $\Bmax$ increases
the required successor budget, while a value that is too small can restrict
the witness class. In practice, certificate refinement is therefore most
efficient when the limiting contribution in
\eqref{eq:app-rho-cert} is identified first and the corresponding synthesis
or verification component is improved.

\subsection{Complete FoSeRL certification results} \label{app:certificate-reporting} Figure~\ref{fig:all-cert-results} reports the complete robustness profiles for all retained RL-policy settings. Each environment--policy specification is shown as a paired set of panels: the upper panel reports the formally certified FoSeRL frontier, while the lower panel shows the corresponding empirical pre-verification reference. Within each panel, curves correspond to the tested temporal attack densities $K/T$, the horizontal axis gives the attack magnitude $b$, and the vertical axis gives the normalized degradation tolerance $\alpha_{\%}$. All formal FoSeRL results use the target reliability $\rho=0.90$ and per-specification confidence parameter $\beta=0.001$. Across environments and policy families, relaxing the degradation tolerance generally permits larger attack magnitudes, whereas increasing the temporal attack density typically contracts the admissible budget. The empirical and certified profiles exhibit the same broad structure, but the gap between them varies substantially across settings. 
\begin{wrapfigure}{r}{0.45\textwidth}
    \vspace{-0.8em}
    \centering
    \includegraphics[width=\linewidth]{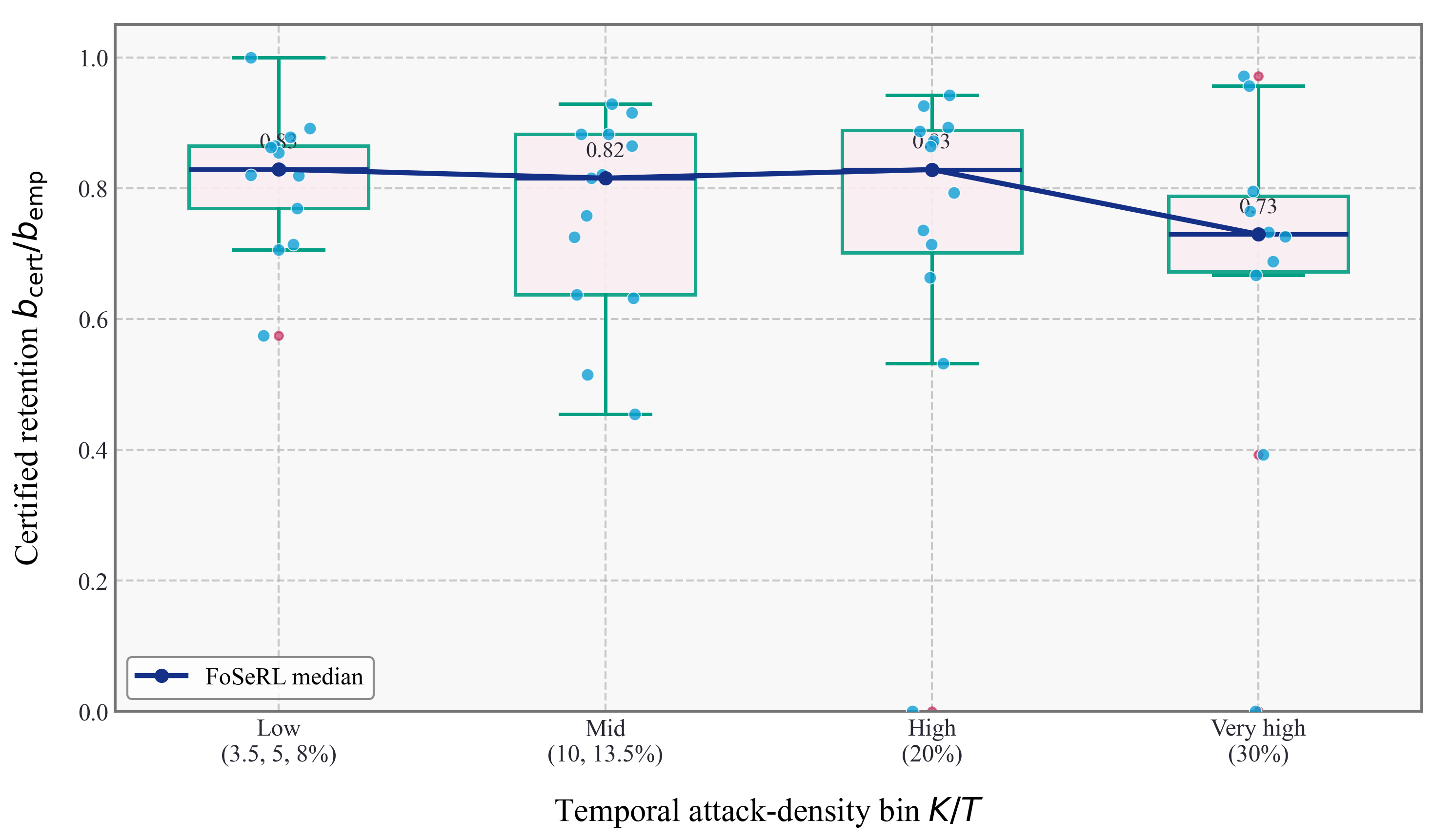}
    \vspace{-2em}
    \caption{\scriptsize
    FoSeRL certified-budget retention
    $b_{\mathrm{cert}}/b_{\mathrm{emp}}$ across temporal attack-density
    bins. Points denote individual reported specifications and the blue
    line connects bin medians.
    }
    \label{fig:app-retention-summary}
\end{wrapfigure}
The aggregate view in Fig.~\ref{fig:app-retention-summary} shows that this
retention remains broadly stable across the low-to-high density regimes,
with a larger reduction appearing in the very-high-density bin. In some cases the formal frontier closely tracks the empirical reference, while in others verification produces a visibly more conservative budget. This variation reflects the cost of converting a design-stage empirical frontier into an independently verified certificate and also shows that the conservatism of the certification stage is environment-- and policy-dependent rather than uniform across tasks. Each marker corresponds to an individually tested specification; connecting lines are included only as visual guides and do not imply certification of intermediate or untested budgets.

\begin{figure*}[p]
\centering
\includegraphics[width=0.92\textwidth]{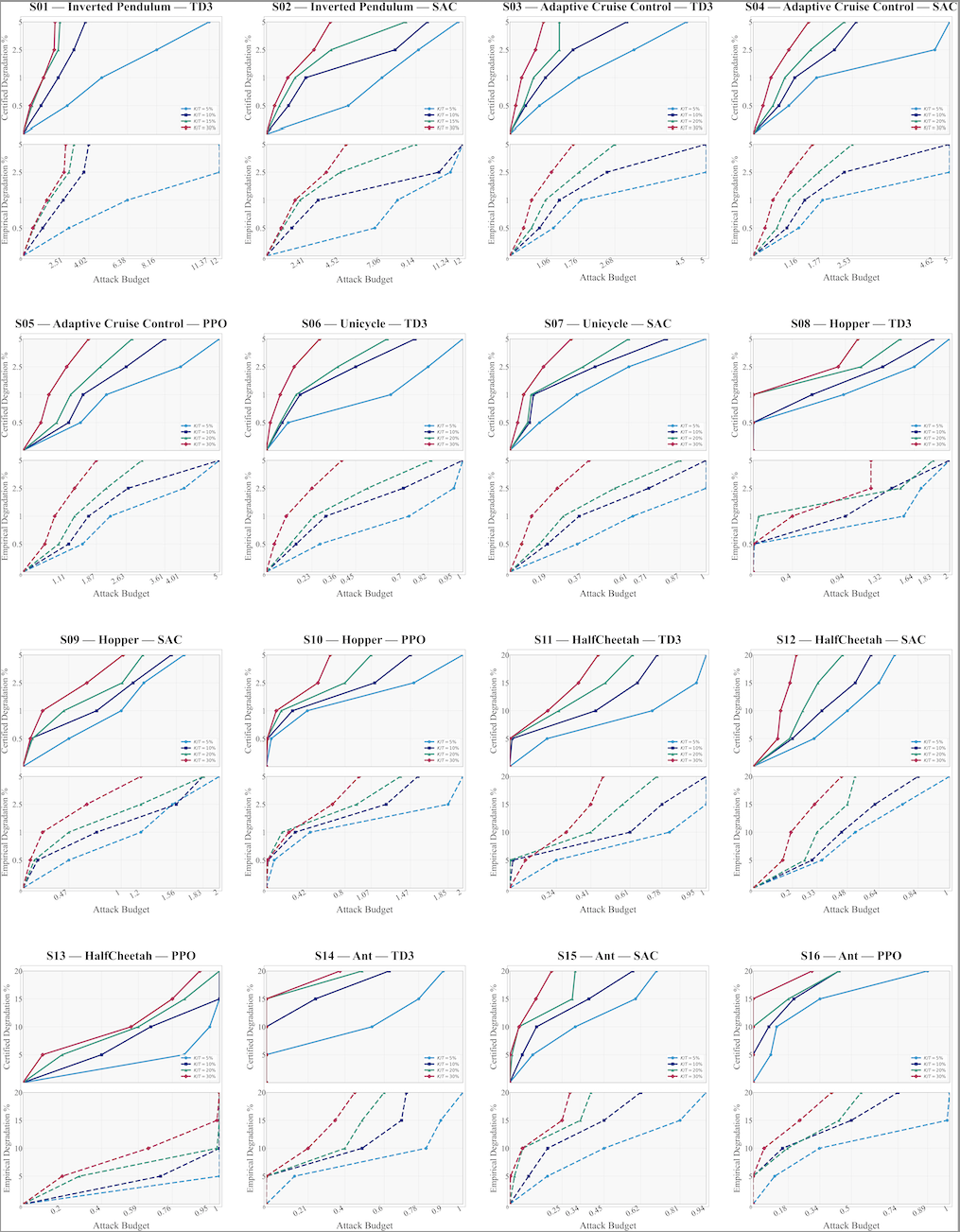}
\caption{ Complete empirical and certified robustness profiles for all retained RL-policy experiments. Each environment--policy setting is shown as a paired panel: the upper plot gives the FoSeRL-certified frontier and the lower plot the corresponding empirical pre-verification reference. Curves represent the tested temporal attack densities $K/T$; the horizontal axis is the attack magnitude $b$, and the vertical axis is the normalized degradation tolerance $\alpha_{\%}$. Formal FoSeRL certificates target reliability $\rho=0.90$ with per-specification confidence parameter $\beta=0.001$. Across settings, larger degradation tolerances generally support larger attack budgets, while denser temporal attacks tend to reduce the admissible magnitude. The gap between the empirical and certified frontiers varies across environments and policies, revealing setting-dependent verification conservatism. Markers denote individually tested specifications; connecting lines are visual guides only and do not certify interpolation or untested budgets. }
\label{fig:all-cert-results}
\end{figure*}

\subsection{Comparison with policy smoothing}
\label{app:rs-comparison}
\begin{wrapfigure}{r}{0.45\textwidth}
    \vspace{-0.8em}
    \centering
    \includegraphics[width=\linewidth]{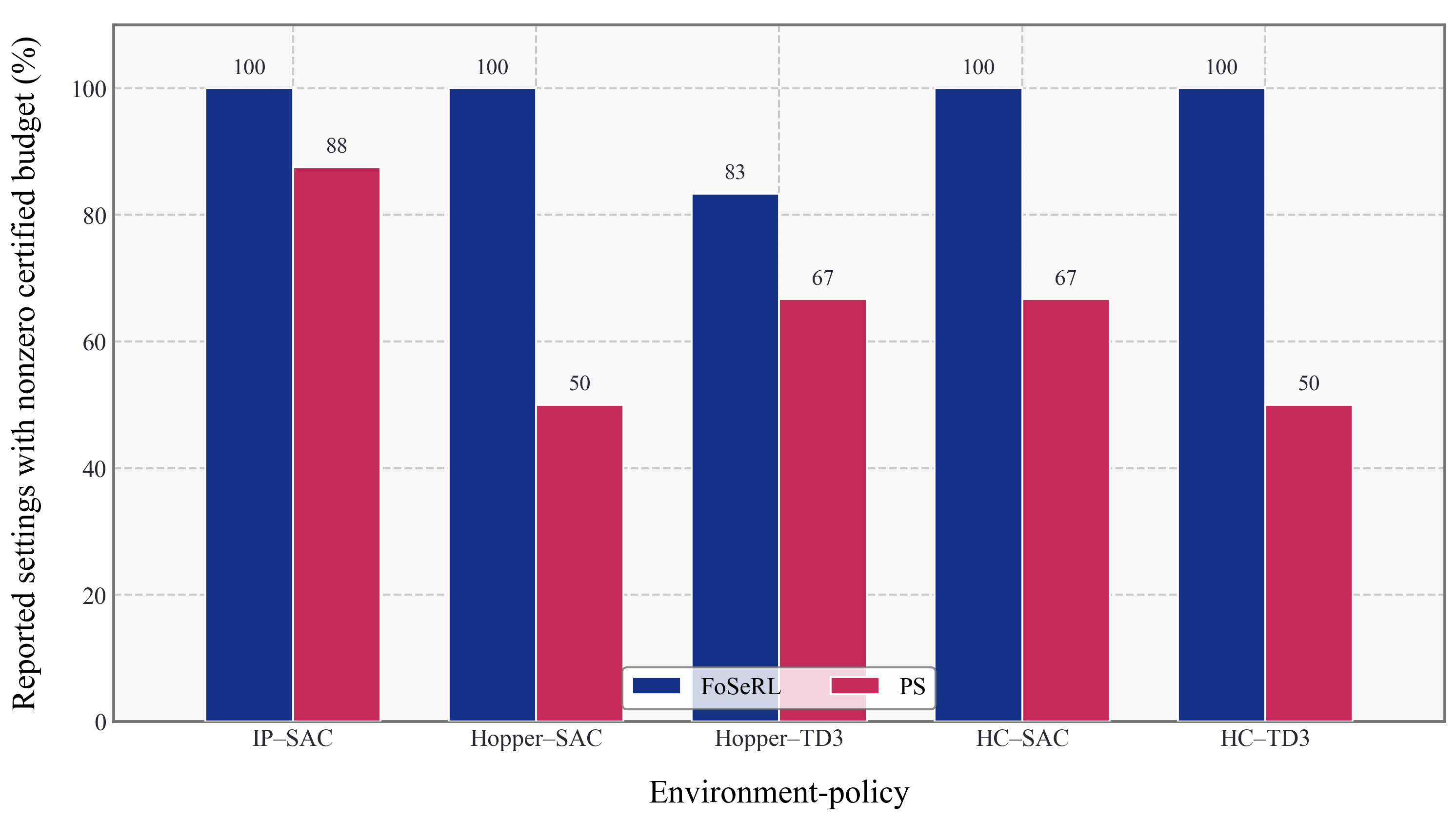}
    \vspace{-1em}
    \caption{\scriptsize
    Fraction of reported settings with a nonzero certified budget for
    FoSeRL and policy smoothing (PS), grouped by environment--policy pair.
    }
    \label{fig:app-nonzero-summary}
    \vspace{-1em}
\end{wrapfigure}
Table~\ref{tab:app-foserl-rs-comparison} reports the detailed external-baseline
comparison, while Fig.~\ref{fig:app-nonzero-summary} summarizes how often each
method returns a nonzero certificate across the reported specifications.
Since no existing formal method matches our complete setting, we align policy
smoothing (PS) with FoSeRL on the action channel, paired degradation event,
and reliability/confidence targets, while retaining the native $\ell_2$
perturbation geometry of the smoothing certificate. The PS
guarantee is worst-case over the corresponding $\ell_2$ perturbation set and
therefore does not depend on a particular attack optimizer; uniform
precommitted perturbations are used only for the associated empirical rollout
calculations. Full benchmark-specific attack and norm specifications are given
in Table~\ref{tab:fixed-case-study-settings}, and the broader positioning of PS
relative to FoSeRL is discussed in Appendix~\ref{app:rel}.
Comparable PS certificates are obtained for Inverted Pendulum, Hopper, and
HalfCheetah. ACC is excluded because the available smoothing-noise channel
does not coincide with the attacked ego-control channel, while Unicycle does
not meet the required nominal-reliability condition at the tested tolerance.
For Ant, PS yields no nonzero certified budget at the tested settings. Across
the comparable cases, FoSeRL consistently certifies substantially larger
budgets, including several specifications for which PS certifies only the
zero budget.

\newpage
\begingroup
\footnotesize
\newlength{\CompTableColSep}
\setlength{\CompTableColSep}{11pt}
\newcommand{\CompTableRowStretch}{0.85}
\setlength{\tabcolsep}{\CompTableColSep}
\renewcommand{\arraystretch}{\CompTableRowStretch}

\setlength{\LTpre}{10pt}
\setlength{\LTpost}{0pt}

\setlength{\LTleft}{\fill}
\setlength{\LTright}{\fill}

\newcommand{\rszero}[1]{%
  \cellcolor{PSZeroFill}\textbf{#1}%
}

\newcommand{\foserlhi}[1]{%
  \cellcolor{FoSeRLGoodFill}\textbf{#1}%
}

\newcommand{\ktline}{%
  \arrayrulecolor{black!18}%
  \cmidrule(lr){4-10}%
  \arrayrulecolor{black}%
}

\begin{longtable}{
@{}
l
c
c
c
c
c
c
cc
c
@{}
}

\caption{\small
Detailed FoSeRL--policy-smoothing (PS) comparison on aligned
environment--policy settings. For each setting, we report the search limit
$b_{\max}$, attack cardinality $K$ and density $K/T$, normalized degradation
tolerance $\alpha_{\%}$, empirical design reference $b_{\mathrm{emp}}$, and
the certified budgets returned by FoSeRL and PS. The normalized gap is
$\left|b_{\mathrm{cert}}^{\mathrm{FoSeRL}}
-b_{\mathrm{cert}}^{\mathrm{PS}}\right|/b_{\max}$.
A zero PS entry indicates that no nonzero budget was certified at that
specification.
\vspace{-0.1cm}}
\label{tab:app-foserl-rs-comparison}
\\

\toprule

\multirow{2}{*}{\textbf{Environment}}
&
\multirow{2}{*}{\textbf{Policy}}
&
\multirow{2}{*}{$\mathbf{b_{\max}}$}
&
\multirow{2}{*}{$\mathbf{K}$}
&
\multirow{2}{*}{$\mathbf{K/T}$}
&
\multirow{2}{*}{$\boldsymbol{\alpha_{\%}}$}
&
\multirow{2}{*}{$\mathbf{b_{\mathrm{emp}}}$}
&
\multicolumn{2}{c}{\boldmath$\mathbf{b_{\mathrm{cert}}}$}
&
\multirow{2}{*}{\textbf{Norm. gap}}
\\

\cmidrule(lr){8-9}

&
&
&
&
&
&
&
\textbf{FoSeRL}
&
\textbf{PS}
&
\\

\midrule
\endfirsthead

\toprule

\multirow{2}{*}{\textbf{Environment}}
&
\multirow{2}{*}{\textbf{Policy}}
&
\multirow{2}{*}{$\mathbf{b_{\max}}$}
&
\multirow{2}{*}{$\mathbf{K}$}
&
\multirow{2}{*}{$\mathbf{K/T}$}
&
\multirow{2}{*}{$\boldsymbol{\alpha_{\%}}$}
&
\multirow{2}{*}{$\mathbf{b_{\mathrm{emp}}}$}
&
\multicolumn{2}{c}{\boldmath$\mathbf{b_{\mathrm{cert}}}$}
&
\multirow{2}{*}{\textbf{Norm. gap}}
\\

\cmidrule(lr){8-9}

&
&
&
&
&
&
&
\textbf{FoSeRL}
&
\textbf{PS}
&
\\

\midrule
\endhead

\midrule
\multicolumn{10}{r}{\scriptsize\emph{Continued on next page}}
\\
\endfoot

\bottomrule
\endlastfoot


\multirow{8}{*}{
\makecell[l]{Inverted Pendulum\\[1pt]($T=200$)}
}
&
\multirow{8}{*}[1.5ex]{SAC}
&
\multirow{8}{*}[1.5ex]{4}
& 7 & 3.5\% & 5\%
& 4.00
& 4.00
& 0.0198
& 0.9951
\\

\ktline

&
&
& 16 & 8\% & 2.5\%
& 2.60
& 2.00
& 0.0129
& 0.4968
\\

&
&
& 16 & 8\% & 5\%
& 3.47
& 3.00
& 0.0132
& 0.7467
\\

\ktline

&
&
& 27 & 13.5\% & 1\%
& 1.02
& \foserlhi{0.65}
& \rszero{0.0000}
& 0.1625
\\

&
&
& 27 & 13.5\% & 2.5\%
& 1.84
& 1.50
& 0.0101
& 0.3725
\\

&
&
& 27 & 13.5\% & 5\%
& 3.47
& 3.00
& 0.0101
& 0.7475
\\

\ktline

&
&
& 40 & 20\% & 2.5\%
& 1.72
& 1.50
& 0.0081
& 0.3730
\\

&
&
& 40 & 20\% & 5\%
& 2.80
& 2.50
& 0.0083
& 0.6229
\\

\midrule


\multirow{20}{*}{
\makecell[l]{Hopper\\[1pt]($T=1000$)}
}
&
\multirow{8}{*}[1.5ex]{SAC}
&
\multirow{8}{*}[1.5ex]{2}
& 50 & 5\% & 1\%
& 1.40
& \foserlhi{1.00}
& \rszero{0.0000}
& 0.5000
\\

&
&
& 50 & 5\% & 5\%
& 1.84
& 1.64
& 0.0032
& 0.8184
\\

\ktline

&
&
& 100 & 10\% & 1\%
& 0.99
& \foserlhi{0.75}
& \rszero{0.0000}
& 0.3750
\\

&
&
& 100 & 10\% & 5\%
& 1.84
& 1.51
& 0.0023
& 0.7539
\\

\ktline

&
&
& 200 & 20\% & 1\%
& 0.79
& \foserlhi{0.42}
& \rszero{0.0000}
& 0.2100
\\

&
&
& 200 & 20\% & 5\%
& 1.84
& 1.22
& 0.0016
& 0.6092
\\

\ktline

&
&
& 300 & 30\% & 1\%
& 0.51
& \foserlhi{0.20}
& \rszero{0.0000}
& 0.1000
\\

&
&
& 300 & 30\% & 5\%
& 1.05
& 1.02
& 0.0013
& 0.5094
\\

\cmidrule(lr){2-10}


&
\multirow{12}{*}[1.5ex]{TD3}
&
\multirow{12}{*}[1.5ex]{2.5}
& 50 & 5\% & 1\%
& 1.11
& \foserlhi{0.92}
& \rszero{0.0000}
& 0.3680
\\

&
&
& 50 & 5\% & 2.5\%
& 1.92
& 1.64
& 0.0020
& 0.6552
\\

&
&
& 50 & 5\% & 5\%
& 2.44
& 2.00
& 0.0032
& 0.7987
\\

\ktline

&
&
& 100 & 10\% & 1\%
& 0.95
& \foserlhi{0.60}
& \rszero{0.0000}
& 0.2400
\\

&
&
& 100 & 10\% & 2.5\%
& 1.82
& 1.32
& 0.0014
& 0.5274
\\

&
&
& 100 & 10\% & 5\%
& 1.97
& 1.83
& 0.0023
& 0.7311
\\

\ktline

&
&
& 200 & 20\% & 1\%
& 0.95
& 0.00
& \rszero{0.0000}
& 0.0000
\\

&
&
& 200 & 20\% & 2.5\%
& 1.54
& 1.10
& 0.0010
& 0.4396
\\

&
&
& 200 & 20\% & 5\%
& 1.62
& 1.50
& 0.0016
& 0.5994
\\

\ktline

&
&
& 300 & 30\% & 1\%
& 0.91
& 0.00
& \rszero{0.0000}
& 0.0000
\\

&
&
& 300 & 30\% & 2.5\%
& 0.91
& 0.87
& 0.0008
& 0.3477
\\

&
&
& 300 & 30\% & 5\%
& 1.46
& 1.07
& 0.0013
& 0.4275
\\

\midrule


\multirow{20}{*}{
\makecell[l]{HalfCheetah\\[1pt]($T=500$)}
}
&
\multirow{12}{*}[1.5ex]{SAC}
&
\multirow{12}{*}[1.5ex]{1}
& 25 & 5\% & 5\%
& 0.54
& \foserlhi{0.31}
& \rszero{0.0000}
& 0.3100
\\

&
&
& 25 & 5\% & 10\%
& 0.68
& 0.48
& 0.0024
& 0.4776
\\

&
&
& 25 & 5\% & 20\%
& 0.82
& 0.72
& 0.0032
& 0.7168
\\

\ktline

&
&
& 50 & 10\% & 5\%
& 0.44
& \foserlhi{0.20}
& \rszero{0.0000}
& 0.2000
\\

&
&
& 50 & 10\% & 10\%
& 0.68
& 0.35
& 0.0017
& 0.3483
\\

&
&
& 50 & 10\% & 20\%
& 0.68
& 0.60
& 0.0023
& 0.5977
\\

\ktline

&
&
& 100 & 20\% & 5\%
& 0.22
& \foserlhi{0.19}
& \rszero{0.0000}
& 0.1900
\\

&
&
& 100 & 20\% & 10\%
& 0.34
& 0.25
& 0.0012
& 0.2488
\\

&
&
& 100 & 20\% & 20\%
& 0.58
& 0.46
& 0.0016
& 0.4584
\\

\ktline

&
&
& 150 & 30\% & 5\%
& 0.17
& \foserlhi{0.13}
& \rszero{0.0000}
& 0.1300
\\

&
&
& 150 & 30\% & 10\%
& 0.21
& 0.14
& 0.0010
& 0.1390
\\

&
&
& 150 & 30\% & 20\%
& 0.32
& 0.22
& 0.0013
& 0.2187
\\

\cmidrule(lr){2-10}


&
\multirow{8}{*}[1.5ex]{TD3}
&
\multirow{8}{*}[1.5ex]{1.5}
& 25 & 5\% & 15\%
& 1.16
& \foserlhi{0.95}
& \rszero{0.0000}
& 0.6333
\\

&
&
& 25 & 5\% & 20\%
& 1.16
& 1.00
& 0.0007
& 0.6662
\\

\ktline

&
&
& 50 & 10\% & 15\%
& 0.71
& \foserlhi{0.65}
& \rszero{0.0000}
& 0.4333
\\

&
&
& 50 & 10\% & 20\%
& 0.85
& 0.75
& 0.0005
& 0.4997
\\

\ktline

&
&
& 100 & 20\% & 15\%
& 0.52
& \foserlhi{0.49}
& \rszero{0.0000}
& 0.3267
\\

&
&
& 100 & 20\% & 20\%
& 0.71
& 0.63
& 0.0004
& 0.4197
\\

\ktline

&
&
& 150 & 30\% & 15\%
& 0.44
& \foserlhi{0.35}
& \rszero{0.0000}
& 0.2333
\\

&
&
& 150 & 30\% & 20\%
& 0.62
& 0.45
& 0.0003
& 0.2998
\\

\end{longtable}

\endgroup

\paragraph{Qualitative comparison.}
The quantitative separation in
Table~\ref{tab:app-foserl-rs-comparison} is also visible in representative
closed-loop rollouts. Figures~\ref{fig:qual-ip} and~\ref{fig:qual-mujoco}
show nominal executions, executions at FoSeRL-certified budgets, and attacks
beyond those certified budgets for Inverted Pendulum, Hopper, and
HalfCheetah. At the reported FoSeRL budgets, the illustrated degradation
traces remain below the prescribed terminal tolerance, whereas the larger
uncertified perturbations produce substantially greater degradation and, in
the Hopper example, early collapse. The corresponding PS certificates remain
confined to much smaller perturbation magnitudes, consistent with the
quantitative comparison above.

\begin{figure}[t]
  \centering
  \includegraphics[width=0.92\textwidth]{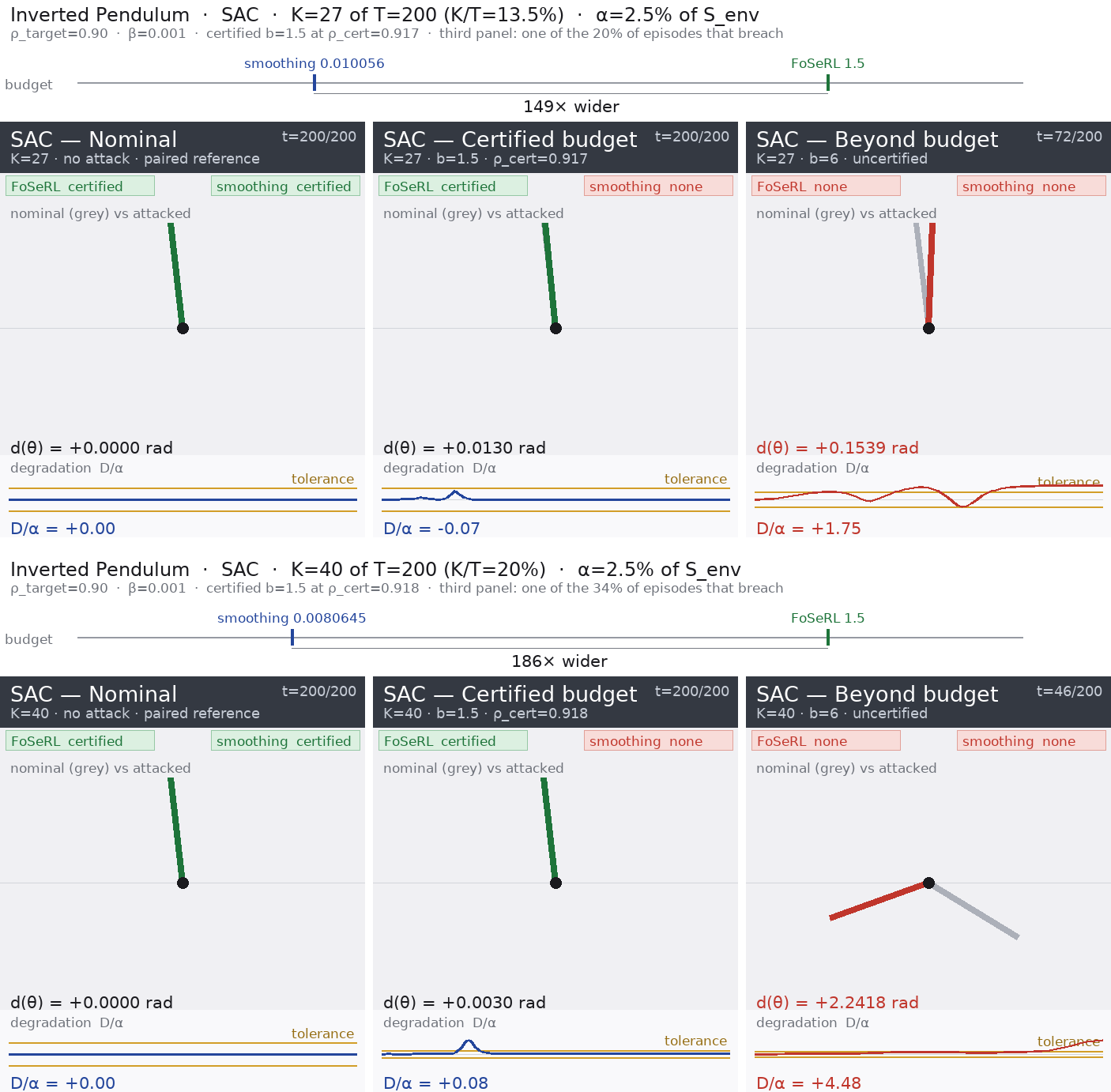}
  \caption{Inverted pendulum (SAC), at $K=27$ (top block) and $K=40$ (bottom),
  $\alpha=2.5\%$ of $S_{\mathrm{env}}$, $\rho_{\mathrm{target}}=0.90$,
  $\beta=10^{-3}$. Each block shows the nominal rollout, the rollout at FoSeRL's
  certified budget, and a rollout beyond it, with the attacked pole in colour
  against the nominal in grey. The budget bar contrasts the two certificates;
  note the scopes differ, smoothing being worst-case over
  $\|\delta\|_2\le b\sqrt{n_aK}$ and FoSeRL distributional under the declared
  attack law. $D/\alpha$ is signed, so a marginally negative value ($-0.07$, top
  middle) indicates an attack realisation that slightly helped that episode. The
  third panel of each block is one of the $20\%$ ($K=27$) and $34\%$ ($K=40$) of
  episodes that breach beyond the certified budget; breaching outside the
  certified region is the certificate behaving as specified.}
  \label{fig:qual-ip}
\end{figure}

\begin{figure}[t]
  \centering
  \begin{subfigure}{\textwidth}
    \centering
    \includegraphics[width=\textwidth]{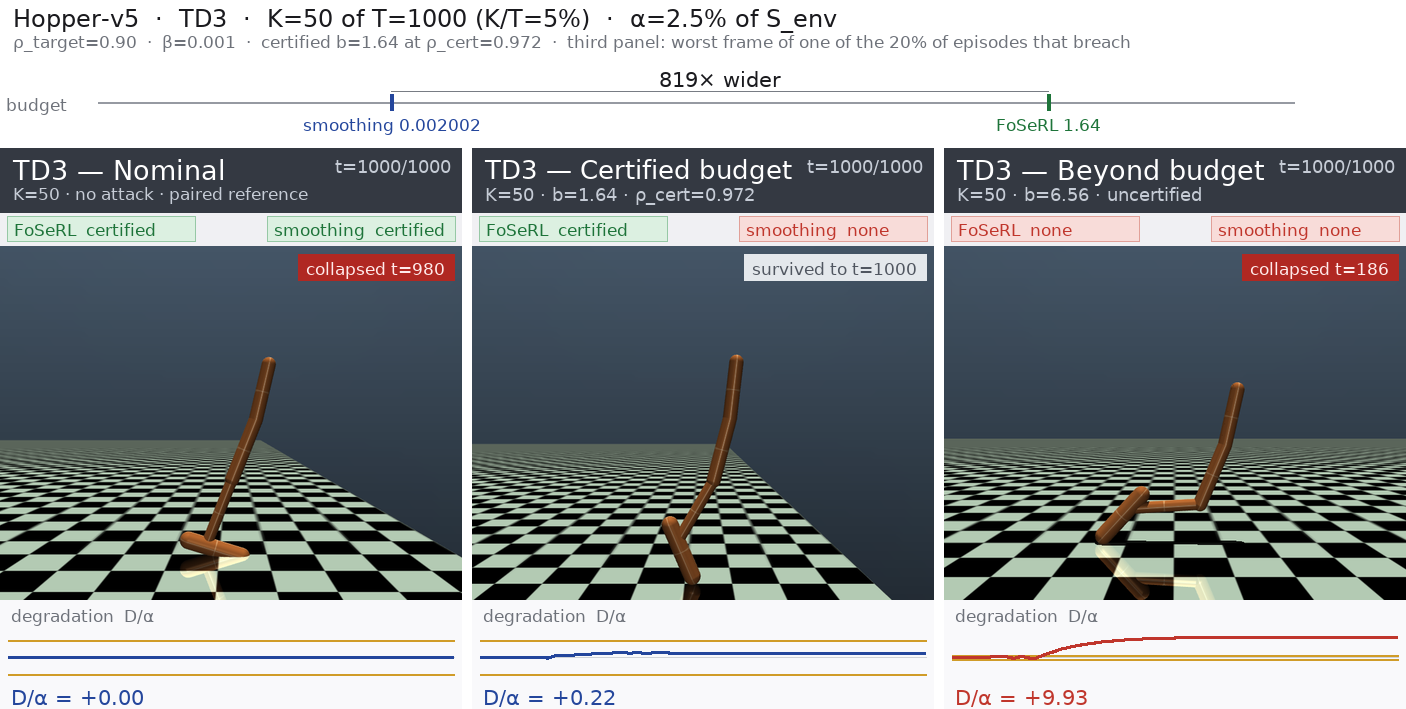}
    \caption{Hopper-v5 (TD3), $K=50$ of $T=1000$, $\alpha=2.5\%$; certified
    $b=1.64$ at $\rho_{\mathrm{cert}}=0.972$, against a smoothing budget of
    $2.0\times10^{-3}$.}
    \label{fig:qual-hopper}
  \end{subfigure}\\[6pt]
  \begin{subfigure}{\textwidth}
    \centering
    \includegraphics[width=\textwidth]{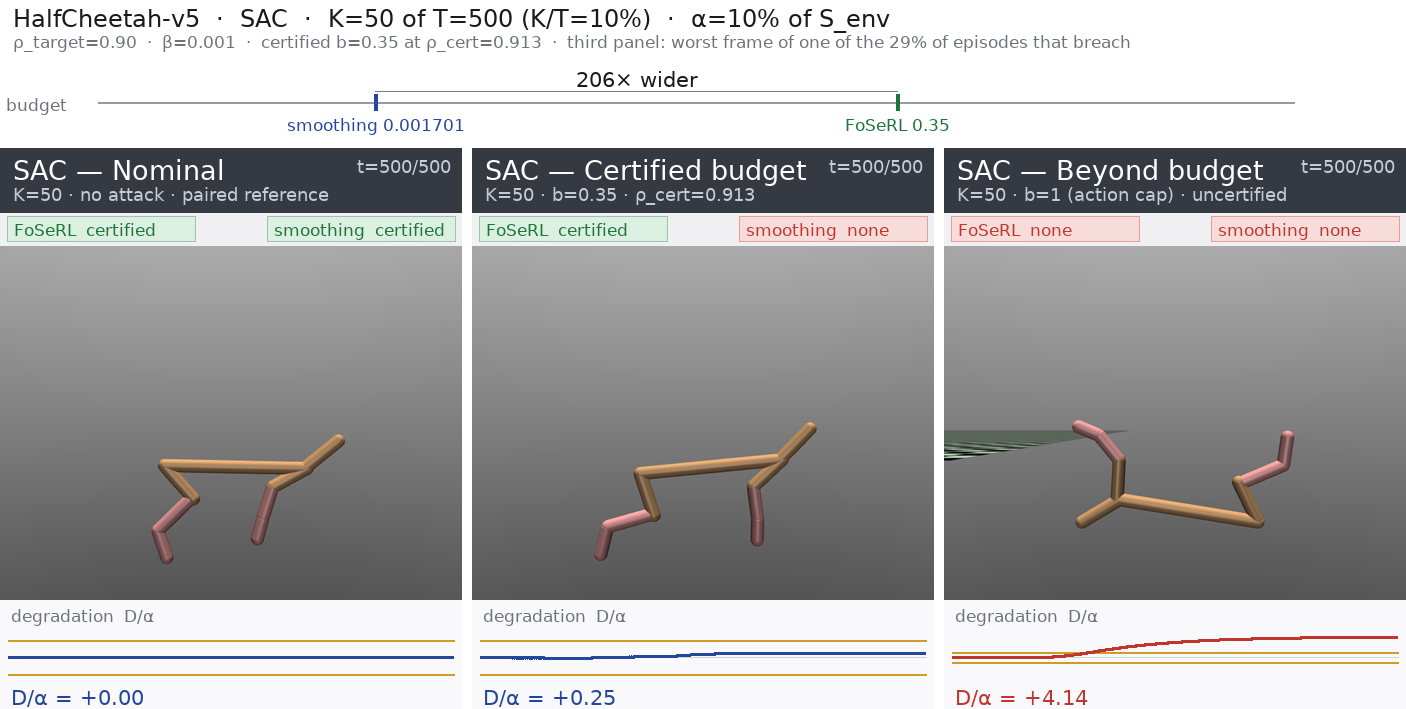}
    \caption{HalfCheetah-v5 (SAC), $K=50$ of $T=500$, $\alpha=10\%$; certified
    $b=0.35$ at $\rho_{\mathrm{cert}}=0.913$, against a smoothing budget of
    $1.7\times10^{-3}$.}
    \label{fig:qual-hc}
  \end{subfigure}
  \caption{Certified and uncertified regimes on MuJoCo locomotion, the setting in
  which neither \citet{kumar2022policy} nor \citet{mu2024reward} report results.
  Panels show the nominal rollout, the rollout at FoSeRL's certified budget, and
  a rollout beyond it, with the running degradation $D/\alpha$ beneath each.
  Because the two branches share the noise realisation but diverge after the
  first perturbation, their termination times differ: in
  \subref{fig:qual-hopper} the nominal episode terminates at $t=980$ while the
  attacked one at the certified budget survives to the horizon, and beyond the
  budget it collapses at $t=186$. Third panels are drawn from the $20\%$
  (Hopper) and $29\%$ (HalfCheetah) of episodes that breach outside the certified
  region.}
  \label{fig:qual-mujoco}
\end{figure}

\paragraph{Scope of the comparison.}
These visualizations should be interpreted together with the differing
certificate scopes. Policy smoothing provides a worst-case guarantee over its
corresponding $\ell_2$ perturbation set, whereas the data-driven FoSeRL
certificate is distributional under the declared attack law. The
beyond-certificate panels are deliberately selected from episodes that violate
the degradation tolerance outside the certified region; they illustrate the
behavior beyond the guarantee rather than a failure of the certificate.

\FloatBarrier

\newpage
\section{Additional Discussion}
\label{sec:discussion}
\vspace{-0.3cm}
This section expands on Sec.~\ref{sec:main-discussion}, explaining the
long-horizon verification design, the information supplied by the barrier
witness, and the interpretation and refinement of the resulting certificates.
Technical boundary cases and sampling requirements are collected in
Appendix~\ref{app:formulation}; literature positioning and external-baseline
selection are discussed separately in Appendix~\ref{app:rel}.

\subsection{From training-time robustness to long-horizon certification}
\label{app:discussion-training}
The framework of \citet{taheri2026probabilistic} models stochastic training
as a dynamical system and certifies the accuracy of its terminal parameters
under data perturbations, already incorporating a training horizon and a
time-dependent barrier. FoSeRL changes both the intervention and the
certified object: the policy is already fixed, perturbations affect its
executed actions, and performance is measured by the accumulated loss
relative to a matched nominal episode. The paired-state construction and
return-gap coordinate make this sequential objective a terminal safety
property, while the separate cardinality and magnitude budgets capture the
temporal allocation of attack resources.
Long horizons also require controlling the accumulation of local drift
allowances. With the earlier linear term $b_3T$, a verification correction
$\eta_3$ contributes $\eta_3T/b_2$ to the failure bound; for example,
$\eta_3/b_2=10^{-3}$ already contributes one at $T=1000$, exceeding the
failure allowance $1-\rho=0.1$. FoSeRL retains the geometric accumulation
$\cG_T(\kappa)$ and calibrates the residual allowance using the final
reliability bound. For $\kappa<1$,
$\cG_T(\kappa)\le\min\{T,(1-\kappa)^{-1}\}$, avoiding unnecessary linear
accumulation without removing the horizon dependence of the problem
(Appendix~\ref{app:verification-planning}). Nor can $T$ simply be replaced
by $K$: an altered state can affect later decisions and rewards after the
last attacked action.

\subsection{Episode-level calibration and statistical validity}
\label{app:discussion-calibration}
FoSeRL calibrates complete episodes rather than treating their dependent
transitions as independent observations. The score in
Definition~\ref{def:score} aggregates initial and drift residuals while
retaining terminal separation as a hard condition. All calibration episodes
remain in the sample, including those assigned infinite scores because of
unsafe-terminal separation violations. No parametric score distribution is
required; validity instead relies on independent episodes and fresh
conditional successor samples.
Calibration balances the propagated residual margin against the probability
of episode-level condition violations. Simultaneous validity over ranks
permits selecting this balance from the calibration data without an
additional rank-selection penalty. The finite-successor bound then connects
the sampled drift residuals to their exact counterparts: conditional on a
violating episode, its first exact drift violation is detected with
probability at least $1-\beta_0$. This yields the transfer penalty
$\bar v_{k^*}/(1-\beta_0)$ without a union bound over time. Combined with
the good-prefix barrier argument, these steps establish the end-to-end
operational guarantee of Theorem~\ref{thm:conformal}, rather than only
marginal coverage of a score.

\subsection{Reliability, confidence, and violation levels}
\label{app:discussion-probability-levels}
The target $\rho$ concerns a fresh attack--episode draw: it is the required
probability of satisfying $\Dtraj\le\alpha$ under $\QQ_{K,b}$. The reported
$\rho_{\mathrm{cert}}$ is the lower bound in
\eqref{eq:app-rho-cert}, and acceptance requires
$\rho_{\mathrm{cert}}\ge\rho$. In contrast, $\beta$ controls uncertainty
over calibration data: conditional on the frozen specification,
Theorem~\ref{thm:conformal} bounds false acceptance by $\beta$.
The auxiliary $\beta_0$ controls the one-sided successor approximation at a
fixed state--attack--time tuple through Lemma~\ref{lem:finiteM}. Its effect
is incorporated into the transfer penalty, not deducted again from the
outer confidence, which remains $1-\beta$. Finally, $\bar v_{k^*}$ bounds
out-of-sample violations of the selected empirical episode conditions; it is
neither an empirical failure frequency nor the terminal-degradation risk
itself. The latter also includes the dynamical contribution, as made
explicit by the Base--Margin--Bridge decomposition in
\eqref{eq:app-cert-components}--\eqref{eq:app-rho-cert}.

\subsection{Barrier witnesses versus direct terminal certification}
\label{app:discussion-direct}
\textbf{FoSeRL provides a dynamical witness in addition to a scalar
robustness guarantee.} Direct terminal certification applies a one-sided
binomial bound to independent observations of the same paired terminal
event under the same operational specification. It is a useful formal
reference when the required output is only the terminal probability, and
can yield larger certified budgets with less computation because it does
not require barrier synthesis or conditional successor queries. FoSeRL's
additional contribution is the structured certificate supporting that
probability and the analysis it enables.

\textbf{State--time structure.}
The witness $B_{\phi^*}(\xhat,t)$ links initial conditions, expected one-step
evolution, and terminal separation through the conditions of
Definition~\ref{def:barrier}. These relations make the dynamics part of
the certification argument, rather than reducing each episode solely to a
terminal success indicator. The same structure underlies the ideal
uniform guarantee and its finite-sample operational counterpart, with their
respective scopes stated in Theorems~\ref{thm:ideal}
and~\ref{thm:conformal}.

\textbf{Condition-specific diagnosis and refinement.}
The residuals in Definition~\ref{def:residuals} distinguish initial-value,
terminal-separation, and drift violations, including the sampled phases at
which drift conditions are difficult to satisfy. The reliability
decomposition further separates the frozen witness contribution from
calibration slack and episode-level violations. This provides concrete
directions for refinement: concentrated late-horizon drift violations can
motivate targeted synthesis data or time-dependent features, whereas poor
terminal separation calls for improving the terminal witness. The scalar
binomial certificate alone does not provide this condition-based
decomposition. These diagnostics concern the witness, not certified
conditional failure probabilities at arbitrary states.

In Table~\ref{tab:main-baseline-results}, FoSeRL retains approximately
$71$--$99\%$ of Direct's certified magnitude across the representative
settings while additionally returning this state--time witness. Thus, the
comparison is not between formal and empirical validity: both methods
provide formal operational guarantees. It is a trade-off between scalar
budget tightness and computation on one side, and a dynamically structured
certificate with targeted diagnostic value on the other. When the purpose
includes understanding and refining the certifying conditions, the barrier
provides information beyond the terminal bound itself.

\subsection{Generality, simulator access, and attack coverage}
\label{app:discussion-scope}
FoSeRL separates policy evaluation from policy construction. Under the
stated execution and sampling assumptions, its verification argument does
not require a particular learning algorithm, policy architecture,
differentiable controller, or analytical dynamics family. The deployed
policy remains unchanged, including its chosen randomization mode; the
learned barrier is an auxiliary witness, not a replacement controller.
This generality uses a concrete interface: synchronously coupled rollouts
and repeatable conditional successor queries, with fresh complete one-step
randomness regenerated as specified in Sec.~\ref{sec:datadriven}.
The resulting guarantee concerns the execution process being sampled.
The distinction from offline RL and evaluation from single-branch logs is
discussed in Appendix~\ref{app:rel}.
The operational attack law need not be uniform. It may emphasize difficult
precommitted plans or design-selected mixtures, provided it is fixed before
calibration and attacks remain independent of the initial state and episode
randomness as required by Definition~\ref{def:jointlaws}.
Theorem~\ref{thm:conformal} certifies reliability under this law; uniform
coverage of all admissible precommitted attacks instead follows from the
ideal barrier conditions. State-adaptive adversaries require an extension
of the present attack model.

\subsection{Reading robustness frontiers and selecting specifications}
\label{app:discussion-frontiers}
The robustness profiles expose how attack frequency and magnitude interact
with the allowed return loss. Their interpretation retains the declared
policy execution, attack channel and norm, horizon, coupling, and
statistical target. The PS comparison in Appendix~\ref{app:rs-comparison}
also retains the distinction between operational and worst-case attack
coverage. Normalization by $S_{\mathrm{env}}$ supplies a common reporting
scale within an environment, not a percentage of each policy's nominal
return or a universal ranking across tasks.
The empirical design budget supports candidate construction, whereas an
accepted operational budget is supported by independent verification.
Neither is asserted to be the optimal robustness radius. The ideal frontier
has the nesting properties of Appendix~\ref{app:frontier-monotonicity};
operational certificates do not automatically inherit them when changing
the budget changes the attack law. Connecting lines therefore do not
certify interpolation or untested specifications.
New specifications can nevertheless be added incrementally without
predeclaring the entire future experiment family. Each frozen specification
receives its own audit; selecting a budget after exploration requires fresh
final calibration or appropriate selection-aware confidence control.
Appendix~\ref{app:collection-confidence} gives the finite-collection and
sequential protocols, including accounting for repeated tests and stopping.

\subsection{Performance robustness versus pathwise safety}
\label{app:discussion-performance}
Paired degradation complements nominal performance by measuring the loss
relative to the same episode without an attack. Shared randomness does not
force equal actions or states after a perturbation: each branch follows its
own closed-loop evolution, preserving the attack's subsequent consequences.
The running gap can increase or recover because individual reward
differences can have either sign. Accordingly, crossing the tolerance before
$T$ is not itself terminal failure, and the maximum in the verification
score aggregates residuals rather than return loss.
This formulation certifies cumulative performance without imposing an
additional pathwise constraint. Collision avoidance or all-time state
safety can be studied alongside it, but requires an explicit safety event
and corresponding verification conditions. A safety-violation indicator in
the augmented state offers one route for such an extension.

\subsection{Certificate refinement and policy improvement}
\label{app:discussion-refinement}
Separating construction from verification allows the witness to improve
without changing the meaning of acceptance. Richer features, alternative
bounded architectures, and targeted synthesis data may improve candidate
quality; the synthesis law may emphasize difficult cases rather than match
the operational law. Each candidate and verification design must then be
frozen before fresh independent calibration.
The decomposition in Appendix~\ref{app:certificate-workflow} also supports
targeted resource allocation. The episode count $N$ controls calibration
resolution, whereas the successor count $M$ controls the approximation
correction. The choice of $\kappa$ balances geometric propagation against
the strictness of the local drift inequality. These are distinct design
levers, not interchangeable ways to improve a certificate
(Appendix~\ref{app:verification-planning}). Beyond witness refinement,
residual diagnostics offer a direction for certificate-informed policy
improvement, with each updated policy treated as a new target and
independently re-certified.

\subsection{Limitations and outlook}
\label{app:discussion-limitations}
FoSeRL's simulation-based interface supports broad applicability but
requires repeated coupled rollouts and conditional successor queries.
Long horizons and expensive policy or simulator evaluations can therefore
increase computational cost, as quantified in
Appendix~\ref{app:verification-planning}. The reported case studies use
the original execution process without surrogate-assisted synthesis. For
larger-scale applications, lightweight surrogates could reduce
candidate-construction cost by generating synthesis rollouts. The candidate
would still be frozen and formally verified using fresh paired episodes
and conditional successors from the original execution process. Provided
the verification assumptions and confidence protocol are retained,
surrogate accuracy would affect candidate quality and certifiability, not
the stated validity of an accepted certificate. Such acceleration would
concern synthesis; formal verification would retain its original-process
sampling requirements.
The sufficient barrier conditions also trade some budget tightness for a
state--time witness; their conservatism is discussed in
Appendix~\ref{app:common-barrier-limitation}. These computational and
tightness considerations accompany the additional certificate structure,
while policy and environment generality hold within the execution, sampling,
and attack-coverage scope specified in
Appendix~\ref{app:discussion-scope}. The separation between synthesis and
verification provides a route to improving efficiency and witness quality
while retaining independent formal validation.

\section{Positioning, Baseline Selection, and Extended Comparisons}
Here we specify what FoSeRL certifies and how that object relates to
existing guarantees. We first place the relevant literature on the axes that
distinguish certified quantities from one another, and separate the methods that
share our problem formulation from those that solve a different one
(Sec.~\ref{app:rel}). We then state the criteria a method must meet to serve as a
quantitative baseline here, apply them, and explain why exactly one external
method qualifies (Sec.~\ref{app:why-smoothing}). Finally, we record the regime in
which that baseline has previously been evaluated, which differs from the
higher-dimensional, longer-horizon settings considered in our experiments
(Sec.~\ref{app:regime}).
\label{app:rel}
\subsection{Positioning}
Here we illustrate, using Table~\ref{tab:related} the exact characterization of our problem and what it certifies.
\begin{table}[ht]

\centering
\scriptsize

\begin{tabular}{@{}p{3.2in}cccccccc@{}}
\toprule
\textbf{Method} & \hd{Formal} & \hd{Channel} & \hd{Budget} & \hd{Fixed $\pi$} & \hd{Level} & \hd{Form} & \hd{Paired} & \hd{Resolved} \\
\midrule
\multicolumn{9}{@{}l}{\textit{Attacks on deployed policies}}\\
FGSM on policies \citep{huang2017adversarial}                & \No & obs & $\ell_p$ & \Y & ep & --- & \No & \No \\
Strategically-timed \citep{lin2017tactics}                    & \No & obs & $(K,b)$  & \Y & ep & --- & \No & \No \\
Critical-point \citep{sun2020stealthy}                        & \No & obs & $(K,b)$  & \Y & ep & --- & \No & \No \\
MAS action attack \citep{lee2020spatiotemporal}               & \No & act & $(K,b)$  & \Y & ep & --- & \No & \No \\
\addlinespace
\multicolumn{9}{@{}l}{\textit{Robust training: hardens a policy, certifies none}}\\
RARL \citep{pinto2017robust}                                  & \No & dyn & set      & \No & ep & --- & \No & \No \\
NR-/PR-MDP \citep{tessler2019action}                          & \No & act & $\ell_p$ & \No & ep & --- & \No & \No \\
SA-MDP \citep{zhang2020robust}                                & \Part & obs & $\ell_p$ & \No & ep & mean & \No & \No \\
ATLA \citep{zhang2021robust}                                  & \No & obs & $\ell_p$ & \No & ep & --- & \No & \No \\
WocaR \citep{liang2022wocar}                                  & \Part & obs & $\ell_p$ & \No & ep & mean & \No & \No \\
RADIAL \citep{oikarinen2021radial}                            & \Part & obs & $\ell_p$ & \No & st & wc & \No & \No \\
\addlinespace
\multicolumn{9}{@{}l}{\textit{Certified robustness for RL}}\\
CARRL \citep{everett2021certifiable}                          & \Y & obs & $\ell_p$ & \No & st & wc   & \No & \Y \\
Policy smoothing \citep{kumar2022policy}                      & \Y & obs & $\ell_2$ & \No & ep & mean & \No & \No \\
CROP \citep{wu2022crop}                                       & \Y & obs & $\ell_2$ & \No & ep & mean & \No & \No \\
COPA \citep{wu2022copa}                                       & \Y & data & $\ell_0$& \No & ep & wc   & \No & \No \\
ReCePS \citep{mu2024reward}                                   & \Y & act & $\ell_p$ & \No & ep & mean & \No & \No \\
CAROL \citep{yang2024carol}                                   & \Y & obs & $\ell_p$ & \No & ep & wc   & \No & \No \\
\addlinespace
\multicolumn{9}{@{}l}{\textit{Guarantees without an adversary model}}\\
Robust MDPs \citep{iyengar2005robust,nilim2005robust}         & \Y & kern & set     & \No & ep & wc   & \No & \Y \\
Offline RL / OPE \citep{levine2020offline}                    & \No & --- & ---      & \Y & ep & mean & \No & \No \\
HCOPE \citep{thomas2015high}                                  & \Y & --- & ---      & \Y & ep & mean & \No & \No \\
Statistical model checking \citep{legay2010statistical}       & \Y & --- & ---      & \Y & ep & tail & \No & \No \\
\addlinespace
\multicolumn{9}{@{}l}{\textit{Barrier certificates}}\\
Stochastic barriers \citep{prajna2007framework,steinhardt2012finite} & \Y & --- & --- & \Y & ep & tail & \No & \Y \\
Neural synthesis \citep{abate2021fossil,anand2023formally}    & \Y & --- & ---      & \Y & ep & tail & \No & \Y \\
CRABS \citep{luo2021crabs}                                    & \Y & dyn & set      & \No & ep & wc   & \No & \Y \\
\addlinespace
\multicolumn{9}{@{}l}{\textit{Static robustness certification (parent framework)}}\\
Poisoning certificates \citep{taheri2026probabilistic}        & \Y & data & $\ell_p$& \Y & --- & tail & \Y & \No \\
\midrule
\textbf{FoSeRL} (this paper)                                  & \Y & \textbf{act} & $\mathbf{(K,b)}$ & \Y & ep & \textbf{tail} & \Y & \Y \\
\bottomrule

\end{tabular}
\caption{Related work, and where FoSeRL differs. \textbf{Formal}: a mathematical guarantee rather than an empirical rate (\textcolor{gray}{$\circ$}: formal bounds used internally, empirical robustness reported). \textbf{Channel}: what the adversary perturbs --- \emph{act}=executed actions, \emph{obs}=observations, \emph{data}=training data, \emph{dyn}=dynamics/disturbance, \emph{kern}=transition kernel. \textbf{Budget}: geometry of the attack set --- $(K,b)$=cardinality \emph{and} magnitude, $\ell_p$=norm ball at every step, $\ell_0$=count only, \emph{set}=uncertainty set. \textbf{Fixed $\pi$}: the guaranteed object is the deployed policy itself, not a smoothed, retrained, robust-argmax or aggregated surrogate. \textbf{Level}: \emph{ep}=whole episode, \emph{st}=per state. \textbf{Form}: \emph{tail}=$\Pr[\cdot\le\alpha]\ge\rho$, \emph{mean}=expected return, \emph{wc}=worst case. \textbf{Paired}: degradation measured against the same episode/model unattacked. \textbf{Resolved}: the output is a function of state (and time), not a scalar. To the best of our knowledge, the cardinality--magnitude budget appears in
the \emph{attack} literature but in no prior \emph{certificate}. For FoSeRL, this denotes a state--time witness and diagnostics,
not certified conditional reliability at each state. \vspace{-0.5cm}}
\label{tab:related}
\end{table}

\paragraph{Relation to offline reinforcement learning.}
FoSeRL shares surface features with offline RL \citep{levine2020offline}: the
policy is fixed and no improvement is performed. The statistical problem differs,
however. Offline RL evaluates or improves a target policy from trajectories
logged by a potentially different behaviour policy, which makes distribution
shift, importance weighting, and coverage its central concerns
\citep{precup2000eligibility,chen2019information}. We instead assume generative
access and collect fresh coupled executions of the certified policy under the
declared attack law, so no off-policy correction or concentrability assumption
enters the analysis. The gap is sharper for the paired objective. A
single-branch nominal log does not contain the attacked counterfactual under the
same realization of the episode randomness, and for deterministic policies
trajectory-wise importance sampling degenerates once the attacked execution
leaves the support of the nominal behaviour. Recovering the paired quantity from
logs alone would therefore require the realized randomness together with a
simulator, or a learned model able to generate the counterfactual branch. In this
respect FoSeRL is closer to high-confidence policy evaluation
\citep{thomas2015high,thomas2019preventing}, though it returns a state- and
time-dependent certificate for a degradation tail event rather than a scalar
bound on policy value. One connection does carry over. Because the barrier is
synthesized and calibrated from sampled coupled trajectories, pointwise readings
of it outside regions the operational data represent are extrapolation, echoing
the support concerns of offline RL \citep{fujimoto2019off}. This does not affect
the distributional certificate, which concerns fresh draws from the declared law,
but it bounds how the barrier should be read away from that coverage. Pessimistic
regularization \citep{kumar2020conservative,yu2020mopo}, learned-model routes to
certification from deployment logs \citep{yu2020mopo,kidambi2020morel}, and
certificate-aware policy improvement in the spirit of safe policy improvement
\citep{laroche2019safe} are natural directions, and lie outside the present
scope.
\subsection{Why policy smoothing is the only external baseline}
\label{app:why-smoothing}

A method qualifies as a quantitative baseline for FoSeRL only if it certifies the
\emph{same object under the same threat}: (i) a formal guarantee rather than an
empirical robustness rate; (ii) on the per-trajectory event
$\{D_{\mathrm{traj}}\le\alpha\}$ rather than on an expectation or a per-state
property; (iii) under perturbations of the executed action with a cardinality
budget; and (iv) for the policy as deployed. Applying these four criteria to the
literature of Table~\ref{tab:related} leaves exactly one method.

\paragraph{Gaussian policy smoothing (comparable).}
When the smoothing noise occupies the same channel the attack perturbs, the
attacked branch is the nominal closed loop driven by shifted noise, and a
Neyman--Pearson shift bound certifies the same terminal event we do. It is
therefore our external baseline. Its guarantee is in one respect
\emph{stronger} than ours: it holds worst-case over every $\delta$ with
$\|\delta\|_2 \le b\sqrt{n_aK}$, a set strictly containing our attack class,
whereas our data-driven certificate holds under the declared attack law. The
comparison is accordingly a budget--scope trade, and the certified budgets differ
by $140\times$--$1575\times$ in the comparable cells.

\paragraph{Direct terminal-event certification (internal reference).}
Direct applies a one-sided Clopper--Pearson bound to the paired terminal
event under the same operational specification as FoSeRL. It provides a
formal scalar reference without barrier synthesis or conditional successor
queries. Appendix~\ref{app:discussion-direct} explains the additional
state--time witness and condition-specific diagnostics supplied by FoSeRL.

\paragraph{Methods excluded, and why.}
\emph{Per-state smoothing} certifies action agreement at individual observations
and makes no statement about accumulated degradation, so no budget of ours maps
onto it. \emph{ReCePS} bounds $\mathbb{E}[J]$ rather than a per-trajectory event
and states its guarantee in observation space; a comparison would require
transplanting $B=b\sqrt{n_aK}$ onto the action channel, and our reimplementation
yields smoothed-return evaluations but no certified bound. \emph{COPA} certifies
against training-data poisoning, a different attack phase entirely.
\emph{Robust-training methods} (RARL, NR-/PR-MDP, ATLA, WocaR, RADIAL) harden a
policy but emit no certificate for a given one, and change the policy in the
process, violating (i) and (iv). \emph{Robust MDPs} synthesise a policy under
kernel ambiguity rather than auditing a fixed one. \emph{Stochastic and neural
barrier certificates} supply our machinery but bound state-safety events, not
adversarial performance loss. Reporting numbers from any of these against ours
would compare different certified quantities, not different certification methods.

\begin{table}[t]
\centering
\caption{Related formal robustness and certification frameworks, compared against FoSeRL.
All three provide formal (high-probability) guarantees; they differ in what is certified and
under which perturbation model. Methods with a different problem formulation are listed for
completeness but are not used as quantitative baselines.\vspace{-0.2cm}}
\label{tab:related_formal_methods}
\scriptsize
\setlength{\tabcolsep}{3pt}
\renewcommand{\arraystretch}{1.15}
\begin{tabular}{@{} p{2.5cm} p{4.5cm} p{4.5cm} p{1.3cm} @{}}
\toprule
Method & Problem formulation and certified property &
Attack model and main assumptions & Comparable to FoSeRL? \\
\midrule

Per-state policy smoothing &
Per-decision robustness of the policy map: certifies action agreement at individual
observations, with no statement about accumulated trajectory degradation &
$\ell_2$ ball applied per observation; each visited state certified independently, with
fresh sampling at every point and no coupling across timesteps &
No \\
\addlinespace

Gaussian policy smoothing &
Sequential robustness under actuator-channel perturbation: certifies the same event as
FoSeRL, $P(D_{\mathrm{traj}}\le\alpha)\ge\rho$ &
Worst case over all $\|\delta\|_2\le b\sqrt{n_aK}$ via a Neyman--Pearson shift bound, paired
by a union bound. Requires noise $\mathcal{N}(0,\sigma^2I)$ shared by both branches and
entering the \emph{same} channel as the attack; $\tau$ grid fixed on design data &
Yes \\
\addlinespace

ReCePS &
Robustness of expected return: certifies a lower bound on $\mathbb{E}[J]$ rather than a
per-trajectory event &
Observation-space $\ell_2$ ball of radius $B$; guarantee holds on the mean, not per
trajectory. Smoothed policy evaluated at inference &
No \\

\bottomrule
\end{tabular}
\end{table}

\paragraph{Fairness of the smoothing comparison.}
Two conditions bound where the baseline applies, and both are properties of the
environment rather than of our implementation. The noise must occupy the attack's
channel, which fails on ACC, where the process noise acts on the lead vehicle's
speed while the attack perturbs the ego control. And the nominal return
distribution must place sufficient mass in an $\alpha$-wide window, since at
$b=0$ the paired bound collapses to that mass; this is an infinite-data ceiling,
not a sample-size artefact, and it is why Unicycle ($S_{\mathrm{env}}/\sigma=6.45$,
against $217$ on the pendulum) cannot be certified by this baseline at any $n$.
We also verified that the baseline is not sample-starved: with $\beta$ split over
the $\tau$ grid, the Clopper--Pearson cap on $\rho(b{=}0)$ alone manufactures
failures below $n\approx 2048$, so every smoothing number reported here uses
$n=4096$, where that cap is $0.994$ against a $0.90$ target.

\begin{table*}[t]
\centering
\caption{Formal comparison with the one method addressing the same problem formulation as
FoSeRL. All quantities are formally certified rather than empirically estimated, at
$\rho_{\mathrm{target}}=0.90$ and confidence $1-\beta=0.999$ ($\beta=\beta_0=10^{-3}$).
FoSeRL gives a distributional guarantee; smoothing gives a worst-case guarantee over the
$\ell_2$ ball $\|\delta\|_2\le b\sqrt{n_aK}$. One representative cell per environment; the
last column is the certified budget ratio over the reported nonzero-PS rows in
Table~\ref{tab:app-foserl-rs-comparison}\vspace{-0.2cm}}
\label{tab:formal_baselines}
\footnotesize
\setlength{\tabcolsep}{7pt}
\begin{tabular}{@{} l c c c rr rr r @{}}
\toprule
& & & & \multicolumn{2}{c}{FoSeRL} & \multicolumn{2}{c}{Smoothing} & \\
\cmidrule(lr){5-6} \cmidrule(lr){7-8}
Environment & $K$ & $\alpha_{\%}$ & $T$ &
$b_{\mathrm{cert}}$ & $\rho_{\mathrm{cert}}$ &
$b_{\mathrm{cert}}$ & $\rho_{\mathrm{cert}}$ & Ratio \\
\midrule
Pendulum (SAC)    & 27 & 2.5\% & 200  & $1.5$  & $0.9174$ & $1.01\times10^{-2}$ & $0.9000$ & $148.5$--$301.2\times$  \\
Hopper (TD3)      & 50 & 2.5\% & 1000 & $1.64$ & $0.9724$ & $2.00\times10^{-3}$ & $0.9000$ & $512.5$--$1100\times$ \\
HalfCheetah (SAC) & 50 & 10\%  & 500  & $0.35$ & $0.9132$ & $1.70\times10^{-3}$ & $0.9000$ & $140$--$1575\times$ \\
\bottomrule
\end{tabular}
\end{table*}

\subsection{Evaluation regime: where smoothing has been tested, and where it has not}
\label{app:regime}

The smoothing line of work has been evaluated primarily on low-dimensional
classic-control and Atari benchmarks. \citet{kumar2022policy} report Cartpole,
Mountain Car, Pong, and Freeway, using DQN for the discrete tasks and DDPG for
Mountain Car; \citet{mu2024reward} follow a similar evaluation regime on
Cartpole, Mountain Car, and Freeway. The only continuous-action environment
in these evaluations is Mountain Car, with a single actuated dimension.
Neither work reports results on MuJoCo locomotion. This distinction becomes important as the action dimension and attack
cardinality increase. Under the aligned smoothing construction used here, the
trajectory-level perturbation radius scales as $b\sqrt{n_aK}$, so larger
$n_a$ and $K$ make the resulting certified per-step budget increasingly
restrictive. Our Hopper and HalfCheetah experiments therefore evaluate the
baseline in a substantially higher-dimensional and longer-horizon regime than
the settings considered in the original smoothing evaluations.
Representative quantitative and qualitative comparisons in these regimes are
reported in Appendix~\ref{app:rs-comparison}.

\end{document}